\documentclass{article}

\usepackage[square]{natbib}

\usepackage{microtype}
\usepackage{graphicx}
\usepackage{booktabs}
\usepackage{xspace}
\usepackage{colortbl}
\usepackage[accepted]{icml2026}
\usepackage{amsmath}

\usepackage[utf8]{inputenc} 
\usepackage[T1]{fontenc}    
\usepackage{lipsum,booktabs}
\usepackage{amsmath,mathrsfs,amssymb,amsfonts,bm}
\usepackage{rotating}
\usepackage{pdflscape}
\usepackage{tcolorbox}
\usepackage{hyperref,url,dsfont,nicefrac}
\usepackage{multirow}
\usepackage{makecell}
\usepackage{anyfontsize}
\usepackage{bbding}
\usepackage{paralist}
\usepackage{enumerate}
\hypersetup{
    colorlinks,
    breaklinks,
    linkcolor = blue,
    citecolor = blue,
    urlcolor  = black,
}
\allowdisplaybreaks
\usepackage{appendix}
\usepackage{multirow,makecell,tabularx}
\usepackage{xfrac}
\usepackage{nicefrac}
\usepackage{threeparttable}
\usepackage{pifont}
\usepackage{tikz}
\usepackage{wrapfig}
\usepackage{enumitem}
\usepackage{circledsteps}

\usepackage{algorithm}
\usepackage{algpseudocode}

\newlength\myindent
\renewcommand{\tilde}{\widetilde}
\renewcommand{\hat}{\widehat}

\def \A {\mathcal{A}}

\def \Bc {\mathcal{B}}

\def \B {\mathbb{B}}

\def \D {\mathcal{D}}
\def \E {\mathbb{E}}

\def \F {\mathcal{F}}

\def \H {\mathcal{H}}
\def \I {\mathcal{I}}

\def \M {\mathcal{M}}

\def \O {\mathcal{O}}

\def \R {\mathbb{R}}
\def \Rc {\mathcal{R}}

\def \Vb {\bar{V}}

\def \X {\mathcal{X}}

\def \a {\mathbf{a}}
\def \b {\mathbf{b}}

\def \c {\mathbf{c}}

\def \eb {\mathbf{e}}

\def \g {\mathbf{g}}

\def \gt {\tilde{g}}
\def \tgb{\tilde{\mathbf{g}}}

\def \I {\mathbb{I}}

\def \m {\boldsymbol{m}}

\def \p {\boldsymbol{p}}

\def \r {\mathbf{r}}

\def \u {\mathbf{u}}
\def \v {\mathbf{v}}
\def \w {\mathbf{w}}
\def \x {\mathbf{x}}
\def \y {\mathbf{y}}

\def \P {\mathbb{P}}

\def \wh {\hat{\w}}

\def \ws {\w^\star}

\def \xh {\hat{\x}}

\def \xs {\x^\star}

\def \mub {\boldsymbol{\mu}}

\def \epsilon {\varepsilon}

\def \Ot {\tilde{\O}}

\def \sumT {\sum_{t=1}^T}

\def \sumTT {\sum_{t=2}^T}

\def \sumN {\sum_{i=1}^N}

\def \is {{i^\star}}

\def \ith {{$i$\text{-th}}\xspace}

\def \ellb {\boldsymbol{\ell}}

\def \Reg {\textsc{Reg}}

\def \DReg {\textsc{D-Reg}}
\def \meta {\textsc{Meta-Reg}}
\def \base {\textsc{Base-Reg}}

\def \mlprod {\textsc{Adapt-ML-Prod}\xspace}

\def \exp {\text{exp}}
\def \sc {\text{sc}}

\newcommand{\rom}[1]{\setcounter{romancounter}{#1}\textit{(\roman{romancounter})}}

\usepackage{mathtools}
\usepackage{autobreak}
\let\norm\undefined 
\newcommand\norm[1]{\left\| #1 \right\|}

\newcommand\abs[1]{\left| #1 \right|}
\newcommand\inner[2]{\langle #1, #2 \rangle}
\newcommand\ceil[1]{\lceil #1 \rceil}

\newcommand\sbr[1]{\left( #1 \right)}
\newcommand\mbr[1]{\left[ #1 \right]}
\newcommand\bbr[1]{\left\{ #1 \right\}}

\newcommand\term[1]{\text{\textsc{term}~(\textsc{#1})}}
\newcommand\given[1][]{\:#1\vert\:}
\newcommand\givenn[1][]{\:#1\middle\vert\:}
\newcommand\seq[1]{\{#1\}_{t=1}^T}

\DeclareMathOperator*{\argmin}{arg\,min}

\def \define {\triangleq}

\numberwithin{equation}{section}

\def \ohedge {\textsf{Optimistic-Hedge}\xspace}

\usepackage{amsthm}

\newenvironment{proof}{\par\noindent{\bf Proof\ }}{\hfill\qedsymbol\\[2mm]}
\renewcommand\qedsymbol{$\blacksquare$}

\newtheorem{Theorem}{Theorem}
\newtheorem{Lemma}{Lemma}
\newtheorem{Corollary}{Corollary}

\theoremstyle{definition}

\newtheorem{Assumption}{Assumption}

\definecolor{darkgreen}{RGB}{0,100,0}
\newtheorem{NoteTemp}{Note}

\newtheorem{Remark}{Remark}

\usepackage{graphicx,subfigure} 

\definecolor{myred}{RGB}{192,0,1}

\definecolor{myblue}{RGB}{68,114,196}

\definecolor{mygreen}{RGB}{0,128,0}

\definecolor{myorange}{RGB}{230,120,0}

\definecolor{mypurple}{RGB}{110,0,160}

\usepackage{prettyref}
\newcommand{\pref}[1]{\prettyref{#1}}

\newcommand{\savehyperref}[2]{\texorpdfstring{\hyperref[#1]{#2}}{#2}}
\newrefformat{eq}{\savehyperref{#1}{Eq.~\textup{(\ref*{#1})}}}
\newrefformat{eqn}{\savehyperref{#1}{Eq.~(\ref*{#1})}}
\newrefformat{lem}{\savehyperref{#1}{Lemma~\ref*{#1}}}
\newrefformat{lemma}{\savehyperref{#1}{Lemma~\ref*{#1}}}
\newrefformat{def}{\savehyperref{#1}{Definition~\ref*{#1}}}
\newrefformat{line}{\savehyperref{#1}{Line~\ref*{#1}}}
\newrefformat{thm}{\savehyperref{#1}{Theorem~\ref*{#1}}}
\newrefformat{cond}{\savehyperref{#1}{Condition~\ref*{#1}}}
\newrefformat{table}{\savehyperref{#1}{Table~\ref*{#1}}}
\newrefformat{corr}{\savehyperref{#1}{Corollary~\ref*{#1}}}
\newrefformat{cor}{\savehyperref{#1}{Corollary~\ref*{#1}}}
\newrefformat{sec}{\savehyperref{#1}{Section~\ref*{#1}}}
\newrefformat{subsec}{\savehyperref{#1}{Section~\ref*{#1}}}
\newrefformat{app}{\savehyperref{#1}{Appendix~\ref*{#1}}}
\newrefformat{appendix}{\savehyperref{#1}{Appendix~\ref*{#1}}}
\newrefformat{ass}{\savehyperref{#1}{Assumption~\ref*{#1}}}
\newrefformat{assum}{\savehyperref{#1}{Assumption~\ref*{#1}}}
\newrefformat{ex}{\savehyperref{#1}{Example~\ref*{#1}}}
\newrefformat{fig}{\savehyperref{#1}{Figure~\ref*{#1}}}
\newrefformat{alg}{\savehyperref{#1}{Algorithm~\ref*{#1}}}
\newrefformat{remark}{\savehyperref{#1}{Remark~\ref*{#1}}}
\newrefformat{conj}{\savehyperref{#1}{Conjecture~\ref*{#1}}}
\newrefformat{prop}{\savehyperref{#1}{Proposition~\ref*{#1}}}
\newrefformat{proto}{\savehyperref{#1}{Protocol~\ref*{#1}}}
\newrefformat{prob}{\savehyperref{#1}{Problem~\ref*{#1}}}
\newrefformat{claim}{\savehyperref{#1}{Claim~\ref*{#1}}}
\newrefformat{que}{\savehyperref{#1}{Question~\ref*{#1}}}
\newrefformat{op}{\savehyperref{#1}{Open Problem~\ref*{#1}}}
\newrefformat{fn}{\savehyperref{#1}{Footnote~\ref*{#1}}}
\newrefformat{require}{\savehyperref{#1}{Requirement~\ref*{#1}}}

\def \mix {\textnormal{\textsc{mix}}}

\makeatletter
\newcommand{\toccontents}{\@starttoc{toc}}

\newcommand{\Bottomcoef}{\kappa}
\def \mlprod {\textsf{Adapt-ML-Prod}\xspace}
\def \omlprod {\textsf{Optimistic-Adapt-ML-Prod}\xspace}
\def \ith {{i\text{-th}}\xspace}

\def \OOGD {\textnormal{\textsf{OOGD}}\xspace}
\newcounter{romancounter}

\def \scvx {\textnormal{sc}}
\def \exp {\textnormal{exp}}
\def \cvx {\textnormal{cvx}}
\def \lin {\textnormal{lin}}

\def \Vb {\bar{V}}

\def \hsc {h^{\textnormal{sc}}}
\def \hc {h^{\cvx}}
\def \hl {h^{\lin}}

\def \meta {\textsc{Meta-Reg}}
\def \base {\textsc{Base-Reg}}

\begin{document}

\twocolumn[
    \icmltitle{Improved Dimension Dependence for\\ Bandit Convex Optimization with Gradient Variation}

    \icmlsetsymbol{equal}{*}

    \begin{icmlauthorlist}
        \icmlauthor{Hang Yu}{keylab,AI}
        \icmlauthor{Yu-Hu Yan}{keylab,AI}
        \icmlauthor{Peng Zhao}{keylab,AI}
    \end{icmlauthorlist}

    \icmlaffiliation{keylab}{State Key Laboratory for Novel Software Technology, Nanjing University, China}
    \icmlaffiliation{AI}{School of Artificial Intelligence, Nanjing University, China}

    \icmlcorrespondingauthor{Peng Zhao}{zhaop@lamda.nju.edu.cn}

    \icmlkeywords{}

    \vskip 0.3in
]
\printAffiliationsAndNotice{}

\begin{abstract}
    Gradient-variation online learning has drawn increasing attention due to its deep connections to game theory and optimization.
    It has been studied extensively in the full-information setting, but is underexplored with bandit feedback.
    In this work, we focus on gradient variation in Bandit Convex Optimization (BCO) with \mbox{two-point} feedback.
    By proposing a refined analysis of the \mbox{\emph{non-consecutive}} gradient variation, a fundamental quantity in gradient variation with bandit feedback, we improve the dimension dependence for both convex and strongly convex functions compared with the best known results~\citep{chiang2013beating}.
    Our improved analysis of the \mbox{non-consecutive} gradient variation also implies other favorable \mbox{problem-dependent} guarantees, such as \mbox{gradient-variance} and \mbox{small-loss} regret bounds.
    Beyond the \mbox{two-point} setup, we demonstrate the versatility of our technique by achieving the \emph{first} \mbox{gradient-variation} bound for one-point bandit linear optimization over \mbox{hyper-rectangular} domains.
    Finally, we validate the effectiveness of our results in more challenging tasks such as dynamic and universal regret minimization, establishing the \emph{first} gradient-variation dynamic and universal regret bounds for two-point BCO.
\end{abstract}


\section{Introduction}
\label{sec:intro}
Online Convex Optimization (OCO) is a powerful and fundamental framework for modeling the interaction between a learner and the environment over time \citep{hazan2016introduction,orabona2019modern}.
In round $t\in[T]$, the learner selects $\x_t \in \X \subseteq \R^d$, while the environment simultaneously chooses a convex function $f_t:\X\to \R$.
Then the learner suffers the loss $f_t(\x_t)$ and receives gradient feedback about the online function, aiming to optimize the \mbox{game-theoretical} performance measure known as regret~\citep{cesa2006prediction}, which is defined as
\begin{equation*}
    \Reg_T^{\text{(OCO)}} \define \sumT f_t(\x_t) - \min_{\x\in\X} \sumT f_t(\x).
\end{equation*}
For OCO, the minimax optimal regret results are $\O(\sqrt{T})$ for convex and $\O(\log T)$ for strongly convex functions \citep{hazan2016introduction}.
Beyond worst-case minimax optimality, the literature considers enhancing the adaptivity of the learner by adapting the regret to problem-dependent hardness.
Among various problem-dependent quantities, \emph{gradient variation}~\citep{chiang2012online, yang2014regret} is of particular interest, defined as the cumulative variation of gradients across consecutive functions:
\begin{equation}
    \label{eq:VT}
    V_T \define \sumTT \sup_{\x \in \X} \|\nabla f_t(\x) - \nabla f_{t-1}(\x)\|^2_2.
\end{equation}
By adapting to the gradient variation, the aforementioned minimax regret guarantees can be improved to $\O(\sqrt{V_T})$ for convex functions and $\O(\log V_T)$ for strongly convex functions. Gradient-variation online learning has received renewed attention in recent years, particularly since its introduction into dynamic regret minimization \citep{zhao2020dynamic,zhao2024adaptivity}, which has led to dedicated algorithmic structures and new technical tools and has further inspired a growing body of work in various settings~\citep{qiu2023gradient,tsai2023datadependent,yan2023universal,tarzanagh2024online,xie2024gradient,yan2024simple,zhao2026gradient}.
This line of research also reveals that gradient-variation adaptivity has profound connections to bridging adversarial and stochastic optimization~\citep{Sarah2022between,chen2024optimistic}, enabling fast rates in games~\citep{rakhlin2013optimization,syrgkanis2015fast}, and facilitating acceleration in smooth offline optimization~\citep{cutkosky2019anytime,kavis2019unixgrad,yan2025optimistic,zhao2025gradient}, among others.

\begin{table*}[!t]
    \centering
    \caption{\small{Comparison of problem-dependent regret bounds for two-point BCO. Here, we consider the \emph{nondegenerate} setup for clarity, where we assume $V_T, W_T, F_T \ge \Omega(d)$. $V_T, W_T$, and $F_T$ denote the gradient variation~\eqref{eq:VT}, gradient variance~\eqref{eq:WT}, and small loss~\eqref{eq:FT}, respectively. The \smash{$\Ot(\cdot)$} notation omits logarithmic factors in the dimension $d$ and the time horizon $T$. We use `\textemdash' to denote results that match but do not improve upon state-of-the-art bounds.}}
    \renewcommand{\arraystretch}{1.2}
    \resizebox{0.8\linewidth}{!}{
    \begin{tabular}{r|ccc}
    \hline

    \hline
    & \textbf{Linear} & \textbf{Convex} & \textbf{$\lambda$-Strongly Convex} \\
    \hline
    \cellcolor{gray!10}\rule{0pt}{5mm}\textbf{\citet{chiang2013beating}} & \cellcolor{gray!10}$\O\big(d^{\frac{3}{2}}\sqrt{V_T}\big)$ & \cellcolor{gray!10} $\Ot\big(d^2\sqrt{V_T}\big)$ &\cellcolor{gray!10} $\O\big(\frac{d^2}{\lambda}\log V_T\big)$\\[1mm]
    \hline \hline
    \rule{0pt}{5mm} \textbf{Ours} [Gradient Variation $V_T$] & \textemdash & $\Ot\big(d^{\frac{3}{2}}\sqrt{V_T}\big)$ [\pref{thm:cvx-base}] & $\O\big(\frac{d}{\lambda}\log V_T\big)$ [\pref{thm:scvx-base}]  \\[1mm]
    \hline
    \rule{0pt}{5mm} \textbf{Ours} [Gradient Variance $W_T$] & $\O\big(\sqrt{dW_T}\big)$ [\pref{thm:WT}] & $\O\big(d\sqrt{W_T}\big)$ [\pref{thm:WT}] & $\O\big(\frac{d}{\lambda}\log W_T\big)$ [\pref{thm:WT}]  \\[1mm]
    \hline
    \cline{2-4}
    \rule{0pt}{5mm} \textbf{Ours} [Small Loss $F_T$] & $\O\big(\sqrt{dF_T}\big)$ [\pref{thm:FT}] & $\O\big(\sqrt{dF_T}\big)$ [\pref{thm:FT}] & $\O\big(\frac{d}{\lambda}\log F_T\big)$ [\pref{thm:FT}]  \\[1mm]
    \hline

    \hline
    \end{tabular}
    }
    \label{table:mainresults}
\end{table*}

While gradient-variation online learning has been studied extensively in the full-information setting, it is still underexplored in Bandit Convex Optimization (BCO), where the learner only has access to the function values. Based on the number of function values queried, BCO can be classified into \mbox{one-point}, \mbox{two-point}, and \mbox{multi-point} settings.
In the \mbox{one-point} setup, achieving gradient-variation regret (specialized as squared path-length regret in multi-armed bandits) remains open~\citep{wei2018more}.
By contrast, when it comes to the \emph{two-point} setup, the \mbox{gradient-variation} regret bounds become attainable~\citep{chiang2013beating}.
Specifically, two-point BCO allows the learner to query two points $\x_t, \x_t^\prime \in \X$ at round $t \in [T]$, and to observe the corresponding loss values, $f_t(\x_t)$ and $f_t(\x_t^\prime)$, revealed by an oblivious environment.
\citet{chiang2013beating} initiated the study of gradient variation in two-point BCO and provided the first bounds of \smash{$\O(\sqrt{d^3 V_T})$}, \smash{$\Ot(d^2 \sqrt{V_T})$}, and \smash{$\O(\frac{d^2}{\lambda}\log (dV_T))$} for linear, convex, and $\lambda$-strongly convex functions, where $d$ is the dimension and $\Ot(\cdot)$ omits the logarithmic factors in $T$ and $d$.
While their results enjoy the optimal dependence on $V_T$, they incur a large dimension dependence, as the $\Omega(\sqrt{dT})$ convexity lower bound~\citep{duchi2015optimal} indicates that a tighter dimension dependence is possible for gradient-variation bounds in two-point BCO.

Mitigating the dimension dependence is a fundamental challenge in BCO~\citep{agarwal2010optimal,fokkema2024online} and zeroth-order stochastic optimization~\citep{duchi2015optimal,nesterov2017random,wang2018stochastic}, and there has been a lot of progress on this front.
The difficulty stems from the inherent information bottleneck in bandit feedback, where reconstructing a $d$-dimensional gradient from scalar function values necessitates a sampling complexity that scales unfavorably with dimension $d$~\citep{lattimore2025banditconvexoptimisation}.

In bandit gradient-variation online learning, reducing the dimension dependence poses additional challenges.
To see this, we provide an intuition.
In OCO, where the learner has access to the full gradient information in \emph{all} directions, e.g., $\nabla f_t$ and $\nabla f_{t-1}$, the gradient-variation regret is straightforward to achieve by using the well-known optimistic online learning technique~\citep{chiang2012online}.
However, with bandit feedback, the learner can sample \emph{only one} direction at each round.
For example, at the $t$-th round, the learner samples a random direction $i_t \in [d]$, constructs a gradient estimator, and obtains an estimate of $\nabla_{i_t} f_t$, where $\nabla_{i} f$ denotes the gradient of $f$ in the $i$-th direction.
Therefore, it is hard to analyze $\nabla_{i_t} f_t - \nabla_{i_{t-1}} f_{t-1}$ directly because the two directions between consecutive rounds are very likely to be different.
To this end, in bandit optimization, an essential quantity is a \emph{non-consecutive} version of the gradient variation~\citep{chiang2013beating, wei2018more}, which conceptually depends on the following term:
\begin{equation}
    \sumT (\nabla_{i_t} f_t - \nabla_{i_t} f_{\alpha_t})^2,
\end{equation}
where $\alpha_t$ is the largest integer such that $0 \leq \alpha_t < t$ and $i_{\alpha_t} = i_t$.
Since the learner can only sample one direction at each round, the non-consecutive sampling gap, i.e., $t-\alpha_t$, will inevitably scale with the dimension $d$, leading to an additional dimension dependence compared with regret bounds in the full-information setting. 

In this work, we tighten the dimension dependence of the gradient-variation regret bounds in two-point BCO by unraveling the inherent correlation structure in the \emph{non-consecutive gradient variation}.
By carefully decoupling these dependencies, we achieve \smash{$\Ot(d^{\frac{3}{2}}\sqrt{V_T})$} for convex functions and \smash{$\O(\frac{d}{\lambda}\log V_T)$} for $\lambda$-strongly convex functions, thereby improving the best known results by factors of nearly \smash{$\sqrt{d}$} and $d$, respectively.

Our analysis for \mbox{non-consecutive} gradient-variation also implies regret scaling with other favorable \mbox{problem-dependent} quantities, such as gradient variance $W_T$ and small loss $F_T$, thereby offering multiple perspectives to depict the \mbox{problem-dependent} hardness.
Among the implied results, in particular, we achieve \smash{$\O\sbr{\sqrt{d F_T}+d}$} for convex functions and \smash{$\O\sbr{\sqrt{d W_T}+d}$} for linear functions, which are both optimal up to an \emph{additive} $\O(d)$ term. \pref{table:mainresults} summarizes our complete results.

Beyond the two-point setup, we generalize our techniques to one-point Bandit Linear Optimization (BLO).
Briefly, we introduce a novel gradient estimator with an associated algorithm and establish the \emph{first} gradient-variation regret bound for one-point BLO, in a special case where the domain is a hyper-rectangle, highlighting the versatility of our approach.

Finally, we showcase the effectiveness of our methods in more challenging environments: \textit{(i)} \mbox{\emph{dynamic regret}}~\citep{zhang2018adaptive}, where the learner competes against oblivious time-varying comparators; \textit{(ii)} \mbox{\emph{universal regret}}~\citep{van2016metagrad}, where the learner has no prior knowledge of the function's curvature but aims to match the guarantees of curvature-aware methods.
To conclude, we establish the \emph{first} gradient-variation dynamic and universal regret bounds for two-point BCO.

\textbf{Contributions.}~~
Our contributions are summarized below:
\begin{itemize}[left=2pt, itemsep=-2pt, topsep=-3pt]
    \item For two-point BCO with gradient variation, we obtain $\Ot(d^{\frac{3}{2}}\sqrt{V_T})$ and $\O(\frac{d}{\lambda}\log V_T)$ for convex and $\lambda$-strongly convex functions, thereby improving the previously best known results by factors of almost $\sqrt{d}$ and $d$, respectively.
    \item We achieve the \emph{first} gradient-variance and small-loss regret bounds for two-point BCO; among them, \smash{$\O(\sqrt{d F_T}+d)$} and \smash{$\O(\sqrt{d W_T}+d)$} for convex and linear functions, respectively, are \emph{the first} \mbox{problem-dependent} guarantees that can recover the minimax optimal $\O(\sqrt{dT})$ regret bound.
    \item We derive the \emph{first} gradient-variation regret bound in the one-point BLO setting over hyper-rectangular domains.
    \item We establish the \emph{first} gradient-variation dynamic and universal regret bounds in two-point BCO.
\end{itemize}

\textbf{Organization.}~~
The rest of the paper is organized as follows:
In \pref{sec:preliminary}, we introduce the preliminaries.
In \pref{sec:two-point}, we present our main results for two-point BCO.
In \pref{sec:1point}, we extend our methods to one-point BLO.
In \pref{sec:extension}, we generalize our methods to more challenging environments, including dynamic regret and universal regret.
Finally, in \pref{sec:conclusion}, we conclude the paper.
Due to page limits, all proofs are deferred to appendices.


\section{Preliminaries}
\label{sec:preliminary}
In this section, we introduce the notation and assumptions, and briefly review the method of \citet{chiang2013beating}.

\subsection{Notations and Assumptions}
\label{subsec:notations}
\textbf{Notation.}~~  For any $N \in \mathbb{N}$, we define $[N]$ as $\{1, \ldots, N\}$. We represent the $i$-th coordinate of a bold vector $\v$ (or $\boldsymbol{v}$) by the corresponding regular-font symbol $v_i$, i.e., $\v$ (or $\boldsymbol{v}$) $= (v_1, \dots, v_d)^\top$. We use $\nabla_i f$ to denote the partial derivative of $f$ with respect to the $i$-th coordinate.
We use $\|\cdot\|$ for $\|\cdot\|_2$ by default.
We write $a \lesssim b$, or $a = \O(b)$, if there exists a constant $C < \infty$ such that $a \le Cb$. We use $\O(\cdot)$ to highlight the dependence on $d$, $T$, and \mbox{problem-dependent} quantities, while $\Ot(\cdot)$ omits logarithmic factors in $d$ and $T$. Throughout the paper, we treat the $\log\log T$ factor as a constant and omit it following~\citet{luo2015achieving}.

\begin{Assumption}[Boundedness]
    \label{ass:boundedness}
    The feasible domain $\X\subseteq \R^d$ is compact, convex, and satisfies $r\B\subseteq \X\subseteq R\B$, where $\B=\{\x\in \R^d \given \|\x\|\le 1\}$ is a unit ball.
\end{Assumption}
\begin{Assumption}[Lipschitzness]
    \label{ass:Lipschitzness}
    For any $\x,\x^\prime\in\X$ and all $t\in[T]$, $|f_t(\x)-f_t(\x^\prime)|\le G \|\x-\x^\prime\|$.
\end{Assumption}
\begin{Assumption}[Smoothness]
    \label{ass:Smoothness}
    For any $\x,\x^\prime\in \X$ and all $t\in[T]$, $\|\nabla f_t(\x)-\nabla f_t(\x^\prime)\|\le L \|\x-\x^\prime\|$.
\end{Assumption}
Assumptions~\ref{ass:boundedness} and~\ref{ass:Lipschitzness} are standard for BCO \citep{flaxman2004online,agarwal2010optimal,lattimore2025banditconvexoptimisation}.
Assumption~\ref{ass:Smoothness} is essential for first-order methods to achieve problem-dependent regret~\citep{srebro2010smoothness,chiang2012online}.

\subsection{A Review of \texorpdfstring{\citet{chiang2013beating}}{Chiang et al. (2013)}}
\label{subsec:chiang}
For full-information feedback, a standard technique for gradient-variation regret is Optimistic Online Gradient Descent (\OOGD) \citep{chiang2012online}. At round $t$, the learner uses an optimism term $M_{t}$, which serves as a predictive hint of the upcoming gradient $\nabla f_{t}(\x_{t})$. Based on this optimism, \OOGD proceeds with the following two-step updates:
\begin{equation*}
    \x_t = \Pi_{\X} \mbr{\xh_{t}-\eta_{t} M_{t}},\ \xh_{t+1} = \Pi_{\X} \mbr{\xh_t-\eta_t \nabla f_{t}(\x_{t})}
\end{equation*}
where $\eta_t > 0$ is a time-varying step size, \smash{$\xh_t$} and \smash{$\xh_{t+1}$} are internal decisions, and $\Pi_{\X}[\x] \define \argmin_{\y \in \X} \|\x - \y\|$ is the Euclidean projection onto the feasible domain $\X$.
The resulting regret depends on the cumulative prediction error \smash{$\sum_{t=1}^T \|\nabla f_t(\x_t) - M_t\|^2$}, which characterizes the accuracy of the prediction $M_t$.
A straightforward instantiation of $M_t$ is to set it as the preceding gradient $\nabla f_{t-1}(\x_{t-1})$.
Such a predictive choice is sufficient to attain the optimal gradient-variation regret~\citep{chiang2012online}.

For two-point BCO, where the learner only has access to function values rather than gradients, we define the corresponding cumulative prediction error as:
\begin{equation}
    \label{eq:VbT}
    \Vb_T \define \sum_{t=1}^T \|\g_t - \tgb_t\|^2,
\end{equation}
where $\g_t$ is the gradient estimator at round $t$ and $\tgb_t$ denotes the optimism constructed from historical information up to round $t-1$.
A direct choice in bandits would use the estimator $\g_t = c_t \u_t$ and set the optimism $\tgb_t$ to the preceding estimator $\g_{t-1}$.
Here, $\u_t$ denotes a random vector drawn from a specified distribution, and $c_t$ is an estimation constant chosen to ensure $\E[\g_t]\approx \nabla f_t(\x_t)$.
For example, in \mbox{two-point} BCO~\citep{agarwal2010optimal}, $\u_t$ is uniformly sampled from the unit sphere, and $c_t = \frac{d}{2\delta}(f_t(\x_t + \delta \u_t) - f_t(\x_t - \delta \u_t))$, where $\delta > 0$ is a small exploration parameter.
However, in this case, the gap $\|\g_t - \tgb_t\|$ becomes unmanageable, as the randomness of $\u_t$ and $\u_{t-1}$ causes severely misaligned consecutive estimators with high probability.

To bridge this gap, inspired by the gradient estimator in \citet{hazan2009betterBCO,hazan2011betterBCO}, \citet{chiang2013beating} introduced a novel gradient estimator and an optimism term that effectively address direction misalignment.
Specifically, at $t\in[T]$, the gradient estimator $\g_t$ and the optimism $\tgb_t$ are constructed as follows:
\begin{equation}
    \label{eq:chiangestimator}
    \begin{gathered}
        \g_t=d\left(v_t-\tilde{g}_{t, i_t}\right) \eb_{i_t}+\tilde{\g}_t,\\
        \tilde{\g}_{t+1}=\left(v_t-\tilde{g}_{t, i_t}\right) \eb_{i_t}+\tilde{\g}_t,
    \end{gathered}
\end{equation}
where $i_t$ is drawn uniformly from $[d]$, $\{\eb_1,\ldots,\eb_d\}$ is the standard basis of $\R^d$, and $v_t \define \frac{1}{2\delta} (f_t(\w_t + \delta \eb_{i_t}) - f_t(\w_t - \delta \eb_{i_t}))$ serves as an estimate of the directional derivative of $f_t$ at $\w_t$ along $\eb_{i_t}$.
Here, $\w_t$ is the center around which the query points $\x_t$ and $\x_t^\prime$ are sampled as $\x_t=\w_t+\delta\eb_{i_t}$ and $\x_t^\prime =\w_t-\delta\eb_{i_t}$, where $\delta > 0$ is a small exploration parameter.
By concentrating the difference on a single coordinate, $\g_t - \tgb_t = d (v_t - \gt_{t, i_t}) \eb_{i_t}$, \pref{eq:chiangestimator} yields a manageable difference between $\g_t$ and $\tgb_t$, which further leads to a controllable $\Vb_T$.
Leveraging this construction, \citet{chiang2013beating} integrated the estimator and optimism in \pref{eq:chiangestimator} into \OOGD. We restate their method in \pref{alg:chiang}.

Despite this innovative design, \citet{chiang2013beating} underestimated the non-consecutive nature of the gradient estimators.
Specifically, by choosing the estimators from \pref{eq:chiangestimator}, $\Vb_T$~\eqref{eq:VbT} exhibits the following structure:
\begin{equation}
    \label{eq:non-consecutive}
    \Vb_T = d^2\sum_{t=1}^T  (v_t-v_{\alpha_t})^2,
\end{equation}
where $\alpha_t$ is the largest integer such that $0 \leq \alpha_t < t$ and $i_{\alpha_t} = i_t$.
Intuitively, the $(v_t-v_{\alpha_t})$ term measures the gap between two directional derivative estimates from two iterations in which the sampled direction is the same, leading to a natural non-consecutive structure.
Essentially, this term is closely tied to the dimension dependence of the regret bound, as it captures the distance accumulated before the algorithm re-samples the same direction, given that it draws one direction out of $d$ at each round. In the next section, we present the analysis of \citet{chiang2013beating} for the essential quantity $\Vb_T$, its limitations, and our improved analysis.


\section{Our Method}
\label{sec:two-point}
In this section, we improve the analysis of the \emph{non-consecutive} structure defined in \pref{eq:non-consecutive}.
For the sake of emphasis, we refer to $\Vb_T$~\eqref{eq:non-consecutive} as the non-consecutive gradient variation in the remainder of the paper.
In \pref{sec:cvx-base}, we present an improved analysis of $\E[\Vb_T]$, which enables us to establish enhanced regret bounds for general convex functions.
In \pref{sec:scvx-base}, we provide a tight characterization of the maximal term within the expected non-consecutive gradient variation, i.e., \smash{$\max_{t\in[T]}\E[\|\g_t-\tgb_{t}\|^2]$}. This characterization, coupled with a stabilized step-size schedule, yields an improved regret bound in the strongly convex setting.
Finally, in \pref{subsec:F_T and W_T}, we show that the non-consecutive gradient variation yields other problem-dependent guarantees, such as gradient-variance and small-loss bounds.

\subsection{Improvement on Convex Case}
\label{sec:cvx-base}
In this part, we focus on bandit gradient-variation regret for convex functions and improve upon the result of \citet{chiang2013beating} by a factor of nearly $\sqrt{d}$, thereby closing the regret gap between the convex and linear settings in the $\Ot(\cdot)$-notation, i.e., regardless of the logarithmic factors in the dimension $d$.

\begin{algorithm}[t]
\caption{Algorithm of \citet{chiang2013beating}}
\label{alg:chiang}
\begin{algorithmic}[1]
\Require Step sizes $\{\eta_t\}_{t=1}^T$.
\State Let $\w_1 = \hat{\w}_1 = \mathbf{0}$ and $\tgb_1 = \mathbf{0}$. Set exploration parameter $\delta=\frac{1}{2d^2 L TR}$ and shrinkage parameter $\xi=\frac{\delta}{r}$.
\For{$t =1,2,\ldots,T$}
    \State Choose $i_t$ uniformly from $[d]$.
    \State Submit two query points $\x_t = \w_t + \delta \eb_{i_t}, \x_t^\prime = \w_t - \delta \eb_{i_t}$, and observe $f_t(\x_t)$ and $f_t(\x_t^\prime)$.
    \State Compute the gradient estimator $\g_t$ and the optimism $\tgb_{t+1}$ as in \pref{eq:chiangestimator}
    \State Update the iterate as follows:
    \begin{align*}
        \hat{\w}_{t+1} = {} & \Pi_{(1 - \xi)\mathcal{X}}\mbr{\hat{\w}_{t} - \eta_t \g_t},\\
        \w_{t+1} = {} & \Pi_{(1 - \xi)\mathcal{X}}\mbr{\hat{\w}_{t+1} - \eta_{t+1} \tgb_{t+1}}
    \end{align*}
\EndFor
\end{algorithmic}
\end{algorithm}

We first restate the decomposition of the non-consecutive gradient variation $\Vb_T$~\eqref{eq:non-consecutive} from \citet{chiang2013beating} in \pref{lem:decom of VbT}, with the proof deferred to \pref{app:lem0}.

\begin{Lemma}[Decomposition of $\Vb_T$ in \citet{chiang2013beating}]
\label{lem:decom of VbT}
Under Assumptions \ref{ass:boundedness}-\ref{ass:Smoothness}, for convex functions, \pref{alg:chiang} satisfies the following guarantee:
\begin{align}
 &\Vb_T\lesssim 2d^2\sumT\sum_{i=1}^{d}\rho_{t,i}\sbr{ \nabla_{i} f_{t}(\w_{t-1})-\nabla_{i} f_{t-1}(\w_{t-1})}^2\notag\\
 &+ 2d^2\sumT\sum_{i=1}^{d}\rho_{t,i}\sbr{ \nabla_{i} f_{t}(\w_t)-\nabla_{i} f_{t}(\w_{t-1})}^2.\label{eq:VT_decom}
\end{align}
Here, $\rho_{t,i}$ is the non-consecutive sampling gap defined as \smash{$\rho_{t,i} \define \tau_2 - \tau_1$}, where \smash{$\tau_1 \define\max\{ 0\le\tau < t\given i_\tau = i\}$} and \smash{$\tau_2\define\min \{ t\le\tau\le T\given i_\tau = i\}$}, with the convention that $\tau_2=T+1$ if no such $\tau$ exists.
\end{Lemma}

Intuitively, $\rho_{t,i}$ quantifies the duration between the most recent sampling of coordinate $i$ before $t$ and its next sampling at or after $t$. In \pref{eq:VT_decom}, the first term captures gradient variation along the sampled directions and is therefore connected to $V_T$~\eqref{eq:VT}, while the second term essentially measures the algorithmic stability between $\w_{t-1}$ and $\w_t$, which will be cancelled by negative terms in the regret analysis.

Next, we focus on the first term in \pref{eq:VT_decom}, due to its close connection to $V_T$. To highlight the difference between our analysis and that of \citet{chiang2013beating}, we first revisit their analysis in the linear case, then explain why the same strategy suffers an additional $d$ factor in the convex case, and finally show how we overcome this issue.

In the linear setting, as analyzed by \citet{chiang2013beating}, the gradient difference $\nabla_i f_t(\x)-\nabla_i f_{t-1}(\x)$ is constant in $\x$ and is independent of the non-consecutive gap $\rho_{t,i}$. Taking expectation gives
\begin{align}
    &\E\mbr{\sum_{i=1}^{d}\rho_{t,i} (\nabla_i f_t(\w_{t-1})-\nabla_i f_{t-1}(\w_{t-1}))^2}\label{eq:mid-cvx}\\
    \le{}& \max_{i\in[d]}\E[\rho_{t,i}] \sup_{\x\in\X} \|\nabla f_t(\x) - \nabla f_{t-1}(\x)\|^2,\notag
\end{align}
where the inequality follows because linear functions have constant gradients. 
Following Lemma 5 of \citet{chiang2013beating}, we have $\E[\rho_{t,i}] \le 2d$ for all $i\in[d]$, yielding an $\O(d)$ dimension dependence.

Furthermore, when the same analysis of \citet{chiang2013beating} is applied to convex functions, this independence no longer holds, as both $\rho_{t,i}$ and the gradient difference share the randomness of the direction sampling sequence $\{i_s\}_{s=1}^t$. 
This interdependence complicates the analysis and leads to a coarse upper bound below:
\begin{equation*}
    \text{\pref{eq:mid-cvx}}\le \E\mbr{\sum_{i=1}^{d}\rho_{t,i}} \sup_{\x\in\X}\|\nabla f_t(\x)-\nabla f_{t-1}(\x)\|^2,
\end{equation*}
where the inequality is due to $w_i^2 \le \|\w\|^2$ for any vector $\w\in\R^d$. Since each $\E[\rho_{t,i}]$ is of order $\O(d)$, this analysis introduces an additional $d$ factor.

To address this challenge, we decouple the dependence between the sampling gap $\rho_{t,i}$ and the gradient difference by taking the supremum over all sampling gaps:
\begin{align*}
    \text{\pref{eq:mid-cvx}}&=\E\mbr{\sum_{i=1}^{d}\rho_{t,i} (\nabla_i f_t(\w_{t-1})-\nabla_i f_{t-1}(\w_{t-1}))^2}\\
    \le\E&\mbr{\max_{i\in[d]}\rho_{t,i} \sum_{i=1}^d \sbr{\nabla_i f_{t}(\w_{t-1})-\nabla_i f_{t-1}(\w_{t-1})}^2}\\
    \le\E&\mbr{\max_{i\in[d]} \rho_{t,i}} \sup_{\x\in\X}\|\nabla f_t(\x)-\nabla f_{t-1}(\x)\|^2.
\end{align*}
It remains to control the largest non-consecutive sampling gap $\max_{i\in[d]}\rho_{t,i}$.
Our technical contribution is to establish the connection between $\max_{i\in[d]}\rho_{t,i}$ and the Coupon Collector's Problem (CCP).
By doing this, we achieve a tighter $\E[\max_{i\in[d]}\rho_{t,i}]=\O(d\log d)$, which introduces only an extra $\log d$ factor compared with the analysis in the linear setting.
The classical CCP studies the waiting time needed to collect all $d$ coupon types when, at each draw, one coupon is sampled uniformly at random from the $d$ types.
A standard property of CCP is that this expected waiting time is $\O(d\log d)$~\citep{motwani1995randomized}.
In our setting, each coordinate sampling $i_t$ can be viewed as drawing one coupon from $[d]$.
Specifically, for a fixed time $t$, we decompose the maximal gap as
\begin{align}
    \max_{i\in[d]}\rho_{t,i}\le\max_{i\in[d]} (t-\tau_1)+ \max_{i\in[d]} (\tau_2-t+1),\label{eq:decom_rho}
\end{align}
where we recall that \smash{$\tau_1 \define\max\{ 0\le\tau < t\given i_\tau = i\}$} and \smash{$\tau_2\define\min \{ t\le\tau\le T\given i_\tau = i\}$}.
The first term is the number of steps needed to see all coordinates when searching backward from time $t$, and the second term is the number of steps needed to see all coordinates when searching forward from time $t$.
Since the coordinates are sampled independently and uniformly from $[d]$, the expectation of each term is dominated by a CCP waiting time.
Applying these upper bounds to both terms in \pref{eq:decom_rho}, we obtain the desired expected upper bound on $\max_{i\in[d]}\rho_{t,i}$. The detailed reduction is provided in \pref{app:rho}.

The aforementioned insight allows us to derive a refined analysis for the non-consecutive gradient variation $\Vb_T$, as formalized in \pref{lem:bandit-VT-correction-cvx}. The proof is in \pref{app:lem1}.
\begin{Lemma}
\label{lem:bandit-VT-correction-cvx}
Under Assumptions \ref{ass:boundedness}-\ref{ass:Smoothness}, for convex functions,  \pref{alg:chiang} satisfies the following guarantee:
\begin{align*}
    \E\mbr{\Vb_T} \le{}&  16d^3 L^2 \log(dT) \cdot \E\mbr{\sumT \|\w_t-\w_{t-1}\|^2} \\
    &+8d^3 V_T\log d + \O(1).
\end{align*}
\end{Lemma}
By leveraging \pref{lem:bandit-VT-correction-cvx}, we achieve a tighter dimension dependence for convex functions in \pref{thm:cvx-base} below, with the proof deferred to \pref{app:cvx-base}.
\begin{Theorem}
    \label{thm:cvx-base}
    Under Assumptions \ref{ass:boundedness}-\ref{ass:Smoothness},  for convex functions, choosing $\eta_t=\frac{R}{\sqrt{1152d^3R^4L^2\log(dT)+\bar{V}_{t-1}}}$, \pref{alg:chiang} satisfies the following guarantee:
    \begin{align*}
        \E[\Reg_T]\define{}& \E\mbr{\sum_{t=1}^T \frac{1}{2} \sbr{f_t(\x_{t})+f_t(\x_{t}^\prime)}-\min_{\x\in \X}\sum_{t=1}^{T}f_t(\x)}\\
        \le{}& \Ot \big(\sqrt{\min\{d^3 V_T, dT+d^3\}}\big),
    \end{align*}
    where $V_T$ and $\Vb_t$ are defined in \pref{eq:VT} and \pref{eq:VbT}.
\end{Theorem}
Up to logarithmic factors, \pref{thm:cvx-base} effectively closes the regret gap between the convex and linear settings. 
Furthermore, while the dimension dependence in our gradient variation bound is not as tight as that in the minimax-optimal $\O(\sqrt{dT})$, our result performs better in \emph{benign} environments, e.g., when $V_T = o(T/d^2)$.
Meanwhile, our result offers an $\Ot(\sqrt{dT+d^3})$ worst-case safeguard, matching optimal regret up to an additive \smash{$\Ot(d^{3/2})$ term.}

\subsection{Improvement on Strongly Convex Case}
\label{sec:scvx-base}
In this part, we focus on bandit gradient-variation regret for strongly convex functions.
Our solution consists of two key components: a more stable step-size schedule and a tight characterization of the \emph{maximal expected variation}, $\max_{t\in[T]}\E[\|\g_t-\tgb_t\|^2]$.
To contextualize our improvements, we begin with a brief review of the problem-dependent learning rate by \citet{chiang2013beating}.

Specifically, \citet{chiang2013beating} chose a problem-dependent learning rate schedule as \smash{$\eta_t \approx \frac{1}{\lambda \Vb_{t-1}}$}, where $\Vb_{t-1}$ is defined in \pref{eq:VbT}.
This learning rate is not stable enough, as the randomness of the gradient estimator can perturb it when the function value varies dramatically, leading to large regret.
Moreover, the stochasticity in the step size makes the analysis challenging due to the correlation between the step size and the gradient estimator.

To tackle this issue, we adopt a more stable and deterministic learning rate schedule~\citep{chen2023optimistic}:
\begin{equation*}
    \eta_t=\frac{1}{\lambda t}.
\end{equation*}
Building upon this deterministic step size, we propose a tight analysis for the following maximal expected variation.
Below, we establish \pref{lem:str_alpha}, with the proof in \pref{app:g_t}.
\begin{Lemma}
    \label{lem:str_alpha}
    Under Assumptions \ref{ass:boundedness}-\ref{ass:Smoothness}, for convex functions, \pref{alg:chiang} satisfies, for any $t\in[T]$,
    \begin{equation*}
        \E\mbr{\left\|\g_t-\tgb_t\right\|^2}\le 8d G^2+\O\sbr{\frac{1}{d^2T^2}}.
    \end{equation*}
\end{Lemma}
By combining the deterministic step size and the tight analysis for the maximal expected variation, we achieve an improved regret guarantee in the strongly convex setting in \pref{thm:scvx-base}, with the proof deferred to \pref{app:scvx-base}.
\begin{Theorem}
    \label{thm:scvx-base}
    Under Assumptions \ref{ass:boundedness}-\ref{ass:Smoothness}, for $\lambda$-strongly convex functions, choosing \smash{$\eta_t=\frac{1}{\lambda t}$}, \pref{alg:chiang} enjoys
    \begin{equation*}
        \E[\Reg_T]\le \O\sbr{\frac{d}{\lambda}\log(dV_T)}.
    \end{equation*}
\end{Theorem}
Compared to \smash{$\O\big(\frac{d^2}{\lambda}\log(dV_T)\big)$} of {\citet[Theorem 16]{chiang2013beating}}, \pref{thm:scvx-base} tightens the dimensional dependence from $d^2$ to $d$.
As a byproduct, our result also tightens the \emph{worst-case} bound for strongly convex functions, improving the \smash{$\O\big(\frac{d^2}{\lambda}\log T\big)$} of \citet{agarwal2010optimal} by a factor of $d$.
\begin{Corollary}
   \label{cor:scvx}
    With the same assumptions and step size as in \pref{thm:scvx-base}, \pref{alg:chiang} enjoys $\E[\Reg_T]\le \O(\frac{d}{\lambda}\log T)$.
\end{Corollary}
Notably, without smoothness, the same regret guarantee can be achieved by a simple algorithm coupled with a dedicated concentration-based analysis~\citep{shamir2017optimal}.
We defer the formal details and analysis to \pref{app:cor-scvx}.

\subsection{Implications to Small Loss and Gradient Variance}
\label{subsec:F_T and W_T}
In this part, we demonstrate that with careful analysis, the non-consecutive gradient variation naturally yields \mbox{gradient-variance} regret~\citep{hazan2009betterBCO,hazan2011betterBCO} and \mbox{small-loss} regret~\citep{srebro2010smoothness,orabona2012beyond}.

To start with, we present an additional smoothness assumption for small-loss bounds.
\begin{Assumption}[Appendix A of \citet{yan2024simple}]
    \label{ass:Smoothness++}
   Under the condition of $\|\nabla f_t(\x)\| \leq G$ for any $\x \in \X$ and $t \in[T]$, all online functions are $L$-smooth: $\|\nabla f_t(\x)-\nabla f_t(\y)\| \leq L\|\x-\y\|$ for any $t \in[T]$ and \smash{$\x, \y \in \X^+$}, where $\X^{+} \triangleq\{\x+\b \mid \x \in \X, \b \in G / L \cdot \B\}$ is a superset of $\X$.
\end{Assumption}
Without loss of generality, we assume $L\ge 1$ in \pref{ass:Smoothness++}, since any $L^\prime$-smooth function with $L^\prime \le L$ is also $L$-smooth. Then, we define the gradient variance $W_T$ as
\begin{equation}
    \label{eq:WT}
    W_T \define \sup_{\boldsymbol{\xi}_1,\ldots,\boldsymbol{\xi}_T \in \X} \bbr{\sumT \|\nabla f_t(\boldsymbol{\xi}_t) - \mub_T\|^2},
\end{equation}
where \smash{$\mub_T \define \frac{1}{T} \sumT \nabla f_t(\boldsymbol{\xi}_t)$} is the gradient mean.
We define the small loss $F_T$ as
\begin{equation}
    \label{eq:FT}
    F_T \define \min_{\x \in \X} \sumT f_t(\x) - \sumT \min_{\x \in \X^+} f_t(\x).
\end{equation}
We clarify that the small-loss definition here generalizes the standard one defined over non-negative functions \citep{srebro2010smoothness}.
Thus, it requires smoothness on a superset of the original domain $\X$, as shown in \pref{ass:Smoothness++}.

Due to space limitations, we only focus on how to obtain $F_T$ bounds from the non-consecutive gradient variation $\Vb_T$ here.
Specifically, we decompose $\Vb_T$ as follows:
\begin{align*}
    \E[\Vb_T]\lesssim{}& d^2\E\mbr{\sumT (\nabla_{i_t} f_t(\w_t)-\nabla_{i_t} f_{\alpha_t}(\w_{\alpha_t}))^2}\notag\\
    \le{}& 2d^2 \E\mbr{\sumT  \sbr{\nabla_{i_t} f_t(\w_t)^2+\nabla_{i_t} f_{\alpha_t}(\w_{\alpha_t})^2}}.
\end{align*}
The primary challenge lies in evaluating the expectation over \smash{$\nabla_{i_t} f_{\alpha_t}(\w_{\alpha_t})^2$}, which arises from the non-consecutive structure and the interdependence between $\alpha_t$ and $i_t$.
A simplistic way to handle the coupling between $i_t$ and $\alpha_t$ is to coarsely upper-bound the $i_t$-th entry using \smash{$w_{i_t}^2 \le \|\w\|^2$}. While this eliminates the need to take expectation over $i_t$, it results in a loose $\O(d\sqrt{F_T})$ bound. To address this, we provide a refined analysis leveraging the law of total expectation to establish \pref{lem:FT}. The proof is in \pref{app:FT_lem}.
\begin{Lemma}
    \label{lem:FT}
Under Assumptions~\ref{ass:boundedness}, \ref{ass:Lipschitzness}, \ref{ass:Smoothness++},  for convex functions, \pref{alg:chiang} enjoys
\begin{equation}
    \label{eq:FTlem}
    \E[\Vb_T]\lesssim 16dL\E\mbr{\sumT f_t(\w_t)-\sumT \min_{\x\in\X^+}f_t(\x)}.
\end{equation}
\end{Lemma}
Note that the right-hand side of \pref{eq:FTlem} can be converted to the small-loss $F_T$ using standard techniques~\citep{srebro2010smoothness,orabona2012beyond}.

The analysis for gradient variance follows an analogous approach and is thus omitted here for brevity.
To conclude, by leveraging a careful analysis of non-consecutivity, $\Vb_T$ also yields gradient-variance and small-loss bounds. We present the corresponding bounds for linear, convex, and strongly convex functions in Theorems \ref{thm:WT} and \ref{thm:FT}. 
The proofs are deferred to Appendices \ref{app:WT} and \ref{app:FT}.
\begin{Theorem}
    \label{thm:WT}
    Under Assumptions \ref{ass:boundedness}-\ref{ass:Smoothness}, denote by $\Vb_t$ the \mbox{non-consecutive} gradient variation defined in \pref{eq:VbT}.
    \begin{itemize}[left=2pt, itemsep=-2pt, topsep=-3pt]
        \item \pref{alg:chiang} with step size \smash{$\eta_t=R/\sqrt{d^2+\bar{V}_{t-1}}$} enjoys $\O\sbr{\sqrt{dW_T}+d}$ regret for linear functions and $\O\sbr{d\sqrt{W_T+d}}$ regret for convex functions.
        \item \pref{alg:chiang} with step size $\eta_t=1/(\lambda t)$ enjoys $\O\sbr{\frac{d}{\lambda}\log(d W_T)}$ regret for $\lambda$-strongly convex functions.
    \end{itemize}
\end{Theorem}
\begin{Theorem}
    \label{thm:FT}
    Under Assumptions \ref{ass:boundedness}, \ref{ass:Lipschitzness}, \ref{ass:Smoothness++}, denote by $\Vb_t$ the non-consecutive gradient variation defined in \pref{eq:VbT}.
    \begin{itemize}[left=2pt, itemsep=-2pt, topsep=-3pt]
        \item \pref{alg:chiang} with step size \smash{$\eta_t=R/\sqrt{d^2+\bar{V}_{t-1}}$} enjoys $\O\sbr{\sqrt{dF_T}+d}$ regret for convex and linear functions.
        \item \pref{alg:chiang} with step size $\eta_t=1/(\lambda t)$ enjoys $\O\sbr{\frac{d}{\lambda}\log(d F_T)}$ regret for $\lambda$-strongly convex functions.
    \end{itemize}
\end{Theorem}
Up to an \emph{additive} $\O(d)$ term, our gradient-variance bound is optimal for linear functions, while our small-loss result achieves optimality for convex functions.
Notably, these two results recover the minimax optimal \smash{$\O(\sqrt{dT})$} regret when \smash{$T \ge d$}.
For strongly convex functions, the dimension dependencies of our problem-dependent bounds align with the best known results.


\section{One-Point Bandit Linear Optimization}
\label{sec:1point}
Beyond the two-point setup, we adapt our approach, which combines the estimator-and-optimism construction in \pref{eq:chiangestimator} with a refined analysis of non-consecutive sampling gaps, to the \emph{one-point} Bandit Linear Optimization (BLO) setting.
Specifically, in one-point BLO, at each round $t \in [T]$, the learner can query only a single point $\x_t \in \X$ and observes the function value
\smash{$f_t(\x_t)\define\inner{\ellb_t}{\x_t}$},
which is revealed by an oblivious environment.

For one-point BLO, there are partial results that combine optimistic online learning and variance-reduced gradient estimators to derive \emph{gradient-variance} regret~\citep{hazan2009stochastic,hazan2011betterBCO}.
Specifically, their methods update a sequence $\{\w_t\}_{t=1}^T$ and query \smash{$\x_t=\w_t+\epsilon_t \lambda_{t,i_t}^{-\frac12}\u_{t,i_t}$}.
Here, $i_t$ is drawn uniformly from $[d]$, and $\epsilon_t$ is sampled uniformly from \smash{$\{-1,+1\}$}. $\lambda_{t,i_t}$ and $\u_{t,i_t}$ denote the $i_t$-th eigenvalue and eigenvector of the Hessian $\nabla^2 \mathcal{R}(\w_t)$ for a barrier function $\mathcal{R}(\cdot)$. The gradient estimator then takes the form of:
\begin{equation}
    \label{eq:1pvariance}
    \smash{\g_t=d\inner{\ellb_t-\tgb_t}{\x_t}\epsilon_t\lambda_{t,i_t}^{\frac12}\u_{t,i_t}+\tgb_t,}
\end{equation}
where $\tgb_t$ is an unbiased estimator of the mean loss vector, constructed by maintaining a separate reservoir of historical observations for each coordinate.
While their method is effective for obtaining gradient-variance regret bounds for one-point BLO, extending it to gradient-variation regret remains a highly non-trivial open challenge.

Inspired by the construction in \pref{eq:chiangestimator}, we design a novel gradient estimator and establish the \emph{first} gradient-variation regret bound for one-point BLO over hyper-rectangular domains, which is formally defined below.
\begin{Assumption}
    \label{ass:hyper-rectangle}
The domain $\X \subset \R^d$ is a hyper-rectangle of the form \smash{$\X = \prod_{i=1}^d [a_i, b_i]$}, where $a_i < b_i$ for all $i \in [d]$.
\end{Assumption}
The novel gradient estimator is constructed as follows:
\begin{equation}
    \label{eq:chiangestimator1p}
    \begin{gathered}
        \tgb_t = {}  \frac{1}{2}\sum_{i=1}^d \lambda_{t,i}^{\frac{1}{2}} \sbr{r^{(+1)}_i-r^{(-1)}_i} \eb_{i},\\
        \g_t = {}  d(\inner{\ellb_t}{\x_t} -  z_t)\epsilon_t\lambda_{t,i_t}^{\frac{1}{2}}\eb_{i_t} + \tgb_t,
    \end{gathered}
\end{equation}
where $i_t,\epsilon_t$ share the same definition as in \pref{eq:1pvariance}.
Here, \smash{$\mathbf{r}^{(\pm 1)} \in \mathbb{R}^d$} are two auxiliary buffers storing historical function observations associated with the two perturbation signs.
Specifically, let $\alpha_t$ be the most recent round before $t$ at which the same coordinate-sign pair was sampled, i.e., the largest integer satisfying \smash{$0 \leq \alpha_t < t$, $i_{\alpha_t} = i_t$}, and \smash{$\epsilon_{\alpha_t} = \epsilon_t$}.
Then, \smash{$z_t \define r_{i_t}^{(\epsilon_t)}=\inner{\ellb_{\alpha_t}}{\x_{\alpha_t}}$} is the stored function value from the most recent occurrence of \smash{$(i_t, \epsilon_t)$}.
Finally, $\lambda_{t,i}$ denotes the $i$-th eigenvalue of the Hessian $\nabla^2 \mathcal{R}(\w_t)$, where $\mathcal{R}(\cdot)$ is the log-barrier function defined as \smash{$\mathcal{R}(\w) = -\sum_{i=1}^d \sbr{\log(w_i - a_i) + \log(b_i - w_i)}$}.

\begin{algorithm}[t]
    \caption{Gradient-Variation One-Point BLO}
    \label{alg:1p-BLO}
    \begin{algorithmic}[1]
        \Require Step size $\eta>0$
        \State \textbf{Initialization:} $\w_1=\mathbf{0}$, \smash{$\tgb_1 = \mathbf{0}$}, \smash{$G_0=\mathbf{0}$} and buffer vectors \smash{$\r^{(+1)} = \r^{(-1)} =\mathbf{0}\in \R^{d}$}
        \For{round $t \in [T]$}
            \State Compute $\tgb_t$ as in \eqref{eq:chiangestimator1p}
            \State Choose $i_t$ uniformly from $[d]$ and $\epsilon_t$ uniformly from \smash{$\{-1,+1\}$} and fetch \smash{$z_t=r^{(\epsilon_t)}_{i_t}$}
            \State Play action $\x_t = \w_t + \epsilon_t \lambda_{t,i_t}^{-1/2}\eb_{i_t}$
            \State Observe $v_t=\inner{\ellb_t}{\x_t}$ and compute $\g_t$ as in \eqref{eq:chiangestimator1p}
            \State Update buffer vector \smash{$r^{(\epsilon_t)}_{i_t}=v_t$}
            \State Calculate \smash{$G_t=G_{t-1}+\g_t$} and update
            \begin{equation*}
               \smash{\w_{t+1} = \argmin_{\w \in \X}\bbr{\eta \left\langle G_t+\tgb_{t+1},\w\right\rangle+\mathcal{R}(\w)}}
            \end{equation*}
        \EndFor
    \end{algorithmic}
\end{algorithm}
Using \pref{eq:chiangestimator1p}, we present \pref{alg:1p-BLO} with its regret guarantee in \pref{thm:1p-BLO}. The proof is deferred to \pref{app:1point}.
\begin{Theorem}
    \label{thm:1p-BLO}
    Under Assumptions~\ref{ass:boundedness}, \ref{ass:Lipschitzness}, \ref{ass:Smoothness}, \ref{ass:hyper-rectangle}, choosing $\eta=\frac{1}{16RGd^2\sqrt{V_T\ln(2dT)}}$, with \smash{$V_T\define\sumTT\|\ellb_t-\ellb_{t-1}\|^2$}, \pref{alg:1p-BLO} satisfies:
    \begin{equation*}
        \E[\Reg_T]\le \O \Big(d^{\frac{7}{2}}\sqrt{V_T\log^2 T\log(dT)}\Big).
    \end{equation*}
\end{Theorem}
To the best of our knowledge, this is the \emph{first} gradient-variation regret bound for one-point BLO, albeit with the assumption of a hyper-rectangular feasible domain.
\begin{Remark}
    We adopt \pref{ass:hyper-rectangle} mainly for technical reasons.
    Specifically, since the Hessian of $\mathcal{R}(\cdot)$ remains diagonal for hyper-rectangular domains, its eigenvectors coincide with the standard basis vectors, which enables the optimistic term $\tgb_t$ to store historical data component-wise, making our analysis in \pref{sec:cvx-base} implementable.
\end{Remark}
\begin{Remark}
Note that \pref{alg:1p-BLO} is not directly implementable due to a causality issue: the estimator $\tgb_{t}$ depends on $\w_t$, which is in turn determined by $\tgb_{t}$.
However, we can determine $\tgb_{t}$ and $\w_t$ by solving $d$ independent equations involving strictly monotonic functions.
Consequently, $\w_t$ can be efficiently approximated via binary search to a precision of $1/T$ within $\O(\log T)$ iterations, incurring only a negligible additive $\O(1)$ term in the final regret.
We refer readers to \pref{app:practical-1p-BLO} for a practical algorithm.
\end{Remark}

\begin{Remark}
    Compared with the gradient-variance estimator of \citet{hazan2009stochastic,hazan2011betterBCO} that maintains reservoir-based historical observations to estimate the running mean of the loss vectors,
    our objective is different: gradient-variation regret concerns the cumulative differences between loss vectors, rather than their deviations from a running mean.
    Accordingly, \pref{alg:1p-BLO} stores the most recent function value for each coordinate-sign pair and uses it as a reference in our estimator~\eqref{eq:chiangestimator1p}. 
    This design yields the decomposition \smash{$\|\g_t-\tgb_t\|^2\lesssim \|\ellb_t-\ellb_{\alpha_t}\|^2+\|\x_t-\x_{\alpha_t}\|^2$}.
    The summation of the first term above gives the non-consecutive gradient variation, which can be further converted into gradient variation $V_T$, while the second term is handled through the stability analysis of the optimistic methods.
\end{Remark}
Finally, we stress that establishing \mbox{gradient-variation} regret for one-point BLO is pretty challenging, as it reduces to the open problem of obtaining squared \mbox{path-length} in multi-armed bandits~\citep{wei2018more}. 
As a first step, we demonstrate the versatility of our approach on \mbox{hyper-rectangular} domains, deferring the general case to future research.

\section{Extensions}
\label{sec:extension}
In this section, we demonstrate the effectiveness of our methods in more challenging tasks, including dynamic and universal regret minimization.

\subsection{Dynamic Regret Minimization}
\label{subsec:dynamic}
In this part, we extend our technique to the \emph{dynamic regret} setting~\citep{zinkevich2003online} in two-point BCO. Specifically, dynamic regret compares the learner's performance with a sequence of \mbox{time-varying} comparators \smash{$\{\v_t\}_{t=1}^T$} and is defined as follows:
\begin{equation}
    \label{eq:dreg}
    \DReg_T \define \frac{1}{2}\sumT (f_t(\x_{t,1})+f_t(\x_{t,2})) - \sumT f_t(\v_t).
\end{equation}
Ideally, dynamic regret guarantees scale with the path length $P_T \triangleq \sum_{t=2}^T \|\v_t - \v_{t-1}\|$, which captures the environmental non-stationarity and is unknown to the learner.
In the full-information setting, \citet{zhang2018adaptive} obtained the minimax-optimal dynamic regret bound \smash{$\O(\sqrt{T(1+P_T)})$}, which was later strengthened for smooth functions to the gradient-variation dynamic regret bound \smash{$\O(\sqrt{(1+V_T+P_T)(1+P_T)})$}~\citep{zhao2024adaptivity}. For two-point BCO, \citet{zhao2021bandit} first established a dynamic regret bound of \smash{$\O(d\sqrt{T(1+P_T)})$}, which was later sharpened to \smash{$\O(\sqrt{dT(1+P_T)})$} by \citet{he2025non}. However, dynamic gradient-variation regret remains unexplored in two-point BCO, where bandit feedback makes it difficult to exploit gradient variation while simultaneously tracking a time-varying comparator sequence.

Following \citet{zhao2020dynamic}, we adopt an online ensemble framework with a two-layer meta-base architecture to accommodate environmental non-stationarity, where multiple base learners explore the environment and a meta learner tracks the best base learner. To handle bandit feedback, we implement the base learners by running \pref{alg:chiang} with different configurations on carefully designed surrogate functions, while the meta learner uses \ohedge~\citep{rakhlin2013online} to aggregate them.

The key technical step is to use \pref{lem:bandit-VT-correction-cvx} to reduce the non-consecutive gradient variation $\Vb_T$~\eqref{eq:VbT} induced by bandit feedback to the standard gradient variation $V_T$~\eqref{eq:VT}, while allowing the ensemble analysis to adapt to the unknown path length $P_T$ in non-stationary environments. This leads to \emph{the first} gradient-variation dynamic regret guarantee with unknown path length in the two-point BCO setting. We state the resulting guarantee informally below and defer the complete version (\pref{alg:dynamic_regret}), the formal theorem, and the proof to \pref{app:dynamic-cvx}.
\begin{Theorem}[Informal]
    \label{thm:dynamic-cvx-informal}
    Under Assumptions~\ref{ass:boundedness}-\ref{ass:Smoothness}, for convex functions and oblivious comparators $\{\v_t\}_{t=1}^T$, \pref{alg:dynamic_regret} in \pref{app:dynamic-cvx} achieves the dynamic regret bound
    \begin{align*}
        \E[\DReg_T] \le \Ot\big(\sqrt{d^3(1+P_T+V_T)(1+P_T)}\big),
    \end{align*}
    without prior knowledge of the path length $P_T$.
\end{Theorem}
Compared with the \smash{$\O\big(\sqrt{(1+P_T+V_T)(1+P_T)}\big)$} bound~\citep{zhao2020dynamic} under full-information feedback, \pref{thm:dynamic-cvx-informal} achieves the same dependence on $P_T$ and $V_T$ in expectation, up to logarithmic factors. On the other hand, compared with the worst-case optimal \smash{$\O(\sqrt{dT(1+P_T)})$} bound~\citep{he2025non}, our result incurs an additional factor of $d$ in the dimension dependence. Whether this $\O(d)$ gap can be eliminated for gradient-variation dynamic regret remains an open problem.

\subsection{Universal Regret Minimization}
\label{subsec:universal}
In this part, we study \emph{universal regret}~\citep{zhao2025adaptivity} for two-point BCO, where the goal is to match the optimal performance for various function types and curvatures \emph{without} prior knowledge of these properties.
Specifically, for $\mathcal{F}_{\lin}$ (for linear functions), $\mathcal{F}_{\cvx}$ (for convex functions), and $\mathcal{F}^\lambda_{\scvx}$ (for $\lambda$-strongly convex functions), a universal online learning algorithm $\mathcal{A}$ aims to attain the following universal regret guarantee:
\begin{equation}
    \label{eq:universal-goal}
    \Reg_T(\A, \{f_t\}_{t=1}^T) \lesssim
    \begin{cases}
        \Reg_T(\A_\scvx, \F^\lambda_\scvx), \\[1mm]
        \Reg_T(\A_{\cvx}, \F_{\cvx}), \\[1mm]
        \Reg_T(\A_{\lin}, \F_{\lin}),  \\[1mm]
    \end{cases}
\end{equation}
where \smash{$\A_\scvx$, $\A_{\cvx}$, $\A_{\lin}$} are the (optimal) algorithms designed for $\F^\lambda_\scvx$, $\F_{\cvx}$, and $\F_{\lin}$, respectively. That is, $\A$ is supposed to have regret guarantees comparable to those of the optimal algorithms designed for the corresponding function classes. In OCO, universal regret has been extensively studied~\citep{van2016metagrad,wang2019adaptivity,zhang2022simple,yan2023universal}, and \citet{yan2024simple} achieved the optimal universal gradient-variation regret. However, universal gradient-variation regret remains unexplored in two-point BCO, where bandit feedback further complicates the adaptation to unknown curvature.

For adaptation to unknown curvature, we follow \citet{yan2024simple} and instantiate an online ensemble algorithm with \pref{alg:chiang} as base learners on carefully designed surrogate functions, while using \omlprod~\citep{wei2016tracking} as the meta learner on linearized losses; see \pref{alg:UniGrad-Bregman-1grad} for the complete algorithm.

To analyze this ensemble framework under bandit feedback, we further provide a new Bregman-divergence-based decomposition of the non-consecutive gradient variation $\Vb_T$~\eqref{eq:VbT} in \pref{lem:bandit-VT-bregman-cvx}, with the proof deferred to \pref{app:universal}. 
\begin{Lemma}
    \label{lem:bandit-VT-bregman-cvx}
    Under Assumptions \ref{ass:boundedness}, \ref{ass:Lipschitzness}, \ref{ass:Smoothness++}, for convex functions, \pref{alg:chiang} satisfies the following guarantee
    \begin{align*}
        \E\mbr{\Vb_T} \le{}&48d^3 L \log(dT)\cdot \E\mbr{\sumT \D_{f_t}(\ws, \w_t)}\\
        &+12d^3 V_T\log d + \O(1),
    \end{align*}
where $\D_\psi(\x, \y) \define \psi(\x) - \psi(\y) - \langle\nabla \psi(\y), \x - \y\rangle$ denotes the Bregman divergence induced by a strictly convex and differentiable function $\psi: \mathcal{X} \to \mathbb{R}$.
\end{Lemma}
Compared with $\|\w_t-\w_{t-1}\|^2$ in \pref{lem:bandit-VT-correction-cvx}, $\D_{f_t}(\ws, \w_t)$ in \pref{lem:bandit-VT-bregman-cvx} is more tractable for universal regret in the general convex case, because the corresponding negative term naturally appears in the linearization of the regret: $f_t(\w_t)-f_t(\ws)=\inner{\nabla f_t(\w_t)}{\w_t-\ws}-\D_{f_t}(\ws,\w_t).$

Using the online ensemble algorithm and the new decomposition of the non-consecutive gradient variation in \pref{lem:bandit-VT-bregman-cvx}, we establish \emph{the first} gradient-variation universal regret for linear, convex, and strongly convex functions in two-point BCO. We give an informal theorem below and defer the formal version (\pref{thm:2point-BCO-Bregman}) and its proof to \pref{app:universal}.
\begin{Theorem}[Informal]
    \label{thm:2point-BCO-Bregman-informal}
    Under Assumptions \ref{ass:boundedness}, \ref{ass:Lipschitzness}, \ref{ass:Smoothness++}, \pref{alg:UniGrad-Bregman-1grad} in \pref{app:universal} achieves the following gradient-variation universal regret:
        \begin{equation*}
        \E[\Reg_T] \le
        \begin{cases}
            \O\sbr{\frac{d}{\lambda}\log(dV_T)}, & \textnormal{($\lambda$-strongly convex)}, \\[2mm]
            \Ot\sbr{\sqrt{d^3V_T}+ d^3}, & \textnormal{(convex)},\\[2mm]
            \O\sbr{\sqrt{d^3V_T}},& \textnormal{(linear)}.
        \end{cases}
    \end{equation*}
\end{Theorem}
Compared with the curvature-aware bounds, Theorem~\ref{thm:2point-BCO-Bregman-informal} matches: \rom{1} the \smash{$\O\big(\sqrt{d^3V_T}\big)$} bound in the linear setting~\citep{chiang2013beating};
\rom{2} the \smash{$\O\big(\frac{d}{\lambda}\log(dV_T)\big)$} regret in the strongly convex case in \pref{thm:scvx-base};
and \rom{3} the \smash{$\Ot(\sqrt{d^3V_T})$} regret in the general convex case in \pref{thm:cvx-base}, up to an additive $\Ot(d^3)$ term. As a byproduct, we also derive the \emph{first} worst-case universal regret bounds for linear, convex, and strongly convex losses in two-point BCO, with the formal statement deferred to \pref{thm:2point-BCO-worst}.


\section{Conclusion}
\label{sec:conclusion}
In this work, we investigate gradient-variation regret in two-point BCO.
By providing a refined analysis of the non-consecutive structure, we achieve $\Ot(d^{\frac{3}{2}}\sqrt{V_T})$ and $\O(\frac{d}{\lambda}\log(dV_T))$ for convex and $\lambda$-strongly convex functions, improving the best known results by factors of almost $\sqrt{d}$ and $d$, respectively.
We also establish the first \mbox{gradient-variance} and \mbox{small-loss} bounds for two-point BCO, including the first problem-dependent guarantees that recover the minimax optimal $\O(\sqrt{dT})$ regret.
Furthermore, we extend our techniques to \mbox{one-point} BLO, achieving the first \mbox{gradient-variation} regret bound $\Ot(d^{\frac{7}{2}}\sqrt{V_T})$ for \mbox{hyper-rectangular} domains.
We finally validate the effectiveness of our results in more challenging tasks such as dynamic and universal regret minimization.

Several important directions remain open.
First, extending our guarantees to adaptive environments is an important direction.
Second, it remains unknown whether optimal worst-case regret bounds can be attained for exp-concave functions, another important class in online learning.
Finally, our results may also facilitate acceleration for smooth zeroth-order optimization~\citep{nesterov2017random}.

\newpage

\bibliography{online}
\bibliographystyle{icml2026}

\newpage
\onecolumn
\appendix

\section{Properties of Gradient Estimator}
\label{app:g_t}
In this section, we analyze the properties of the gradient estimator in~\pref{eq:chiangestimator}.

Recall that $\alpha_t$ is the largest integer satisfying $0 \leq \alpha_t < t$ and $i_{\alpha_t} = i_t$, where $i_t$ is the direction chosen at round $t$. In addition, recall that $v_t$ is defined as $v_{t} \define \frac{1}{2\delta}(f_t(\w_t+\delta \eb_{i_t}) - f_t(\w_t-\delta \eb_{i_t}))$. To ensure $\alpha_t$ is well-defined for the first occurrence of each direction, we introduce a dummy round $0$ in which every direction $i \in [d]$ is considered sampled. We set $f_0 \equiv 0$, so that $v_0=0$ and $\nabla f_0(\w_0)=\mathbf{0}$ for any dummy point $\w_0$.
\begin{Lemma}
    \label{lem:alpha}
    Under Assumptions \ref{ass:boundedness}-\ref{ass:Smoothness}, choosing the exploration parameter $\delta = \frac{1}{2d^2 L TR}$, the gradient estimator and the optimism defined in \pref{eq:chiangestimator} satisfy the following properties
    \begin{alignat*}{2}
        & \text{\rom{1}} \ \tilde{g}_{t, i_t}=\tilde{g}_{\alpha_t+1, i_t}={v}_{\alpha_t} \qquad \qquad  \qquad \qquad&& \text{\rom{2}} \ |v_{t}|\le G \\
        & \text{\rom{3}} \ \E[v_{t}^2]\le \frac{2G^2}{d}+\O\sbr{\tfrac{1}{d^4T^2}} && \text{\rom{4}} \ \E[{v}_{\alpha_t}^2]\le \frac{2G^2}{d}+\O\sbr{\tfrac{1}{d^4T^2}} \\
        & \text{\rom{5}} \ \|\tgb_{t}\|^2\le dG^2 && \text{\rom{6}} \ \|\g_t\|^2\le 10d^2G^2 \\
        & \text{\rom{7}} \ \E\mbr{\|\g_t\|^2}\le 15dG^2+\O\sbr{\tfrac{1}{d^2T^2}} && \text{\rom{8}} \ \left\|\g_t-\tgb_t\right\|^2\le d^2(\nabla_{i_t} f_{t}(\w_t)-\nabla_{i_t} f_{\alpha_t}(\w_{\alpha_t}))^2+\O\sbr{\tfrac{1}{T}} \\
        & \text{\rom{9}} \ \norm{\g_t-\tgb_t}^2\le 4d^2G^2. && \text{\rom{10}} \ \E\mbr{\left\|\g_t-\tgb_t\right\|^2}\le 8dG^2+\O\sbr{\tfrac{1}{d^2T^2}}\ \text{(\pref{lem:str_alpha} in \pref{sec:scvx-base})} \\
        & \text{\rom{11}} \ \|\E_t[\g_t]-\nabla f_t(\w_t)\|\le \frac{\sqrt{d}L\delta}{2}=\O\sbr{\tfrac{1}{d^{3/2}T}} &&
    \end{alignat*}
\end{Lemma}
\begin{proof}
To prove property \rom{1}, we first provide the intuition behind the construction of $\tilde{\mathbf{g}}_t$. At each round $t$, only the $i_t$-th coordinate of the optimism $\tilde{\mathbf{g}}_{t+1}$ is updated using the new directional derivative estimate $v_t$, while all other coordinates remain unchanged. Consequently, for any coordinate $i$, $\tilde{g}_{t,i}$ always stores the most recent directional derivative estimate $v_s$ from the last round $s < t$ where $i$ was selected. Specifically, for the coordinate $i_t$ chosen at round $t$, its latest update occurred at round $\alpha_t$, which implies $\tilde{g}_{t,i_t} = v_{\alpha_t}$. Furthermore, since the $i_t$-th direction is not sampled between rounds $\alpha_t+1$ and $t-1$, this coordinate stays unchanged, ensuring $\tilde{g}_{t, i_t} = \tilde{g}_{\alpha_t+1, i_t}$.

    For property \rom{2}, by $G$-Lipschitzness, we have
    \begin{equation*}
        |v_{t}| = \left|\frac{1}{2\delta}(f_t(\w_t+\delta \eb_{i_t}) - f_t(\w_t-\delta \eb_{i_t}))\right| \le G.
    \end{equation*}
    For property \rom{3}, by $L$-smoothness, for any $i\in[d]$, $t\in[T]$, and $\x\in(1-\xi)\X$, where $\xi=\delta/r$, we have
    \begin{equation}
        \label{eq:smoothness1}
        \begin{gathered}
        \sbr{f_t(\x+\delta\eb_{i})-f_t(\x)}+\sbr{f_t(\x)-f_t(\x-\delta\eb_{i})}\le 2\inner{\nabla f_t(\x)}{\delta\eb_{i}}+L\|\delta\eb_{i}\|^2,\\
        \sbr{f_t(\x+\delta\eb_{i})-f_t(\x)}+\sbr{f_t(\x)-f_t(\x-\delta\eb_{i})}\ge 2\inner{\nabla f_t(\x)}{\delta\eb_{i}}-L\|\delta\eb_{i}\|^2,
        \end{gathered}
    \end{equation}
    Therefore, the central-difference bias satisfies
    \begin{equation}
        \label{eq:central-bias}
        \left|\frac{f_t(\x+\delta\eb_i)-f_t(\x-\delta\eb_i)}{2\delta}-\nabla_i f_t(\x)\right|\le\frac{L\delta}{2}.
    \end{equation}
    By \pref{eq:smoothness1} and the inequality $(a+b)^2\le 2a^2+2b^2$, we have
    \begin{equation*}
        \E\mbr{v_{t}^2}=\E\mbr{\sbr{\frac{1}{2\delta}\sbr{f_t(\w_t+\delta \eb_{i_t}) - f_t(\w_t-\delta \eb_{i_t})}}^2}
        \le 2\E\mbr{\inner{\nabla f_t(\w_t)}{\eb_{i_t}}^2}+ \frac{L^2}{2}\delta^2\le\frac{2G^2}{d}+\O\sbr{\frac{1}{d^4T^2}}.
    \end{equation*}
    For property \rom{4}, by \pref{eq:central-bias}, we have
    \begin{equation*}
        \E\mbr{v_{\alpha_t}^2}=\E\mbr{\sbr{\frac{1}{2\delta}\sbr{f_{\alpha_t}(\w_{\alpha_t}+\delta \eb_{i_t}) - f_{\alpha_t}(\w_{\alpha_t}-\delta \eb_{i_t})}}^2}
        \le 2\E\mbr{\inner{\nabla f_{\alpha_t}(\w_{\alpha_t})}{\eb_{i_t}}^2}+ \frac{L^2}{2}\delta^2.
    \end{equation*}
    Denote $I_t=(i_1,\ldots,i_t)$ for any $t\in[T]$. Taking the expectation over $I_t$ and choosing $\c=\mathbf{0}$ in \pref{lem:exp_alphat}, we have
    \begin{align*}
        \E\mbr{\inner{\nabla f_{\alpha_t}(\w_{\alpha_t})}{ \eb_{i_t}}^2}
        = {} & \sum_{s=1}^{t-1} \frac{1}{d^2} \left( 1 - \frac{1}{d} \right)^{t-s-1}\E_{I_{s-1}}\mbr{ \|\nabla f_s(\w_s)\|^2}
        \le \frac{G^2}{d^2} \sum_{s=1}^{t-1} \left( 1 - \frac{1}{d} \right)^{t-s-1} \\
        ={}& \frac{G^2}{d^2} \sum_{k=0}^{t-2} \left( 1 - \frac{1}{d} \right)^k \le \frac{G^2}{d^2} \cdot \frac{1 - \left( 1 - \frac{1}{d} \right)^{t-1}}{1 - \left( 1 - \frac{1}{d} \right)}
        \le \frac{G^2}{d} \left[ 1 - \left( 1 - \frac{1}{d} \right)^{t-1} \right] \le \frac{G^2}{d},
    \end{align*}
    where the first inequality is by the Lipschitz assumption.

    For property \rom{5}, since for each $i\in [d]$, there exists $0\le s_i\le t-1$ such that $\gt_{t,i}=v_{s_i}$. Thus, we have
    \begin{equation*}
        \|\tgb_{t}\|^2 = \sum_{i=1}^d \gt_{t,i}^2 = \sum_{i=1}^d v_{s_i}^2 \le dG^2,
    \end{equation*}
    where the last inequality holds due to property \rom{2}.

    For property \rom{6}, we have
    \begin{equation*}
        \|\g_t\|^2 = \|d(v_t-v_{\alpha_t}) \eb_{i_t} + \tgb_t\|^2 \le 2d^2(v_t-v_{\alpha_t})^2 + 2\|\tgb_t\|^2 \le 10d^2G^2,
    \end{equation*}
    where the last inequality holds due to property \rom{2} and property \rom{5}.

    For property \rom{7}, we have
    \begin{equation*}
        \E\mbr{\|\g_t\|^2} = \E\mbr{\|d(v_t-v_{\alpha_t}) \eb_{i_t} + \tgb_t\|^2} \le 3d^2\E\mbr{v_t^2} + 3d^2\E\mbr{v_{\alpha_t}^2} + 3\E\mbr{\|\tgb_t\|^2}
        \le 15dG^2+\O\sbr{\frac{1}{d^2T^2}},
    \end{equation*}
    where the last inequality holds due to property \rom{3}, property \rom{4}, and property \rom{5}.

    For property \rom{8}, we have
    \begin{equation*}
        \|\g_t-\tgb_t\|^2 = \|d(v_t-v_{\alpha_t}) \eb_{i_t}\|^2=d^2(v_t-v_{\alpha_t})^2.
    \end{equation*}
    Define $\epsilon=|v_t-\inner{\nabla f_t(\w_t)}{\eb_{i_t}}|+|v_{\alpha_t}-\inner{\nabla f_{\alpha_t}(\w_{\alpha_t})}{\eb_{i_t}}|$. By \pref{eq:central-bias}, we have $\epsilon\le L \delta$. Thus, we have
    \begin{align*}
        &(v_t-v_{\alpha_t})^2\le (|\nabla_{i_t} f_{t}(\w_t)-\nabla_{i_t} f_{\alpha_t}(\w_{\alpha_t})|+\epsilon)^2\\
        \le{}& \sbr{\nabla_{i_t} f_{t}(\w_t)-\nabla_{i_t} f_{\alpha_t}(\w_{\alpha_t})}^2+4\epsilon G+\epsilon^2\le \sbr{\nabla_{i_t} f_{t}(\w_t)-\nabla_{i_t} f_{\alpha_t}(\w_{\alpha_t})}^2+\O\sbr{\frac{1}{d^2T}}.
    \end{align*}
    For property \rom{9}, we have
    \begin{equation*}
        \|\g_t-\tgb_t\|^2 = \|d(v_t-v_{\alpha_t}) \eb_{i_t}\|^2
        \le4d^2G^2,
    \end{equation*}
    where the last inequality holds due to property \rom{2}.

    For property \rom{10}, we have
    \begin{equation*}
        \E\mbr{\|\g_t-\tgb_t\|^2} = \E\mbr{\|d(v_t-v_{\alpha_t}) \eb_{i_t}\|^2} \le 2d^2\E\mbr{v_t^2} + 2d^2\E\mbr{v_{\alpha_t}^2} \le 8dG^2+\O\sbr{\frac{1}{d^2T^2}},
    \end{equation*}
    where the last inequality holds due to property \rom{3} and property \rom{4}.

    For property \rom{11}, by \pref{eq:central-bias}, we have for $i\in[d]$
    \begin{align*}
        \abs{\frac{1}{2\delta}\sbr{f_t(\w_t+\delta\eb_i)-f_t(\w_t-\delta\eb_i)}-\nabla_i f_t(\w_t)}\le L\delta/2.
    \end{align*}
 Since $\w_t$ and $\tgb_t$ are fixed under $\E_t[\cdot]$, we have
    \begin{align*}
        \E_t[\g_t]={}& \E_t\mbr{d(v_t-\gt_{t,i_t}) \eb_{i_t} + \tgb_t} = d\E_t\mbr{v_t \eb_{i_t}},
    \end{align*}
    where the first step uses $d\E_t[\gt_{t,i_t}\eb_{i_t}]=\tgb_t$ because $i_t$ is uniformly sampled from $[d]$. Thus, we have
    \begin{align*}
        \|\E_t[\g_t]-\nabla f_t(\w_t)\| \le \sqrt{d} \max_{i\in[d]} \abs{\frac{1}{2\delta}\sbr{f_t(\w_t+\delta\eb_i)-f_t(\w_t-\delta\eb_i)}-\nabla_i f_t(\w_t)} 
        \le \frac{\sqrt{d}L\delta}{2}=\O\sbr{\frac{1}{d^{3/2}T}}.
    \end{align*}
\end{proof}

Finally, we prove a useful technical lemma that was used throughout the proof.
\begin{Lemma}
\label{lem:exp_alphat}
    With the same assumptions and parameters as in \pref{lem:alpha}, for any $t\in[T]$, $i\in[d]$ and constant vector $\c\in \R^d$, \pref{alg:chiang} ensures that
    \begin{equation*}
        \E\mbr{\inner{\nabla f_{\alpha_t}(\w_{\alpha_t})-\c }{\eb_{i_t}}^2}\le \sum_{s=1}^{t-1} \frac{1}{d^2} \left( 1 - \frac{1}{d} \right)^{t-s-1}\E_{I_{s-1}}\mbr{ \|\nabla f_s(\w_s)-\c\|^2}+\frac{1}{d}\left(1-\frac{1}{d}\right)^{t-1} \|\c\|^2,
    \end{equation*}
    where $I_{s-1} \define (i_1,\ldots,i_{s-1})$ is the history of previously selected coordinates.
\end{Lemma}
\begin{proof}
    Denoting by $\I(\cdot)$ the indicator function, we have
\begin{align*}
    &\E\mbr{\inner{\nabla f_{\alpha_t}(\w_{\alpha_t})-\c}{ \eb_{i_t}}^2}=\E_{I_t}\mbr{\sum_{s=0}^{t-1} \sum_{i=1}^d \I(i_t=i, \alpha_t=s) \left\langle \nabla f_s(\w_s)-\c, \eb_i \right\rangle^2}\\
={}&\sum_{s=0}^{t-1} \sum_{i=1}^d\E_{I_t}\mbr{\I(i_t=i, \alpha_t=s) \inner{\nabla f_s(\w_s)-\c}{ \eb_i}^2} \\
={}&\sum_{s=0}^{t-1} \sum_{i=1}^d\E_{I_{s-1}}\mbr{\E_{i_s,\ldots,i_t}\mbr{\I(i_t=i, \alpha_t=s) \inner{\nabla f_s(\w_s)-\c}{ \eb_i}^2\given I_{s-1}}} \\
={} &\sum_{s=0}^{t-1} \sum_{i=1}^d \P(i_t=i, \alpha_t=s) \E_{I_{s-1}}\mbr{\left\langle \nabla f_s(\w_s)-\c, \eb_i \right\rangle^2}\\
={}& \sum_{s=1}^{t-1} \sum_{i=1}^d \left[ \frac{1}{d^2} \left( 1 - \frac{1}{d} \right)^{t-s-1} \right]\E_{I_{s-1}}\mbr{ \left\langle \nabla f_s(\w_s)-\c, \eb_i \right\rangle^2}+  \frac{1}{d}\left(1-\frac{1}{d}\right)^{t-1} \|\nabla f_0(\w_0)-\c\|^2\\
= {}&\sum_{s=1}^{t-1} \frac{1}{d^2} \left( 1 - \frac{1}{d} \right)^{t-s-1}\E_{I_{s-1}} \mbr{ \sum_{i=1}^d \left\langle \nabla f_s(\w_s)-\c, \eb_i \right\rangle^2 }+\frac{1}{d}\left(1-\frac{1}{d}\right)^{t-1} \|\c\|^2,
\end{align*}
    where the first step uses the law of total expectation to sum over all possible realizations of the random coordinate $i_t$ and the last update time $\alpha_t$, the third step uses the tower property $\E[\cdot] = \E_{I_{s-1}}[\E[\cdot \given I_{s-1}]]$, the fourth step follows because $\I(i_t=i, \alpha_t=s)$ only depends on $i_s,\ldots,i_t$ and, given $s$, $\w_s$ only depends on $I_{s-1}$, the fifth step substitutes the joint probability $\P(i_t=i, \alpha_t=s) = \frac{1}{d^2}(1-\frac{1}{d})^{t-s-1}$ for $s>0$, which is derived from the independent uniform sampling of coordinates, and the last step uses $\nabla f_0(\w_0)\define\mathbf{0}$.
\end{proof}

\section{Analysis of Non-Consecutive Sampling Gap}
\label{app:rho}

In this section, we analyze the stochastic behavior of the non-consecutive sampling gap $\rho_{t,i}$. At each time step $t \in [T]$, a sample $i_t$ is drawn independently and uniformly from $[d]$. Then, the gap $\rho_{t,i}$ is defined as:
\begin{equation}
    \label{eq:rho}
    \rho_{t,i} \define \tau_2 - \tau_1\text{, where } \tau_1 \define \max\{ 0\le\tau < t\given i_\tau = i\}, \text{ and } \tau_2\define \min \{ t\le\tau\le T \given i_\tau = i\}.
\end{equation}
For completeness, we set $\tau_2 = T+1$ if coordinate $i$ is never sampled at or after time $t$ and $\tau_1=0$ if coordinate $i$ is never sampled before time $t$. Intuitively, $\rho_{t,i}$ quantifies the duration between the most recent sampling of coordinate $i$ before time $t$ and its next sampling at or after time $t$.

We now establish the connection between $\rho_{t,i}$, the geometric distribution, and the Coupon Collector's Problem (CCP).
\paragraph{Relation to Geometric Distribution.}
To start, we analyze $\rho_{t,i}+1$ by decomposing it into two components, $t - \tau_1$ and $\tau_2 - t+1$, which represent the waiting times of two independent sampling processes:
 \begin{itemize}
    \item  \textbf{Backward Search:} Starting from time $t$, we search backward through the sequence $X_{t-1}, X_{t-2}, \dots$ until coordinate $i$ is first encountered at index $\tau_1$. The search stops exactly at this hit, making $t - \tau_1$ the search duration.
    \item \textbf{Forward Search:} Similarly, we search forward from time $t$ through $X_t, X_{t+1}, \dots$ until the next occurrence of $i$ at index $\tau_2$. The forward duration is $\tau_2 - t+1$.
 \end{itemize}

Since each $i_\tau$ is sampled uniformly from $[d]$, the waiting times $t - \tau_1$ and $\tau_2 - t + 1$ are independent truncated geometric random variables with success probability $p = 1/d$. Specifically, $t-\tau_1$ is truncated at $t$, and $\tau_2-t+1$ is truncated at $T-t+2$, with the last value corresponding to no future occurrence. Consequently, the analysis of $\rho_{t,i}$ reduces to studying the sum of two independent geometric random variables subject to boundary constraints.

\paragraph{Relation to Coupon Collector's Problem.}

While the geometric distribution characterizes the gap for a fixed coordinate $i$, the term $\max_{i \in [d]} \rho_{t,i}$ is naturally captured by the Coupon Collector's Problem. In the standard CCP, a collector seeks to obtain all $d$ distinct types of coupons, where each draw yields type $i \in [d]$ uniformly at random with probability $1/d$. The primary object of interest is the \emph{stopping time}, defined as the total number of draws required to complete the collection.

In our context, $\max_{i \in [d]} (t - \tau_1)$ corresponds to the time required to observe every coordinate at least once when searching backward from $t$. Similarly, $\max_{i \in [d]} (\tau_2 - t+1)$ represents the time required to encounter every coordinate when searching forward. Thus, both terms behave as the waiting time in a standard CCP with $d$ bins truncated by the available time horizon.

The following lemma summarizes the key properties of these distributions:
\begin{Lemma}
    \label{lem:Geo}
Consider a process where in each round $t=1, 2, \ldots$, a number $X_t$ is sampled uniformly and independently from $[n]$.
    \begin{itemize}
        \item  For any specific $i \in [n]$, let $T_i = \min \{t \given X_t = i\}$ be the first time $i$ is sampled. Then $\E[T_i] = n$.
        \item Let $T_{\text{all}} = \max_{i \in [n]} T_i$ be the time when the last remaining number in $[n]$ is sampled (i.e., the time to collect all numbers). Then $\E[T_{\text{all}}] \le 2n \log  n$ (for $n \ge 3$).
    \end{itemize}
\end{Lemma}
By \pref{lem:Geo}, setting $n=d$, we have directly the following two properties:
\begin{gather}
    \E[\rho_{t,i}]\le2 \E[T_i]-1\le 2d. \label{eq:1rho}\\
    \E\mbr{\max_{i\in[d]}\rho_{t,i}}\le2 \E\mbr{\max_{i\in[d]}T_i}-1\le 4d\log  d. \label{eq:max_rho}
\end{gather}
Finally, we provide the proof of \pref{lem:Geo}.
\begin{proof}
To start with, we prove $\E[T_i] = n$ using the properties of the geometric distribution.

Fix a specific target number $i \in [n]$. In any round $t$, the probability of sampling $i$ is $p = \frac{1}{n}$. Since the samples $X_t$ are drawn independently, the waiting time $T_i$ for the first occurrence of $i$ follows a geometric distribution with success probability $p$.
The expectation of a geometric random variable with parameter $p$ is $1/p$. Therefore:
\begin{equation*}
    \E[T_i] = \frac{1}{p} = \frac{1}{1/n} = n.
\end{equation*}
Next, we prove the second property, which is equivalent to the classical \textit{Coupon Collector's Problem}. We decompose the total waiting time $T_{\text{all}}$ into phases. Let $S_k$ denote the number of \textit{additional} rounds required to collect the $k$-th distinct number, given that $k-1$ distinct numbers have already been collected. The total time is the sum of the times in each phase:
\begin{equation*}
    T_{\text{all}} = \sum_{k=1}^{n} S_k.
\end{equation*}
Consider the phase where $k-1$ distinct numbers have already been found. There are $n - (k-1)$ "unseen" numbers remaining in the pool of size $n$. The probability $p_k$ of drawing a \textit{new} number in any single trial during this phase is:
\begin{equation*}
    p_k = \frac{n - (k - 1)}{n}.
\end{equation*}
The number of trials $S_k$ in this phase follows a geometric distribution with parameter $p_k$. Thus, its expected value is:
\begin{equation*}
    \E[S_k] = \frac{1}{p_k} = \frac{n}{n - k + 1}.
\end{equation*}
By the linearity of expectation, the total expected time is:
\begin{equation}
    \E[T_{\text{all}}] = \sum_{k=1}^{n} \E[S_k] = \sum_{k=1}^{n} \frac{n}{n - k + 1} = n \sum_{j=1}^{n} \frac{1}{j} \label{eq:harmonic_sum},
\end{equation}
where the last step follows from the change of index $j = n - k + 1$. The sum $H_n = \sum_{j=1}^n \frac{1}{j}$ is the $n$-th Harmonic number, which satisfies the bound $H_n \le \log  n + 1$ for all $n \ge 1$. Substituting this into Eq.~\eqref{eq:harmonic_sum}:
\begin{equation*}
    \E[T_{\text{all}}] \le n (\log  n + 1) = n \log  n + n\le 2n\log n,
\end{equation*}
which completes the proof.
\end{proof}


\section{Omitted Proofs in \pref{sec:two-point}}
\label{app:base-proof}
In this section, we provide the detailed proofs for the results presented in \pref{sec:two-point}.

\subsection{Proof of \pref{lem:decom of VbT}}
\label{app:lem0}
\begin{proof}
By property \rom{8} of \pref{lem:alpha}, we have
\begin{equation}
    \Vb_T=\sum_{t=1}^T \|\g_t-\tgb_t\|^2\le d^2\sum_{t=1}^T(\nabla_{i_t} f_{t}(\w_t)-\nabla_{i_t} f_{\alpha_t}(\w_{\alpha_t}))^2+\O(1).\label{eq:lem0eq1}
\end{equation}
Then, we bound $(\nabla_{i_t} f_{t}(\w_t)-\nabla_{i_t} f_{\alpha_t}(\w_{\alpha_t}))^2$ as follows:
\begin{align*}
    &(\nabla_{i_t} f_{t}(\w_t)-\nabla_{i_t} f_{\alpha_t}(\w_{\alpha_t}))^2=\sbr{\sum_{s=\alpha_t+1}^{t} \nabla_{i_t} f_{s}(\w_s)-\nabla_{i_t} f_{s-1}(\w_{s-1})}^2\\
    \le {}&(t-\alpha_t)\sum_{s=\alpha_t+1}^{t}\sbr{ \nabla_{i_t} f_{s}(\w_s)-\nabla_{i_t} f_{s-1}(\w_{s-1})}^2
    \le \sum_{s=\alpha_t+1}^{t}\rho_{s,i_t}\sbr{ \nabla_{i_t} f_{s}(\w_s)-\nabla_{i_t} f_{s-1}(\w_{s-1})}^2\\
    \le {}& 2\sum_{s=\alpha_t+1}^{t}\rho_{s,i_t}\sbr{ \nabla_{i_t} f_{s}(\w_s)-\nabla_{i_t} f_s(\w_{s-1})}^2+2\sum_{s=\alpha_t+1}^{t}\rho_{s,i_t}\sbr{ \nabla_{i_t} f_s(\w_{s-1})-\nabla_{i_t} f_{s-1}(\w_{s-1})}^2,
\end{align*}
where the second step is by the Cauchy-Schwarz inequality and the third step uses $t-\alpha_t\le \rho_{s,i_t}$ for $\alpha_t< s\le t$ according to the definition of the non-consecutive sampling gap $\rho_{s,i}$ in \pref{eq:rho}.

Summing over $t=1,\ldots,T$, we have
\begin{align}
    &\sumT(\nabla_{i_t} f_{t}(\w_t)-\nabla_{i_t} f_{\alpha_t}(\w_{\alpha_t}))^2\notag\\
    \le{}& 2\sumT\sum_{s=\alpha_t+1}^{t}\rho_{s,i_t}\sbr{\sbr{ \nabla_{i_t} f_{s}(\w_s)-\nabla_{i_t} f_s(\w_{s-1})}^2+\sbr{ \nabla_{i_t} f_s(\w_{s-1})-\nabla_{i_t} f_{s-1}(\w_{s-1})}^2} \notag\\
={}& 2\sum_{i=1}^d\sum_{\substack{t:i_t=i \\ 1 \le t \le T}}\sum_{s=\alpha_t+1}^{t}\rho_{s,i}\sbr{\sbr{ \nabla_{i} f_{s}(\w_s)-\nabla_{i} f_s(\w_{s-1})}^2+\sbr{ \nabla_{i} f_s(\w_{s-1})-\nabla_{i} f_{s-1}(\w_{s-1})}^2}\notag \\
\le{}& 2\sum_{i=1}^{d}\sum_{s=1}^T\rho_{s,i}\sbr{ \nabla_{i} f_s(\w_{s-1})-\nabla_{i} f_{s-1}(\w_{s-1})}^2+2\sum_{i=1}^{d}\sum_{s=1}^T\rho_{s,i}\sbr{ \nabla_{i} f_s(\w_s)-\nabla_{i} f_s(\w_{s-1})}^2\notag\\
={}& 2\sumT\sum_{i=1}^{d}\rho_{t,i}\sbr{ \nabla_{i} f_t(\w_{t-1})-\nabla_{i} f_{t-1}(\w_{t-1})}^2+2\sumT\sum_{i=1}^{d}\rho_{t,i}\sbr{ \nabla_{i} f_{t}(\w_t)-\nabla_{i} f_t(\w_{t-1})}^2,\label{eq:lem0decom}
\end{align}
where the second step follows from changing the order of summation, based on the observation that for each fixed $s$ and $i$, there is at most one $t$ such that $i_t=i$ and $s\in(\alpha_t,t]$, and the last step is obtained by changing the order of summation over $i$ and $s$, followed by relabeling the index $s$ as $t$.

Plugging \pref{eq:lem0decom} back to \pref{eq:lem0eq1}, we finish the proof.
\end{proof}

\subsection{Proof of \pref{lem:bandit-VT-correction-cvx}}
\label{app:lem1}
\begin{proof}
Following the proof of \pref{lem:decom of VbT}, we have
\begin{align*}
 &\sumT(\nabla_{i_t} f_{t}(\w_t)-\nabla_{i_t} f_{\alpha_t}(\w_{\alpha_t}))^2\\\le{}& 2\underbrace{\sumT\sum_{i=1}^{d}\rho_{t,i}\sbr{ \nabla_{i} f_t(\w_{t-1})-\nabla_{i} f_{t-1}(\w_{t-1})}^2}_{\term{a}}
 +2\underbrace{\sumT\sum_{i=1}^{d}\rho_{t,i}\sbr{ \nabla_{i} f_{t}(\w_t)-\nabla_{i} f_t(\w_{t-1})}^2}_{\term{b}}
\end{align*}

For \term{a}, we have
\begin{align}
    &\E[\term{a}]
    \le \E\mbr{\sum_{t=1}^{T}\sum_{i=1}^{d}\sbr{\max_{i\in[d]}\rho_{t,i}}\sbr{ \nabla_{i} f_t(\w_{t-1})-\nabla_{i} f_{t-1}(\w_{t-1})}^2} \notag\\
    ={}&\E\mbr{\sum_{t=1}^{T}\sbr{\max_{i\in[d]}\rho_{t,i}}\|\nabla f_t(\w_{t-1})-\nabla f_{t-1}(\w_{t-1})\|^2}
    \le \E\mbr{\sum_{t=1}^{T}\sbr{\max_{i\in[d]}\rho_{t,i}}\sup_{\x\in\X}\|\nabla f_{t}(\x)-\nabla f_{t-1}(\x)\|^2}\notag\\
    =&\sum_{t=1}^{T}\E\mbr{\max_{i\in[d]}\rho_{t,i}}\sup_{\x\in\X}\|\nabla f_{t}(\x)-\nabla f_{t-1}(\x)\|^2
    \le4d\log  d V_T.
    \label{eq:lem1 a}
\end{align}
where the first and third steps are by taking the maximum, the second step is obtained by summing across all coordinates $i \in [d]$, and the last step follows from the property in \pref{eq:max_rho}.

For \term{b}, we have
\begin{align}
    \E[\term{b}]
    \le{}4dL^2\log(dT)\E\mbr{\sum_{t=1}^{T}\|\w_t-\w_{t-1}\|^2}+\O(1).
    \label{eq:lem1 b}
\end{align}
The proof of \pref{eq:lem1 b} is as follows. Let $\bar{\rho}=4d\log(dT)$ and let $Q$ denote the bad event that $\rho_{t, i}>\bar{\rho}$ for some $t \in[T]$ and $i \in[d]$. If $Q$ occurs, then for some coordinate $i$, there is a block of at least $\lfloor\bar{\rho}\rfloor$ consecutive rounds in which coordinate $i$ is not sampled. Therefore, by a union bound,
\begin{equation}
    \label{eq:high_prob_rho}
    \mathbb{P}(Q) \leq d(T+1)\left(1-\frac{1}{d}\right)^{\lfloor\bar{\rho}\rfloor}
    \le d(T+1)\exp\left(-\frac{\bar{\rho}-1}{d}\right)
    \leq \frac{4}{d^3T^3}.
\end{equation}
Using the smoothness of $f_t$ and the deterministic bound $\rho_{t,i}\le T+1$, we have
\begin{align*}
    \E[\term{b}\I(Q)]
    \le{}& \mathbb{P}(Q)\cdot (T+1)\sup_{\w_0,\ldots,\w_T\in\X}
\sum_{t=1}^T \|\nabla f_t(\w_t)-\nabla f_t(\w_{t-1})\|^2\\
    \le{}& \frac{4}{d^3T^3}\cdot (T+1)\cdot 4L^2R^2T\le \O(1),
\end{align*}
where $\I(\cdot)$ denotes the indicator function and the last step uses $\|\w_t-\w_{t-1}\|\le 2R$.

When $\neg Q$ occurs, we have $\rho_{t,i}\le\bar{\rho}$ for all $t\in[T]$ and $i\in[d]$. Therefore,
\begin{align*}
    \term{b}\I(\neg Q)
    \le{}& \bar{\rho}\sum_{t=1}^{T}\sum_{i=1}^{d}\sbr{ \nabla_{i} f_{t}(\w_t)-\nabla_{i} f_t(\w_{t-1})}^2\\
    \le{}& 4dL^2\log(dT)\sum_{t=1}^{T}\|\w_t-\w_{t-1}\|^2.
\end{align*}
Combining the above two inequalities for the events $Q$ and $\neg Q$, we obtain \pref{eq:lem1 b}.
Finally, combining \pref{eq:lem1 a} and \pref{eq:lem1 b} with \pref{eq:lem0eq1}, we complete the proof.
\end{proof}

\subsection{Proof of \pref{thm:cvx-base}}
\label{app:cvx-base}
In this part, we provide the proof of \pref{thm:cvx-base}.
To begin with, we first restate a useful lemma from \citet{chiang2013beating} for self-containedness.
For simplicity, we denote $\w^\star\in\argmin_{\w\in(1-\xi)\X} \sum_{t=1}^T f_t(\w)$.
\begin{Lemma}[{Lemma 2 of \citet{chiang2013beating}}]
    \label{lem:domain_transition}
    Under Assumptions \ref{ass:boundedness}-\ref{ass:Lipschitzness}, for convex functions, \pref{alg:chiang} enjoys
    \begin{equation*}
        \sum_{t=1}^T \frac{1}{2} \sbr{f_t(\x_{t})+f_t(\x_{t}^\prime)}-\min_{\x\in \X}\sum_{t=1}^{T}f_t(\x) \le \sum_{t=1}^{T} f_t(\w_t)-\sum_{t=1}^{T} f_t(\w^\star)+\O(1).
    \end{equation*}
\end{Lemma}
\begin{proof}[of \pref{thm:cvx-base}]
    First, using the standard analysis for \OOGD~\citep{chiang2012online}, we have
    \begin{align}
    \E[\Reg_T]\le {} & \E\mbr{\sum_{t=1}^T f_t(\w_t)-\sum_{t=1}^T f_t(\w^\star)}+\O(1) \tag*{(by \pref{lem:domain_transition})}\\
    \le{}& \E\mbr{\sum_{t=1}^{T}\inner{\nabla f_t(\w_t)}{\w_t-\w^\star}}+\O(1)\le \E\mbr{\sumT\inner{\g_t}{\w_t-\w^\star}}+\O(1)\notag\\
    \le{}&\underbrace{\E\mbr{\sum_{t=1}^T \eta_t\left\|\mathbf{g}_t-\tgb_t\right\|^2}}_{\term{a}}
    +\underbrace{\E\mbr{\sum_{t=1}^T \frac{1}{2\eta_t}\sbr{\|\w^\star-\widehat{\w}_t\|^2-\|\w^\star-\widehat{\w}_{t+1}\|^2}}}_{\term{b}}\notag\\
    & \qquad -\underbrace{\E\mbr{\sum_{t=1}^T \frac{1}{2\eta_t}\sbr{\|\widehat{\w}_{t+1}-\w_t\|^2+\|\w_t-\widehat{\w}_t\|^2}}}_{\term{c}}+\O(1),\label{eq:OMD_framework}
\end{align}
where the second step is by convexity and the third step is by property~\rom{11} in \pref{lem:alpha}.

\paragraph{Problem-Dependent Regret.}
For \term{a}, we have the following bound:
\begin{align}
    \term{a}\le{}& R\E\mbr{\sum_{t=1}^T \frac{\|\g_t-\tgb_t\|^2}{\sqrt{1152d^3R^4L^2\log(dT)+\bar{V}_{t-1}}}}\notag \\
    \le{}& 4R\E\mbr{\sqrt{1152d^3R^4L^2\log(dT)+\sumT \|\g_t-\tgb_t\|^2}+\max_{t\in[T]}\frac{\|\g_t-\tgb_t\|^2}{\sqrt{1152d^3R^4L^2\log(dT)}}}\notag\\
    \le{}&4R\E\mbr{\sqrt{1152d^3R^4L^2\log(dT)+\Vb_T}}+\O\sbr{\sqrt{d}}, \tag{by \pref{lem:alpha}}
\end{align}
where the first step is by the definition of $\eta_t$ and the second inequality is by choosing $a_s=\frac{\|\g_s-\tgb_s\|^2}{1152d^3R^4L^2\log(dT)}$ in \pref{lem:sum_cvx}.

For \term{b}, we have
\begin{align*}
    \term{b}\le {}& \E\mbr{\frac{\|\ws-\wh_1\|^2}{2\eta_1}+\sum_{t=2}^T \sbr{\frac{1}{2\eta_t}-\frac{1}{2\eta_{t-1}}}\|\ws-\wh_t\|^2}\le \E\mbr{\frac{2R^2}{2\eta_1}+\sum_{t=2}^T \sbr{\frac{2R^2}{\eta_t}-\frac{2R^2}{\eta_{t-1}}}} \\
    \le{}& \E\mbr{\frac{2R^2}{\eta_1}+\frac{2R^2}{\eta_T}}= 2R\E\mbr{\sqrt{1152d^3R^4L^2\log(dT)+\Vb_T}}+\O(\sqrt{d^3\log(dT)}),
\end{align*}
where the first step is by changing the order of summation and the second step is by the boundedness.

For \term{c}, we have
\begin{equation*}
    \term{c}\ge \E\mbr{\sumTT \frac{1}{2\eta_{t-1}}\|\w_t-\wh_t\|^2+\|\w_{t-1}-\wh_t\|^2}\ge \E\mbr{\sum_{t=2}^T \frac{1}{4\eta_{t-1}}\|\w_t-\w_{t-1}\|^2}.
\end{equation*}
Thus, we have
\begin{align}
    \label{eq:reg_bound_mid}
    \E[\Reg_T]\le{}& \E\mbr{6R\sqrt{1152d^3R^4L^2\log(dT)+\Vb_T}-\sum_{t=2}^T \frac{1}{4\eta_{t-1}}\|\w_t-\w_{t-1}\|^2}+\O\sbr{\sqrt{d^3\log(dT)}}\notag\\
    \le{}&6R\sqrt{1152d^3R^4L^2\log(dT)+\E\mbr{\Vb_T}}-\E\mbr{\sum_{t=2}^T \frac{1}{4\eta_{t-1}}\|\w_t-\w_{t-1}\|^2}+\O\sbr{\sqrt{d^3\log(dT)}}.
\end{align}
where the last inequality is by Jensen's inequality.

Using \pref{lem:bandit-VT-correction-cvx}, we have
\begin{align*}
    \sqrt{1152d^3R^4L^2\log(dT)+\E\mbr{\Vb_T}}
    &\le\sqrt{1152d^3R^4L^2\log(dT)+2d^2\E\mbr{\sum_{t=1}^{T}\sum_{i=1}^{d}\rho_{t,i}\sbr{ \nabla_{i} f_t(\w_{t-1})-\nabla_{i} f_{t-1}(\w_{t-1})}^2}}\\
    &\qquad\qquad+\sqrt{2d^2\E\mbr{\sum_{t=1}^{T}\sum_{i=1}^{d}\rho_{t,i}\sbr{ \nabla_{i} f_{t}(\w_t)-\nabla_{i} f_t(\w_{t-1})}^2}}\\
    \le{}&\sqrt{1152d^3R^4L^2\log(dT)+8d^3 V_T\log d}+\sqrt{8d^3L^2\log(dT)\E\mbr{\sum_{t=1}^{T}\|\w_t-\w_{t-1}\|^2}}+\O(1).
\end{align*}

Plugging it back to \eqref{eq:reg_bound_mid}, we have
\begin{align*}
  \E[\Reg_T]\le {} & \E\mbr{6R\sqrt{1152d^3R^4L^2\log(dT)+\Vb_T}-\sum_{t=2}^T \frac{1}{4\eta_{t-1}}\|\w_t-\w_{t-1}\|^2}+\O\sbr{\sqrt{d^3\log(dT)}}\\
  \le{}&6R\sqrt{1152d^3\sbr{R^4L^2\log(dT)+8V_T\log d}}+6R\sqrt{8d^3L^2\log(dT)\E\mbr{\sum_{t=1}^{T}\|\w_t-\w_{t-1}\|^2}}\\
  &\qquad\qquad -\frac{RL\sqrt{1152d^3\log(dT)}}{4}\E\mbr{\sum_{t=2}^T \|\w_t-\w_{t-1}\|^2} +\O\sbr{\sqrt{d^3\log(dT)}}\\
  \le{}&\O\sbr{\sqrt{d^3\log(dT)+d^3 V_T\log d}},
\end{align*}
where the last inequality is by AM-GM inequality.
\paragraph{Problem-Independent Regret.}
Using property \rom{10} of \pref{lem:alpha}, we have $\E\mbr{\Vb_T}\le 8dG^2T+\O(1)$.
Plugging it back to \eqref{eq:reg_bound_mid},  we have
\begin{align*}
\E[\Reg_T]\le6R\sqrt{1152d^3R^4L^2\log(dT)+8dG^2T}+\O\sbr{\sqrt{d}}=\O\sbr{\sqrt{d^3\log(dT)+dT}},
\end{align*}
which completes the proof.
\end{proof}

\subsection{Proof of \pref{thm:scvx-base}}
\label{app:scvx-base}
\begin{proof}
    To start, by the standard \OOGD analysis, e.g., {\citet[Lemma 12]{yan2023universal}}, we have
    \begin{align}
        \E[\Reg_T]\le {} & \E\mbr{\sum_{t=1}^T f_t(\w_t)-\sum_{t=1}^T f_t(\ws)}+\O(1)
        \le \E\mbr{\sum_{t=1}^{T}\inner{\nabla f_t(\w_t)}{\w_t-\ws}-\frac{\lambda}{2} \|\w_t-\ws\|^2}+\O(1) \notag\\
        \le{}& \E\mbr{\sumT\inner{\g_t}{\w_t-\ws}-\frac{\lambda}{2} \|\w_t-\ws\|^2}+\O(1) \notag\\
        \le{}& 2\underbrace{\E\mbr{\sum_{t=1}^T \eta_t\left\|\g_t-\tgb_t\right\|^2}}_{\term{a}}
        -\underbrace{\E\mbr{\sum_{t=1}^T \frac{1}{2\eta_t}\sbr{\|\wh_{t+1}-\w_t\|^2+\|\w_t-\wh_t\|^2}}}_{\term{b}}+\O(1), \label{eq:scvx-omd}
    \end{align}
  where the third inequality is by property~\rom{11} in \pref{lem:alpha}.

    For \term{a}, we have
    \begin{align}
        \term{a} \le {} & \O\sbr{\frac{\max_{t\in[T]}\E\mbr{\|\g_t-\tgb_t\|^2}}{\lambda}\log \sbr{\lambda\E\mbr{\sum_{t=1}^T\|\g_t-\tgb_t\|^2}}}\label{eq:reg_bound_mid_scvx}\\
        \le{}& \O\sbr{\frac{d}{\lambda}\log \sbr{\lambda\E[\Vb_T]}}\notag\\
        \le{}&\O\sbr{\frac{d}{\lambda}\log \sbr{\lambda d^3\log d V_T +\lambda d^3 \log(dT)\E\mbr{\sumT \|\w_t-\w_{t-1}\|^2 }}}, \notag
    \end{align}
    where the first step is by choosing $a_t=\E\mbr{\|\g_t-\tgb_t\|^2}$ in \pref{lem:sum_scvx} for any $t\in[T]$, the second step is by property~\rom{10} in \pref{lem:alpha}, and the last step is by \pref{lem:bandit-VT-correction-cvx}.

    Consequently, we obtain the following regret bound:
    \begin{equation*}
        \E[\Reg_T]\le \O\sbr{\frac{d}{\lambda}\log \sbr{\lambda d^3 V_T\log d +\lambda d^3 \log(dT)\E\mbr{\sumT \|\w_t-\w_{t-1}\|^2 }}}-\term{b}.
    \end{equation*}
    Following the analysis in \citet[Theorem 16]{chiang2013beating}, we analyze two cases: \rom{1} $\E[\sumT \|\w_t-\w_{t-1}\|^2]=\O(1)$; and \rom{2} $\E[\sumT \|\w_t-\w_{t-1}\|^2]=\omega(1)$. For the first case,
    \begin{equation*}
        \E[\Reg_T]\le \O\sbr{\frac{d}{\lambda}\log \sbr{\lambda d^3\log d V_T+\lambda d^3 \log(dT)}}=\O\sbr{\frac{d}{\lambda}\log \sbr{dV_T}},
    \end{equation*}
    where we omit the $\O(\log \log T)$ term.
    For the second case, we have
    \begin{align*}
       \E[\Reg_T]\le{}& \O\sbr{\frac{d}{\lambda}\log \sbr{\lambda d^3\sbr{V_T\log d+\log(dT)}}}+\O\sbr{\frac{d}{\lambda}\log \sbr{\E\mbr{\sumT \|\w_t-\w_{t-1}\|^2}}}-\term{b}\\
       \le{}& \O\sbr{\frac{d}{\lambda}\log \sbr{dV_T}}+\O\sbr{\frac{d}{\lambda}\log \sbr{\E\mbr{\sumT \|\w_t-\w_{t-1}\|^2}}}-\frac{\lambda}{4}\E\mbr{\sumT \|\w_t-\w_{t-1}\|^2}\\
       \le{}& \O\sbr{\frac{d}{\lambda}\log \sbr{dV_T}},
    \end{align*}
where the first and third inequalities use $\ln(1+x)\le x$ for $x>0$, and we omit the $\O(\log \log T)$ term.
\end{proof}

\subsection{Omitted Details in \pref{sec:scvx-base}}
\label{app:cor-scvx}
To start, we first present \pref{alg:scvx-T} for strongly convex functions.
\begin{algorithm}[t]
    \caption{Expected Gradient Descent with Two Queries}
    \label{alg:scvx-T}
    \begin{algorithmic}[1]
    \Require Learning rates $\eta_t$, exploration parameter $\delta$ and shrinkage coefficient $\xi$
    \State $\w_1 \gets 0$
    \For{$t = 1, \ldots, T$}
        \State Pick a unit vector $\u_t$ uniformly from the unit sphere $\partial \mathbb{B}$
        \State Observe $f_t(\w_t + \delta \u_t)$ and $f_t(\w_t - \delta \u_t)$
        \State Set $\g_t \gets \frac{d}{2\delta}(f_t(\w_t + \delta \u_t ) - f_t(\w_t - \delta \u_t))\u_t$
        \State Update $\w_{t+1} \gets \Pi_{(1-\xi)\X}(\w_t - \eta_t \g_t)$
    \EndFor
    \end{algorithmic}
\end{algorithm}

\begin{Theorem}
    \label{thm:T-strcvx}
    Under Assumptions \ref{ass:boundedness}-\ref{ass:Lipschitzness}, by setting $\eta_t=\frac{1}{\lambda t}$, $\delta=\frac{1}{2d^2 TR}$,  and $\xi=\frac{\delta}{r}$, \pref{alg:scvx-T} satisfies:
    \begin{equation*}
        \E[\Reg_T]\le \O\sbr{\frac{d}{\lambda}\log  T}.
    \end{equation*}
\end{Theorem}
Before the proof, we first state the following technical concentration inequality.
\begin{Lemma}[Theorem 5.1.4 of \citet{vershynin2018high}]
    \label{lem:concentration_ball}
    Let $X$ be uniformly distributed on the unit sphere $\partial \mathbb{B}$, and let $g: \mathbb{R}^d \rightarrow \mathbb{R}$ be $G$-Lipschitz with respect to the Euclidean norm. Then for all $t \geq 0$,
    \begin{equation*}
        \mathbb{P}(|g(X)-\mathbb{E}[g(X)]|>t) \leq 2\cdot \exp(-c^{\prime} d t^2 / G^2).
    \end{equation*}
    \end{Lemma}
\begin{proof}[of \pref{thm:T-strcvx}]
For non-smooth functions, by simply choosing $\tgb_t = \mathbf{0}$ in \pref{eq:scvx-omd}, we have
\begin{equation*}
    \E[\Reg_T] \le 2 \E\mbr{\sum_{t=1}^T \eta_t \|\g_t\|^2} +\O(1).
\end{equation*}
In the following, we first upper bound $\E[\|\g_t\|^2]$. Following the analysis in \citet{shamir2017optimal}, we have
\begin{align*}
    \E\mbr{\|\g_t\|^2} = {} & \frac{d^2}{4\delta^2}\E\mbr{\|(f_t(\w_t+\delta\u_t)-f_t(\w_t-\delta\u_t))\u_t\|^2}
    \le \frac{d^2}{\delta^2}\E\mbr{\|(f_t(\w_t+\delta\u_t)-\E\mbr{f_t(\w_t+\delta\u_t)})\u_t\|^2}\\
    \le{}&\frac{d^2}{\delta^2}\underbrace{\sqrt{\E\mbr{(f_t(\w_t+\delta\u_t)-\E\mbr{f_t(\w_t+\delta\u_t)})^4}}}_{\term{a}}\underbrace{\sqrt{\E\mbr{\|\u_t\|^4}}}_{\term{b}}.
\end{align*}
    For \term{a}, by the Lipschitzness of $f_t$, setting $g_t(\u)=f_t(\w_t+\delta \u)$, by \pref{lem:concentration_ball}, we have
    \begin{align*}
    & \sqrt{\mathbb{E}\left[(g_t(\u_t)-\mathbb{E}[g_t(\u_t)])^4\right]}=\sqrt{\int_{t=0}^{\infty} \mathbb{P}\left((g_t(\u_t)-\mathbb{E}[g_t(\u_t)])^4>t\right) d t} \\
    & =\sqrt{\int_{t=0}^{\infty} \mathbb{P}(|g_t(\u_t)-\mathbb{E}[g_t(\u_t)]|>\sqrt[4]{t}) d t} \leq \sqrt{\int_{t=0}^{\infty} 2 \exp \left(-\frac{c^{\prime} d \sqrt{t}}{\delta^2 G^2}\right) d t}=\sqrt{2 \frac{\delta^4 G^4}{\left(c^{\prime} d\right)^2}}.
    \end{align*}
    For \term{b}, since $\|\u_t\|^4=1$, we have $\E\left[\|\g_t\|^2\right]\le CdG^2$, where $C$ is a constant independent of $d,T$.

    With this upper bound, using \pref{lem:sum_scvx} with $a_t=\E[\|\g_t\|^2]$, we have
\begin{equation*}
    \E[\Reg_T] \le \O\sbr{ \frac{d}{\lambda}\log\sbr{\E\mbr{\sum_{t=1}^T \|\g_t\|^2}}} \le \O\sbr{\frac{d}{\lambda}\log  T}.
\end{equation*}
\end{proof}

\subsection{Proof of \pref{lem:FT}}
\label{app:FT_lem}
\begin{proof}
To begin with, using property~\rom{8} in \pref{lem:alpha}, it holds that
\begin{align}
\E\mbr{\Vb_T}={}&\E\mbr{\sumT \|\g_t-\tgb_t\|^2}\le d^2\E\mbr{\sumT(\nabla_{i_t} f_{t}(\w_t)-\nabla_{i_t} f_{\alpha_t}(\w_{\alpha_t}))^2}+\O(1)\notag\\
    \le{}&2d^2\underbrace{\E\mbr{\sumT\inner{\nabla f_t(\w_t)}{\eb_{i_t}}^2}}_{\term{a}}+2d^2\underbrace{\E\mbr{\sumT\inner{\nabla f_{\alpha_t}(\w_{\alpha_t})}{\eb_{i_t}}^2}}_{\term{b}}+\O(1).\label{eq:F_T1}
\end{align}
For \term{a}, we have
\begin{align}
    \label{eq:thm4eq1}
    \term{a}=\frac{1}{d}\E\mbr{\sumT\|\nabla f_t(\w_t)\|^2}\le \frac{4L}{d} \E\mbr{F_T^\w},
\end{align}
where $F_T^\w\triangleq \sumT f_t(\w_t)-\sumT \min_{\x\in\X^+}f_t(\x)$ and the inequality is by the self-bounding property $\|\nabla f(\x)\|_2^2 \le 4 L (f(\x)  - \min_{\x \in \X^+} f(\x))$ for any $L$-smooth function $f:\X^{+} \rightarrow \R$ and any $\x \in \X$~\citep{yan2024simple}.

For \term{b}, choosing $\c=\mathbf{0}$ in \pref{lem:exp_alphat}, we have
\begin{align}
    \term{b}\le{}&\sumT \sum_{s=1}^{t-1} \frac{1}{d^2} \left( 1 - \frac{1}{d} \right)^{t-s-1}\E_{I_{s-1}}\mbr{ \|\nabla f_s(\w_s)\|^2}\notag\\
    ={}& \sum_{s=1}^{T-1} \E_{I_{s-1}}\mbr{ \|\nabla f_s(\w_s)\|^2 } \sum_{t=s+1}^{T} \frac{1}{d^2} \left( 1 - \frac{1}{d} \right)^{t-s-1}
    \le \frac{1}{d} \sum_{s=1}^{T-1} \E_{I_{s-1}}\mbr{ \|\nabla f_s(\w_s)\|^2}\le
    \frac{4L}{d}\E\mbr{F_T^\w},\label{eq:thm4eq2}
\end{align}
where the last step is also by the self-bounding property.

Plugging \pref{eq:thm4eq1} and \pref{eq:thm4eq2} back to \eqref{eq:F_T1}, we have
\begin{equation}
    \label{eq:F_T}
   \E\mbr{\Vb_T}\le 16dL \E\mbr{F_T^\w}+\O(1).
\end{equation}
\end{proof}

\subsection{Proof of \pref{thm:WT}}
\label{app:WT}
\begin{proof}
We start from the regret decomposition in \pref{eq:OMD_framework}.
For \term{a}, we have the following bound:
\begin{align*}
    \term{a}\le{}& R\E\mbr{\sumT \frac{\|\g_t-\tgb_t\|^2}{\sqrt{d^2+\bar{V}_{t-1}}}}
    \le 4R\E\mbr{\sqrt{d^2+\sumT \|\g_t-\tgb_t\|^2}+\max_{t\in[T]}\frac{\|\g_t-\tgb_t\|^2}{d}}\notag\\
    \le{}&4R\E\mbr{\sqrt{d^2+\Vb_T}}+\O\sbr{d},
\end{align*}
where the second inequality is by choosing $a_s=\frac{\|\g_s-\tgb_s\|^2}{d^2}$ in \pref{lem:sum_cvx}.

For \term{b}, we have
\begin{equation*}
    \term{b}\le \E\mbr{\frac{2R^2}{\eta_1}+\frac{2R^2}{\eta_T}}= 2R\E\mbr{\sqrt{d^2+\Vb_T}}+\O(d).
\end{equation*}
Thus, we have
\begin{equation}
    \label{eq:empirical_VT for WT}
    \E[\Reg_T]\le6R\sqrt{d^2+\E\mbr{\Vb_T}}+\O\sbr{d}.
\end{equation}
\paragraph{Linear Functions.}
In the linear setting, the gradient of $f_t(\cdot)$ is a constant vector and can be denoted by $\ellb_t$ for simplicity.
For the gradient-variance bound, we have the following decomposition:
\begin{align}
    \E\mbr{\Vb_T} = {}& \E\mbr{\sumT \|\g_t-\tgb_t\|^2}\le d^2\E\mbr{\sumT(\ell_{t,i_t}-\ell_{\alpha_t,i_t})^2}+\O(1)\notag\\
    \le{}&2d^2\underbrace{\E\mbr{\sumT(\ell_{t,i_t}-\mu_{T,i_t})^2}}_{\term{c}}+2d^2\underbrace{\E\mbr{\sumT(\ell_{\alpha_t,i_t}-\mu_{T,i_t})^2}}_{\term{d}}+\O(1).\notag
    \end{align}
For \term{c}, we have
\begin{align}
    \label{eq:thm3eq1}
    \sumT\E\mbr{(\ell_{t,i_t}-\mu_{T,i_t})^2}=\frac{1}{d}\sumT\|\ellb_t-\mub_T\|^2=\frac{W_T}{d}.
\end{align}
For \term{d}, since for linear functions, $\mub_T$ is a constant vector. Choosing $\c=\mub_T$ in \pref{lem:exp_alphat}, we have
\begin{align}
    &\E\mbr{\sumT(\ell_{\alpha_t,i_t}-\mu_{T,i_t})^2}
    \le \sumT \sum_{s=1}^{t-1} \frac{1}{d^2} \left( 1 - \frac{1}{d} \right)^{t-s-1}\E_{I_{s-1}}\mbr{ \|\ellb_s-\mub_T\|^2}+ \frac{1}{d}\sum_{t=1}^T\left(1-\frac{1}{d}\right)^{t-1} \|\mub_T\|^2\notag\\
    \le{}& \sum_{s=1}^{T-1} \E_{I_{s-1}}\mbr{ \|\ellb_s-\mub_T\|^2} \sum_{t=s+1}^{T} \frac{1}{d^2} \left( 1 - \frac{1}{d} \right)^{t-s-1}+\|\mub_T\|^2
    \le \frac{1}{d} \sbr{\sum_{s=1}^{T-1}  \|\ellb_s-\mub_T\|^2}+G^2\le\frac{W_T}{d}+G^2,\label{eq:thm3eq2}
\end{align}
where the second inequality is because for linear functions, $\ellb_t$ is deterministic for $t \in [T]$ and $\|\ellb_0 - \mub_T\|\le G$.

Plugging \pref{eq:thm3eq1} and \pref{eq:thm3eq2} back, we have
\begin{align}
     \E\mbr{\Vb_T}\le4d W_T+\O(d^2).\label{eq:W_T_lin}
\end{align}
Plugging \eqref{eq:W_T_lin} back to \eqref{eq:empirical_VT for WT}, we have
\begin{equation*}
    \E[\Reg_T]\le6R\sqrt{d^2+4dW_T+\O(d^2)}+\O\sbr{d}\le \O\sbr{\sqrt{dW_T}+d}.
\end{equation*}
\paragraph{Convex Functions.}
By property \rom{8} of \pref{lem:alpha}, we have the following decomposition:
\begin{align}
     \Vb_T={}&\sumT \|\g_t-\tgb_t\|^2\le d^2\sumT(\nabla f_{t,i_t}(\w_t)-\nabla f_{\alpha_t,i_t}(\w_{\alpha_t}))^2+\O(1)\notag\\
     \le{}&d^2\sumT\|\nabla f_{t}(\w_t)-\nabla f_{\alpha_t}(\w_{\alpha_t})\|^2+\O(1)\notag\\
     \le{}&2d^2\sumT\sbr{\|\nabla f_{t}(\w_t)-\mub_T\|^2+\|\nabla f_{\alpha_t}(\w_{\alpha_t})-\mub_T\|^2}+\O(1)\notag\\
     \le{}& 4d^2\sumT \|\nabla f_t(\w_t)-\mub_T\|^2+ d^3\|\mub_T\|^2+\O(1)\notag\\
     \le{}&4d^2\sup_{\{\w_1,\ldots,\w_T\} \in (1-\xi)\X} \bbr{\sumT \|\nabla f_t(\w_t) - \mub_T\|^2} +\O(d^3) \le4d^2W_T+\O(d^3),\label{eq:W_T}
\end{align}
where $\mub_T\define \frac{1}{T} \sumT \nabla f_t(\w_t)$. The third step is by $\inner{\x}{\eb_i}^2\le\|\x\|^2$ for any $\x$. The fifth step follows by observing that for each $s \ge 1$, there exists at most a single index $t$ such that $\alpha_t = s$, while $\alpha_t = 0$ occurs no more than $d$ times.

Plugging \eqref{eq:W_T} back to \eqref{eq:empirical_VT for WT}, we have
\begin{equation*}
    \E[\Reg_T]\le6R\sqrt{d^2+4d^2W_T+\O(d^3)}+\O\sbr{d}\le \O\sbr{d\sqrt{W_T+d}}.
\end{equation*}
\paragraph{Strongly Convex Functions.}
By plugging \pref{eq:W_T} into the bound in \pref{eq:reg_bound_mid_scvx}, we directly have
\begin{equation*}
    \E[\Reg_T]\le\O\sbr{\frac{d}{\lambda}\log \sbr{dW_T}},
\end{equation*}
which completes the proof.
\end{proof}

\subsection{Proof of \pref{thm:FT}}
\label{app:FT}
\begin{proof}
\paragraph{Convex Functions.}
We start from the regret decomposition in \pref{eq:empirical_VT for WT}.
Plugging \eqref{eq:F_T} back to \eqref{eq:empirical_VT for WT}, we have
\begin{align*}
    \E[\Reg_T]\le{}&6R\sqrt{d^2+16dL \E\mbr{F_T^\w}+\O(1)}+\O\sbr{d}\\
    \le{}&\O\sbr{\sqrt{d\E\mbr{F_T^\w}}+d}\le\O\sbr{\sqrt{dF_T}+d},
\end{align*}
where the last step uses \pref{lem:small-loss-sqrt} by choosing $a=\O(d)$ and $b$ as a constant independent of $d,T$ and setting
\begin{equation*}
    x=\sumT f_t(\w_t)-\sumT\min_{\x\in\X^+}f_t(\x) \text{ and } y=\min_{\x\in\X}\sumT f_t(\x)-\sumT\min_{\x\in\X^+}f_t(\x).
\end{equation*}
\paragraph{Strongly Convex Functions.}
Plugging \eqref{eq:F_T} into the bound in \pref{eq:reg_bound_mid_scvx}, we directly have
\begin{equation*}
    \E[\Reg_T]\le{} \O\sbr{\frac{d}{\lambda}\log \sbr{d\E\mbr{F_T^\w}}}\le\O\sbr{\frac{d}{\lambda}\log \sbr{dF_T}},
\end{equation*}
where the last step uses \pref{lem:small-loss-log} by choosing $a=\O(d)$ and $b,c$ as constants independent of $d,T$ and setting
\begin{equation*}
    x=\sumT f_t(\w_t)-\sumT\min_{\x\in\X^+}f_t(\x) \text{ and } y=\min_{\x\in\X}\sumT f_t(\x)-\sumT\min_{\x\in\X^+}f_t(\x).
\end{equation*}
\end{proof}


\section{Omitted Details in \pref{sec:1point}}
\label{app:1point}
In this section, we provide the omitted details in \pref{sec:1point}.
Specifically, we first present a practical version of \pref{alg:1p-BLO} in \pref{app:practical-1p-BLO}, and then prove its gradient-variation regret in \pref{app:1p-BLO}.
Without loss of generality, we assume $a_i\le0< b_i$ for all $i \in [d]$.

\subsection{A Practical Version of Algorithm~\ref{alg:1p-BLO}}
\label{app:practical-1p-BLO}
To start, we show the relationship between $w_{t+1,i}$ and $\gt_{t+1,i}$.
\begin{gather}
    \eta \sum_{s=1}^{t} g_{s,i}+\eta\gt_{t+1,i}+\frac{1}{b_i-w_{t+1,i}}-\frac{1}{w_{t+1,i}-a_i}=0,\label{eq:E1}\\
    \gt_{t+1,i}=\frac{1}{2}\lambda_{t+1,i}^{\frac{1}{2}}(r^{(+1)}_i-r^{(-1)}_i),\label{eq:E2}\\
    \lambda_{t+1,i}=\frac{1}{\sbr{b_i-w_{t+1,i}}^2}+\frac{1}{\sbr{w_{t+1,i}-a_i}^2}.\label{eq:E3}
\end{gather}
Let $f_i(x) = -\log(b_i - x) - \log(x - a_i)$ be the univariate log-barrier.
By combining \pref{eq:E1}-\pref{eq:E3}, we can show that for each $i \in [d]$, $w_{t+1,i}$ satisfies the optimality condition $F_{t,i}(x) = 0$, where $F_{t,i}$ is defined as:
\begin{equation}
    \label{eq:F_i}
    F_{t,i}(x)\triangleq\eta \left(\sum_{s=1}^{t}g_{s,i}+\frac{1}{2}\sqrt{f_i^{\prime\prime}(x)}(r^{(+1)}_i-r^{(-1)}_i)\right)+f_i^\prime(x).
\end{equation}
The existence and uniqueness of the solution $w_{t+1,i} \in (a_i, b_i)$ are established in \pref{lem:solve_x}.
\begin{Lemma}
\label{lem:solve_x}
Assume that for each dimension $i \in [d]$, the step size $\eta$ and the buffer constants $c_i = \frac{1}{2}(r^{(+1)}_i - r^{(-1)}_i)$ satisfy the condition $|\eta c_i| < 1$.
Then $F_{t,i}(x)$ is strictly monotone and the equation $F_{t,i}(x) = 0$ has a unique solution $w_{t+1,i}$ in the open interval $(a_i, b_i)$.
\end{Lemma}
\begin{proof}
For simplicity, let $L(x) = x - a_i$ and $R(x) = b_i - x$.
We can express the derivatives of $f_i(x)$ as:
\begin{equation*}
    f_i^\prime(x) = \frac{1}{R(x)} - \frac{1}{L(x)}, \qquad \text{and} \qquad
    f_i^{\prime\prime}(x) = \frac{1}{L(x)^2} + \frac{1}{R(x)^2}.
\end{equation*}
Let $S_{t,i} = \sum_{s=1}^{t} g_{s,i}$ and $c_i = \frac{1}{2}(r^{(+1)}_i - r^{(-1)}_i)$.
We rewrite \eqref{eq:F_i} as:
\begin{equation*}
    F_{t,i}(x) = \eta S_{t,i} + \eta c_i \sqrt{f_i^{\prime\prime}(x)} + f_i^\prime(x)= \eta S_{t,i} + \frac{\eta c_i \sqrt{L(x)^2 + R(x)^2} + L(x) - R(x)}{L(x)R(x)}.
\end{equation*}

\textbf{Existence:}
We evaluate the limits of $F_{t,i}(x)$ as $x$ approaches the boundaries.
For the limit $x \to a_i^+$, we have $L(x) \to 0^+$ and $R(x) \to b_i - a_i$.
The numerator of the fraction satisfies:
\begin{equation*}
    \lim_{x \to a_i^+} \left( \eta c_i \sqrt{L(x)^2 + R(x)^2} + L(x) - R(x) \right) = \eta c_i (b_i - a_i) + 0 - (b_i - a_i) = (\eta c_i - 1)(b_i - a_i).
\end{equation*}
Since $|\eta c_i| < 1$, it follows that $(\eta c_i - 1) < 0$.
As the denominator $L(x)R(x) \to 0^+$, we have:
\begin{align*}
    \lim_{x \to a_i^+} F_{t,i}(x) = \eta S_{t,i} + \frac{(\eta c_i - 1)(b_i - a_i)}{0^+} = -\infty.
\end{align*}
Similarly, for the limit $x \to b_i^-$, we have $R(x) \to 0^+$ and $L(x) \to b_i - a_i$.
The numerator satisfies:
\begin{equation*}
    \lim_{x \to b_i^-} \left( \eta c_i \sqrt{L(x)^2 + R(x)^2} + L(x) - R(x) \right) = \eta c_i (b_i - a_i) + (b_i - a_i) - 0 = (\eta c_i + 1)(b_i - a_i).
\end{equation*}
Since $\eta c_i + 1 > 0$, and the denominator $L(x)R(x) \to 0^+$, we have:
\begin{align*}
    \lim_{x \to b_i^-} F_{t,i}(x) = \eta S_{t,i} + \frac{(\eta c_i + 1)(b_i - a_i)}{0^+} = +\infty.
\end{align*}
By the Intermediate Value Theorem, there exists at least one $x^* \in (a_i, b_i)$ such that $F_{t,i}(x^*) = 0$.

\textbf{Uniqueness:}
We show that $F_{t,i}(x)$ is strictly monotonically increasing by examining its derivative:
\begin{equation*}
    F_{t,i}^\prime(x) = f_i^{\prime\prime}(x) + \eta c_i \frac{f_i^{\prime\prime\prime}(x)}{2\sqrt{f_i^{\prime\prime}(x)}} = f_i^{\prime\prime}(x) \left( 1 + \eta c_i \frac{f_i^{\prime\prime\prime}(x)}{2(f_i^{\prime\prime}(x))^{3/2}} \right).
\end{equation*}
Substituting $f_i^{\prime\prime\prime}(x) = 2(R(x)^{-3} - L(x)^{-3})$, and letting $\rho = R(x)/L(x) \in (0, \infty)$, the ratio term becomes:
\begin{equation*}
    \left| \frac{f_i^{\prime\prime\prime}(x)}{2(f_i^{\prime\prime}(x))^{3/2}} \right| = \left| \frac{R^{-3} - L^{-3}}{(L^{-2} + R^{-2})^{3/2}} \right| = \left| \frac{1 - \rho^3}{(1 + \rho^2)^{3/2}} \right|.
\end{equation*}
Let $h(\rho) = \frac{1 - \rho^3}{(1 + \rho^2)^{3/2}}$.
It can be verified that $|h(\rho)| < 1$ for all $\rho > 0$.
Therefore,
\begin{equation*}
    F_{t,i}^\prime(x) > f_i^{\prime\prime}(x) (1 - |\eta c_i|).
\end{equation*}
Since $|\eta c_i| < 1$ and $f_i^{\prime\prime}(x) > 0$, we have $F_{t,i}^\prime(x) > 0$ for $x \in (a_i, b_i)$.
The strict monotonicity implies the unique solution.
\end{proof}

With the theoretical guarantee, we present a practical version for \pref{alg:1p-BLO} in \pref{alg:practical}.
\begin{algorithm}[!t]
\caption{Practical Version of \pref{alg:1p-BLO}}
\label{alg:practical}
\begin{algorithmic}
    \Require Step size $\eta>0$
    \State Let $\w_1=\mathbf{0}$, $\tgb_1 = \mathbf{0}$.
    Let buffer vectors $\r^{(+1)}=\r^{(-1)}=\mathbf{0}$.
\For{round $t \in [T]$}
    \State Choose $i_t$ uniformly from $[d]$ and $\epsilon_t$ uniformly from $\{-1,+1\}$ and fetch $z_t=r^{(\epsilon_t)}_{i_t}$
    \State Play action $\x_t = \w_t + \epsilon_t \lambda_{t,i_t}^{-1/2}\eb_{i_t}$
    \State Observe partial information $v_t=\inner{\ellb_t}{\x_t}$
    \State Compute $\g_t$ as in \pref{eq:chiangestimator1p}
    \State Update buffer vector $r^{(\epsilon_t)}_{i_t}=v_t$
    \State Construct $F_{t,i}$ as in \pref{eq:F_i} and solve the following system of equations for $\w_{t+1}$
    \begin{equation*}
    F_{t,i}(w_{t+1,i})=0,\  \forall i\in [d]
    \end{equation*}
    \State Compute $\tgb_{t+1}$ as in \pref{eq:chiangestimator1p}
\EndFor
\end{algorithmic}
\end{algorithm}

\subsection{Proof of \pref{thm:1p-BLO}}
\label{app:1p-BLO}
Before the proof, we introduce \pref{lem:Lipschitz} that will be used in the proof.
\begin{Lemma}
    \label{lem:Lipschitz}
    Let $\lambda_{i}(x)=\frac{1}{(b_i-x)^2}+\frac{1}{(x-a_i)^2}$.
    Then, for any $x,y\in (a_i,b_i)$, we have $\left|\lambda_{i}^{-\frac{1}{2}}(x)-\lambda_{i}^{-\frac{1}{2}}(y)\right|\le |x-y|$.
\end{Lemma}

\begin{proof}[of \pref{thm:1p-BLO}]
First, we denote $\|\a\|_{\w_t}=\sqrt{\a^\top\nabla^2 \Rc(\w_t)\a}$ and $\|\a\|_{{\w_t}^{-1}}=\sqrt{\a^\top\nabla^{-2} \Rc (\w_t)\a}$ for any $\a\in \R^d$.

To start, we observe that $\E_{i_t,\epsilon_t}[\x_t]=\w_t$ and
\begin{align*}
    \E_{i_t,\epsilon_t}\mbr{\g_t} = {} & \E_{i_t,\epsilon_t}\mbr{d\sbr{\inner{\ellb_t}{\x_t}-\inner{\ellb_{\alpha_t}}{\x_{\alpha_t}}}\epsilon_t\lambda_{t,i_t}^{\frac{1}{2}}\eb_{i_t}}+\tgb_t \tag{by definition of $\g_t$}\\
    ={}&d\E_{i_t,\epsilon_t}\mbr{\inner{\ellb_t}{\w_t+\epsilon_t\lambda_{t,i_t}^{-\frac{1}{2}}\eb_{i_t}}\epsilon_t\lambda_{t,i_t}^{\frac{1}{2}}\eb_{i_t}}-d\E_{i_t,\epsilon_t}\mbr{\epsilon_t\lambda_{t,i_t}^{\frac{1}{2}}\inner{\ellb_{\alpha_t}}{\x_{\alpha_t}}{\eb_{i_t}}}+\tgb_t\\
    =&d\E_{i_t,\epsilon_t}\mbr{\inner{\ellb_t}{\w_t+\epsilon_t\lambda_{t,i_t}^{-\frac{1}{2}}\eb_{i_t}}\epsilon_t\lambda_{t,i_t}^{\frac{1}{2}}\eb_{i_t}}-\frac{d}{2}\E_{i_t,\epsilon_t}\mbr{\epsilon_t \lambda_{t,i_t}^{\frac{1}{2}} \sbr{\inner{\ellb_{\alpha_t}}{\x_{\alpha_t}}-\inner{\ellb_{\beta_t}}{\x_{\beta_t}}}{\eb_{i_t}}}+\tgb_t\\
    =&d\E_{i_t,\epsilon_t}\mbr{\inner{\ellb_t}{\w_t+\epsilon_t\lambda_{t,i_t}^{-\frac{1}{2}}\eb_{i_t}}\epsilon_t\lambda_{t,i_t}^{\frac{1}{2}}\eb_{i_t}}-d\E_{i_t}\mbr{\gt_{t,i_t}\eb_{i_t}}+\tgb_t \tag{by definition of $\tgb_t$}\\
    =&d\E_{i_t,\epsilon_t}\mbr{\inner{\ellb_t}{\w_t+\epsilon_t\lambda_{t,i_t}^{-\frac{1}{2}}\eb_{i_t}}\epsilon_t\lambda_{t,i_t}^{\frac{1}{2}}\eb_{i_t}}
    = d\E_{i_t}\mbr{\ell_{t,i_t}\eb_{i_t}}=\ellb_t,
\end{align*}
where we define $\beta_t$ as the largest integer such that $0 \leq \beta_t < t$, $i_{\beta_t} = i_t$, and $\epsilon_{\beta_t} =-\epsilon_t$.  The third and sixth steps are by the symmetry of $\epsilon_t$, and both the fifth and the last step follow from $d\E[v_{i}\eb_i]=\sum_{i=1}^{d}v_i\eb_i=\v$, for any $\v \in \R^d$.
Then, we have
\begin{align*}
    \E\mbr{\inner{\ellb_t}{\x_t}}=\E\mbr{\inner{\ellb_t}{\E_{i_t,\epsilon_t}[\x_t]}}=\E\mbr{\inner{\E_{i_t,\epsilon_t}[\g_t]}{\w_t}}=\E\mbr{\inner{\g_t}{\w_t}}.
\end{align*}
Thus, by the FTRL analysis framework, e.g., Theorem 7.47 of \citet{orabona2019modern}, we have
\begin{align}
    \E[\Reg_T] = {} & \E\mbr{\sumT \inner{\ellb_t}{\x_t-\v}}+\E\mbr{\sumT \inner{\ellb_t}{\v-\xs}}
    \le \E\mbr{\sumT  \inner{\g_t}{\w_t-\v}}+\O(1)\notag \\
    \le{}&\underbrace{\E\mbr{\eta\sumT  \left\|\g_t-\tgb_t\right\|^2_{\w_t^{-1}}}}_{\term{a}}
    +\underbrace{\E\mbr{\frac{\mathcal{R}(\v)}{\eta}}}_{\term{b}} - \underbrace{\E\mbr{\frac{1}{\eta}\sumT  \|\w_t-\w_{t+1}\|_{\w_t}^2}}_{\term{c}}+\O(1),\label{eq:OMD_framework1q}
\end{align}
where $\x^\star\in\argmin_{\x\in\X} \sumT \inner{\ellb_t}{\x}$ and  $\v \define\frac{T-1}{T}\xs+\frac{1}{T}\x^\dagger$ where $\x^\dagger\in \X$ satisfies $\nabla \mathcal{R}(\x^\dagger)= \mathbf{0}$.

For \term{a}, we have the following bound:
\begin{equation}
    \label{eq:s_mid}
    \text{\term{a}}= \eta d^2\E\mbr{\sumT (v_t-z_t)^2\lambda_{t,i_t}\eb_{i_t}^\top\nabla^{-2} \mathcal{R}(\w_t)\eb_{i_t}} \le \eta d^2\E\mbr{\sumT (v_t-z_t)^2}
\end{equation}
For $(v_t-z_t)^2$, we have the following decomposition:
\begin{align*}
    (v_t-z_t)^2= {} & \sbr{\inner{\ellb_t}{\x_t}-\inner{\ellb_{\alpha_t}}{\x_{\alpha_t}}}^2
    \le 2\inner{\ellb_t}{\x_t-\x_{\alpha_t}}^2+2\inner{\ellb_t-\ellb_{\alpha_t}}{\x_{\alpha_t}}^2\notag\\
    \le{}& 2G^2\norm{\w_t+\epsilon_t\lambda_{t,i_t}^{-\frac{1}{2}}\eb_{i_t}-\w_{\alpha_t}-\epsilon_t\lambda_{\alpha_t,i_t}^{-\frac{1}{2}}\eb_{i_t}}^2+2R^2\norm{\ellb_t-\ellb_{\alpha_t}}^2 \tag*{(by $\|\ellb_t\| \le G$)}\\
    \le{}& 4G^2\norm{\w_t-\w_{\alpha_t}}^2+4G^2\norm{\lambda_{t,i_t}^{-\frac{1}{2}}-\lambda_{\alpha_t,i_t}^{-\frac{1}{2}}}^2+2R^2\norm{\ellb_t-\ellb_{\alpha_t}}^2\notag\\
    \le{}&8G^2\norm{\w_t-\w_{\alpha_t}}^2+2R^2\norm{\ellb_t-\ellb_{\alpha_t}}^2 \\
    \le{}& 2(t-\alpha_t)\sbr{4G^2\sum_{s=\alpha_t+1}^{t}\|\w_s-\w_{s-1}\|^2+R^2\sum_{s=\alpha_t+1}^{t}\norm{\ellb_s-\ellb_{s-1}}^2}\notag,
\end{align*}
where the third inequality is by \pref{lem:Lipschitz} and the last step is by the triangle inequality and Cauchy-Schwarz inequality.

Plugging the above inequality into \pref{eq:s_mid}, we have
\begin{align}
    \term{a} \le {} & \eta d^2\E\mbr{\sumT\sum_{\epsilon\in\{-1,+1\}}\sum_{i=1}^{d}\rho_{t,i}^{(\epsilon)}\sbr{8G^2\|\w_t-\w_{t-1}\|^2+2R^2\norm{\ellb_t-\ellb_{t-1}}^2}}\notag\\
    \le{}& 16\eta d^4 R^2\sumT \|\ellb_t-\ellb_{t-1}\|^2+128\eta d^4G^2\ln(2dT)\E\mbr{\sum_{t=2}^{T}\|\w_t-\w_{t-1}\|^2}+\O(1),\label{eq:termS}
\end{align}
where in this case, we define the non-consecutive sampling gap $\rho_{t,i}^\epsilon$ as
\begin{equation*}
    \rho_{t,i}^{(\epsilon)} \define \tau_2 - \tau_1\text{, where } \tau_1 \define \max\{ 0\le\tau < t \given i_\tau = i,\epsilon_\tau = \epsilon\}, \text{ and } \tau_2\define \min \{ t\le\tau\le T+1 \given i_\tau = i,\epsilon_\tau = \epsilon\}.
\end{equation*}
The last inequality follows from setting $n=2d$ in \pref{lem:Geo} and noting that $\rho_{t,i}^{(\epsilon)} \leq 8d \ln(2dT)$ simultaneously for all $(i,\epsilon)$ and $t$ with probability at least $1 - \O(1/(d^3T^3))$. The proof for this high-probability bound is analogous to the one provided in \pref{app:cvx-base} and we omit it here.

For \term{b}, we have
\begin{align}
    \label{eq:termb}
    \term{b} = \frac{\Rc(\v)}{\eta} \le \frac{2d\log T+\O(d)}{\eta}.
\end{align}
For \term{c}, we have
\begin{align}
    \term{c} = {} & \frac{1}{\eta}\E\mbr{\sum_{t=2}^T\|\w_t-\w_{t-1}\|_{\w_{t-1}}^2} = \frac{1}{\eta}\E\mbr{\sum_{t=2}^T\sum_{i=1}^d \lambda_{t-1,i}(\w_{t,i}-\w_{t-1,i})^2} \notag\\
    \ge {} & \frac{4}{\max_i (b_i-a_i)^2\eta}\E\mbr{\sum_{t=2}^T\|\w_t-\w_{t-1}\|^2}, \label{eq:termc}
\end{align}
where the last step is by $\lambda_{t-1,i}=\frac{1}{(w_{t-1,i}-a_i)^2}+\frac{1}{(b_i-w_{t-1,i})^2}\geq \frac{4}{(b_i-a_i)^2}$.

Plugging \eqref{eq:termS}, \eqref{eq:termb}, and \eqref{eq:termc} into \eqref{eq:OMD_framework1q}, we have
\begin{align*}
    \E[\Reg_T]\le {} & 16\eta d^4 R^2\sumT \|\ellb_t-\ellb_{t-1}\|^2+\frac{2d\log T+\O(d)}{\eta}+\O(1)\notag\\
    &+\E\mbr{\sbr{128\eta d^4G^2\ln(2dT)- \frac{4}{\eta\max_i (b_i-a_i)^2}}\sum_{t=2}^T \|\w_t-\w_{t-1}\|^2}\le\O\sbr{d^{\frac{7}{2}}\sqrt{V_T\log^2 T\log(dT)}},
\end{align*}
where $V_T=\sumTT \|\ellb_t-\ellb_{t-1}\|^2$ and the last step is by choosing $\eta=\frac{1}{16RGd^2\sqrt{V_T \log(2dT)}}$, using $R=\O(\sqrt{d})$.
\end{proof}
Finally, for completeness, we prove that $\x_t\in \X$ by \pref{lem:xt_in_X}.
\begin{Lemma}
\label{lem:xt_in_X}
Under Assumptions~\ref{ass:boundedness}-\ref{ass:Smoothness}, assume $\X=\prod_{i=1}^d [a_i, b_i]$, where $a_i \le 0 \le b_i$ for all $i \in [d]$. Choosing $\eta=\frac{1}{16RGd^2\sqrt{V_T\ln(2dT)}}$, for any $t$, $\x_t$ played by \pref{alg:practical} is in the interior of $\mathcal{X}$.
\end{Lemma}
\begin{proof}[of \pref{lem:xt_in_X}]
    Recall that for any $t$, $\x_t = \w_t + \epsilon_t \lambda_{t,i_t}^{-1/2}\mathbf{e}_{i_t}$.

    Since the update is coordinate-wise, we only need to verify $a_{i_t} < x_{t, i_t} < b_{i_t}$.
    The $i_t$-th eigenvalue of the log-barrier Hessian at $\w_t$ is $\lambda_{t, i_t} = (b_{i_t} - w_{t, i_t})^{-2} + (w_{t, i_t} - a_{i_t})^{-2}$.
    This definition implies:
    \begin{equation*}
        \lambda_{t, i_t}^{-1/2} < \min(b_{i_t} - w_{t, i_t}, w_{t, i_t} - a_{i_t}).
    \end{equation*}
    Consequently, for either choice of $\epsilon_t \in \{\pm 1\}$, if $\epsilon_t = 1$, $x_{t, i_t} = w_{t, i_t} + \lambda_{t, i_t}^{-1/2} < w_{t, i_t} + (b_{i_t} - w_{t, i_t}) = b_{i_t}$.
    If $\epsilon_t = -1$, \smash{$x_{t, i_t} = w_{t, i_t} - \lambda_{t, i_t}^{-1/2} > w_{t, i_t} - (w_{t, i_t} - a_{i_t}) = a_{i_t}$}.
    Since $x_{t, i_t}$ is strictly within the boundaries, $\x_t$ remains in the interior of $\X$, which finishes the proof.
\end{proof}
Finally, for the sake of completeness, we provide the proof of \pref{lem:Lipschitz}.
\begin{proof}[of \pref{lem:Lipschitz}]
    Let $g(x) = \frac{1}{\sqrt{\lambda_i(x)}} = \left( \frac{1}{(x -a_i)^2} + \frac{1}{(b_i - x)^2} \right)^{-1/2}$.
To prove that $g(x)$ is 1-Lipschitz, it suffices to show that $|g'(x)| \le 1$ for all $x \in (a_i, b_i)$.

Let $u = x-a_i$ and $v = b_i-x$.
Note that since $x \in (a_i, b_i)$, both $u > 0$ and $v > 0$.
The function can be rewritten as:
\begin{equation*}
    g(x) = (u^{-2} + v^{-2})^{-1/2}
\end{equation*}
Taking the derivative with respect to $x$ (noting that $\frac{du}{dx} = 1$ and $\frac{dv}{dx} = -1$):
\begin{equation*}
    |g'(x)| = \left|\frac{u^{-3} - v^{-3}}{(u^{-2} + v^{-2})^{3/2}}\right| \le 1,
\end{equation*}
where the last step follows from the fact that $u, v > 0$.
\end{proof}


\section{Omitted Proofs in \pref{sec:extension}}
\label{app:extension-proof}
In this section, we provide the details for the results presented in \pref{sec:extension}, including dynamic regret and universal regret. Specifically, for both dynamic regret and universal regret, we leverage an online ensemble framework with a two-layer meta-base architecture to mitigate environmental uncertainty. In this setup, a diverse set of base learners is deployed to explore the environment, while a meta-algorithm adaptively tracks the best-performing individual.

\subsection{Omitted Details in \pref{subsec:dynamic}}
\label{app:dynamic-cvx}
In this part, we provide the omitted details for our applications in dynamic regret.
Building upon the online ensemble framework of \citet{zhao2024adaptivity}, we introduce new surrogate loss functions to tackle the information limitations in the bandit setting.
Specifically, we design a two-layer online ensemble structure, with the meta-algorithm employing \ohedge~\citep{hazan2016introduction} to aggregate $N$ base learners.
Each base learner is instantiated as \pref{alg:chiang} with a distinct step size from a predefined pool and the $i$-th base learner outputs $\w_{t,i}$ at round $t$.
In the following, we elaborate on the online scheduling strategy, base learners, and architecture of the meta-algorithm.

\paragraph{Online Scheduling Strategy and Base Learners.}
To ensure thorough exploration, we construct a candidate step-size pool $\mathcal{H}$, where each base learner is instantiated with a unique step size from this pool. Specifically, the set of candidate step sizes $\mathcal{H}$ is defined as
\begin{equation}
    \label{eq:dynamic pool}
    \mathcal{H}=\left\{\eta_i=\min\bbr{\frac{1}{20L\sqrt{d^3\log(dT)}}, \sqrt{\frac{R^2}{d^3 T\log d}} \cdot 2^{i-1}} \givenn i \in[N]\right\},
\end{equation}
where $N = \left\lceil \log_2\left(1+\frac{\sqrt{T\log d}}{16L\sqrt{\log(dT)}}\right)\right\rceil + 1=\O (\log T)$ is the number of base learners. For each $i \in [N]$, we instantiate the base learner $\mathcal{B}_i$ via \pref{alg:chiang} with a fixed step size $\eta_i \in \mathcal{H}$.
Specifically, \pref{alg:chiang} can be interpreted as an instance of Optimistic Online Gradient Descent (\OOGD) \citep{rakhlin2013optimization}, where the gradient $\nabla f_t(\mathbf{x}_t)$ and the optimistic hint $M_t$ are substituted with the estimators $\mathbf{g}_t$ and $\tilde{\mathbf{g}}_t$, respectively.
Recall that \OOGD updates the decision as follows:
\begin{equation}
    \label{eq:OOGD}
    \x_t = \Pi_{\X} \mbr{\xh_{t}-\eta_{t} M_{t}},\ \xh_{t+1} = \Pi_{\X} \mbr{\xh_t-\eta_t \nabla f_{t}(\x_{t})}
\end{equation}
where $\eta_t > 0$ is a time-varying step size, $\xh_t$ and $\xh_{t+1}$ are internal decisions, and $\Pi_{\X}[\x] \define \argmin_{\y \in \X} \|\x - \y\|_2$ is the Euclidean projection onto the feasible domain $\X$.
We then characterize the dynamic regret guarantee of \OOGD below.
\begin{Lemma}[\citet{zhao2024adaptivity}]
    \label{lem:OOGD-dynamic}
    Under Assumptions \ref{ass:boundedness}-\ref{ass:Smoothness},
    if the loss functions are convex, \OOGD with $\eta\le \frac{1}{4L}$, enjoys the following bound: for any $\{\v_t\}_{t=1}^T$
    \begin{equation*}
       \sumT f_t(\w_t)-\sumT f_t(\v_t) \le \eta \sumT\left\|\nabla f_t(\w_t)-M_t\right\|_2^2+\frac{4}{\eta}\left(R^2+R P_T\right)-\frac{1}{4 \eta} \sum_{t=2}^T\left\|\w_t-\w_{t-1}\right\|_2^2,
    \end{equation*}
where $P_T\define\sumTT\norm{\v_t-\v_{t-1}}$.
\end{Lemma}

\paragraph{Meta-algorithm: \ohedge.}
Formally, \ohedge performs the following update: at round $t+1$, given the expert losses $\ellb_t \in \mathbb{R}^N$ and an optimistic vector $\m_{t+1} \in \mathbb{R}^N$, the learner computes the weight vector $\p_{t+1} \in \Delta_N$:
\begin{equation}
\label{eq:ohedge}
p_{t+1, i} \propto \exp \left(-\varepsilon_t\left(\sum_{s=1}^t \ell_{s, i}+m_{t+1, i}\right)\right), \quad \forall i \in[N] .
\end{equation}
Here, $\Delta_N$ represents an $N$-dimensional simplex, and $\varepsilon_t>0$ is the learning rate of the meta-algorithm.
The optimism $\m_{t+1} \in \mathbb{R}^N$ can be interpreted as an optimistic guess of the loss of round $t+1$, and we thus incorporate it into the cumulative loss for update.
\ohedge enjoys the following regret guarantee:
\begin{Lemma}
\label{lem:ohedge}
The regret of \ohedge~\eqref{eq:ohedge} with learning rate
\begin{equation*}
    \varepsilon_t=\sqrt{\frac{\ln N}{C_0^2+\sum_{s=1}^{t}\left\|\ellb_s-\m_s\right\|_{\infty}^2}},
\end{equation*}
 to any expert $i \in[N]$ is at most
\begin{equation*}
    \sumT \inner{\ellb_t}{\p_t - \eb_i} \le 5\sqrt{\ln N\sbr{C_0^2+\sumT \left\|\ellb_t-\m_t\right\|_{\infty}^2}} -\sumT \frac{C_0}{4 \sqrt{\ln N}}\left\|\p_t-\p_{t+1}\right\|_1^2+\frac{\sqrt{\ln N}}{C_0}\max_{t\in[T]}\|\ellb_t-\m_t\|^2_\infty,
\end{equation*}
where $C_0$ is a positive constant.
\end{Lemma}
To tackle the information limitations, we set the aforementioned loss $\ellb_t$ and the optimism $\m_t$ as follows: for $i\in[N]$
\begin{equation}
    \label{eq:dynamic-loss-optimism}
    \ell_{t,i} \define \inner{\g_t}{\w_{t,i}}+\gamma\|\w_{t,i}-\w_{t-1,i}\|^2,\quad
    m_{t,i} \define \inner{\tgb_t}{\w_{t,i}}+\gamma\|\w_{t,i}-\w_{t-1,i}\|^2,
\end{equation}
where $\gamma>0$ is a correction parameter, $\g_t$ and $\tgb_t$ are defined in \pref{eq:chiangestimator} and $\w_{t,i}$ denotes the decision of the $i$-th base learner at round $t$.

\paragraph{Overall Algorithm and Dynamic Regret Analysis.}
Integrating the aforementioned meta-base structure, we propose \pref{alg:dynamic_regret} for gradient-variation dynamic regret minimization in two-point BCO, which attains the performance guarantee established in \pref{thm:dynamic-cvx}.
\begin{algorithm}[!t]
    \caption{Dynamic Gradient-Variation Regret Minimization with Two-Point Feedback}
    \label{alg:dynamic_regret}
    \begin{algorithmic}[1]
        \State \textbf{Input:} Base learner number $N$
        \State \textbf{Initialize:} $\mathcal{A}$ --- meta learner running \ohedge~\eqref{eq:ohedge};
        \Statex \hskip\algorithmicindent $\{\Bc_i\}_{i \in [N]}$ --- the $i$-th base learner runs \pref{alg:chiang} with step size in \pref{eq:dynamic pool}
        \For{$t = 1$ \textbf{to} $T$}
         \State Choose $i_t$ uniformly from $[d]$.
       \State Submit $\x_t = \w_t + \delta \eb_{i_t}$, and $\x_t^\prime = \w_t - \delta \eb_{i_t}$, where $\w_t = \sumN p_{t,i} \w_{t,i}$
       \State Observe partial information $f_t(\x_t)$ and $f_t(\x_t^\prime)$
       \State Construct the gradient estimator $\g_t$ and the optimism $\tgb_t$ as in \pref{eq:chiangestimator}
            \State $\mathcal{A}$ updates with $\ell_{t,i}$ and $m_{t,i}$ as constructed in~\pref{eq:dynamic-loss-optimism} to get $\p_{t+1}$
            \State $\Bc_i$ updates to $\w_{t+1,i}$ with gradient estimator $\g_t$ and optimism $\tgb_t$
        \EndFor
    \end{algorithmic}
\end{algorithm}

\begin{Theorem}
    \label{thm:dynamic-cvx}
    Under  Assumptions \ref{ass:boundedness}-\ref{ass:Smoothness}, with
    \begin{equation*}
        \varepsilon_t=\sqrt{\frac{\ln N}{C_0^2+\sum_{s=1}^{t-1}\left\|\ellb_s-\m_s\right\|_{\infty}^2}},\quad
        C_0=16R^2 L\sqrt{d^3\log(dT)\ln N},\quad
        \gamma=5L\sqrt{d^3\log(dT)},
    \end{equation*}
    for any time-varying comparators $\v_1,\ldots,\v_T\in\X$, \pref{alg:dynamic_regret} achieves the following dynamic regret:
    \begin{equation*}
        \E\big[\DReg_T\big]\le\Ot\sbr{\sqrt{d^3(1+P_T+ V_T)(1+P_T)}}.
    \end{equation*}
\end{Theorem}

\begin{proof}[of \pref{thm:dynamic-cvx}]
To begin with, we decompose the regret as follows:
\begin{equation*}
    \E\mbr{\DReg_T}=\underbrace{\E\mbr{\sumT\left\langle\mathbf{g}_t, \w_t-\w_{t, i^{\star}}\right\rangle}}_{\meta} + \underbrace{\E\mbr{\sumT\left\langle\mathbf{g}_t, \w_{t, i^{\star}}-\mathbf{v}_t\right\rangle}}_{\base}+\O(1).
\end{equation*}
Denote by
\begin{equation*}
    S_{T,i}^\w=\sumTT\norm{\w_{t,i}-\w_{t-1,i}}^2,\ S_{T}^{\p}=\sumTT\norm{\p_{t}-\p_{t-1}}_1^2,
    S_T^\w=\sumTT\norm{\w_t-\w_{t-1}}^2,\ S_T^{\mix}=\sumTT \sumN p_{t,i}\|\w_{t,i}-\w_{t-1,i}\|^2.
\end{equation*}
Then, by {\citet[Lemma 6]{yan2023universal}}, we have
\begin{equation}
    \label{eq:2-layer decompose}
    S_T^\w\le2S_T^{\mix}+2R^2 S_T^{\p}.
\end{equation}
\paragraph{Meta Regret Analysis.}For the meta regret, by \pref{lem:ohedge}, we have
\begin{align*}
    &\meta = \E\mbr{\sumT \inner{\ellb_t}{\p_t-\eb_\is}}+\gamma\E\mbr{S_{T,\is}^\w}-\gamma\E\mbr{S_T^{\mix}} \\
    \le{}&5\sqrt{\ln N\sbr{C_0^2+\E\mbr{\sumT \left\|\ellb_t-\m_t\right\|_{\infty}^2}}}+\gamma \E\mbr{S_{T,\is}^\w}-\gamma\E\mbr{S_T^{\mix}} -\frac{C_0}{4 \sqrt{\ln N}} \E\mbr{S_T^{\p}}+\frac{\sqrt{\ln N}}{C_0}\E\mbr{\max_{t\in[T]}\|\ellb_t-\m_t\|^2_\infty}\\
    \le{}&5\sqrt{\ln N\sbr{C_0^2+R^2\E\mbr{\sumT \left\|\g_t-\tgb_t\right\|_2^2}}}+\gamma\E\mbr{S_{T,\is}^\w}-\gamma\E\mbr{S_T^{\mix}}-\frac{C_0}{4 \sqrt{\ln N}} \E\mbr{S_T^{\p}}+\frac{4d^2G^2R^2\sqrt{\ln N}}{C_0},
\end{align*}
where the second step is by \pref{lem:ohedge}, and the last step follows from the definition of $\ellb_t$ and $\m_t$ in \pref{eq:dynamic-loss-optimism}, property~\rom{9} of \pref{lem:alpha}, and the boundedness $\|\x\|\le R$.
\paragraph{Base Regret Analysis.}For the base regret, substituting $\g_t$ and $\tgb_t$  into $\nabla f_t(\x_t)$ and $M_t$ in \pref{lem:OOGD-dynamic}, we have
\begin{equation*}
    \base \le \eta_\is\E\mbr{\sumT\left\|\g_t-\tgb_t\right\|_2^2}+\frac{4}{\eta_\is}\left(R^2+R P_T\right)-\frac{1}{4 \eta_\is} \E\mbr{S_{T,\is}^\w},
\end{equation*}
where $P_T=\sumTT\norm{\v_t-\v_{t-1}}$.
\paragraph{Overall Regret Analysis.}
Combining the meta regret and base regret, we have
\begin{align*}
    &\E\mbr{\DReg_T} \le 5\sqrt{\ln N\sbr{C_0^2+R^2\E\mbr{\sumT \left\|\g_t-\tgb_t\right\|_2^2}}}+\eta_\is \E\mbr{\sumT \|\g_t-\tgb_t\|_2^2}+\frac{4}{\eta_\is}(R^2+R P_T)\\
    & \qquad +\sbr{\gamma - \frac{1}{4 \eta_\is}} \E\mbr{S_{T,\is}^\w}-\gamma \E\mbr{S_T^{\mix}}-\frac{C_0}{4 \sqrt{\ln N}} \E\mbr{S_T^{\p}}+\O\sbr{\frac{4d^2G^2R^2\sqrt{\ln N}}{C_0}}\\
    \le{}& 5\sqrt{\ln N\sbr{C_0^2+8d^3R^2V_T\log d+16d^3L^2R^2\log(dT) \E\mbr{S_{T}^\w}}}\\
    & \qquad +8d^3\eta_\is V_T\log d+ 16d^3L^2\log(dT)\eta_{\is}\E\mbr{S_{T}^\w}+\frac{4}{\eta_\is}\left(R^2+R P_T\right)\\
    & \qquad +\sbr{\gamma - \frac{1}{4 \eta_\is}} \E\mbr{S_{T,\is}^\w}-\gamma \E\mbr{S_T^{\mix}}-\frac{C_0}{4 \sqrt{\ln N}} \E\mbr{S_T^{\p}}+\O\sbr{\frac{4d^2G^2R^2\sqrt{\ln N}}{C_0}}\\
      \le{}& 5\sqrt{\ln N\sbr{C_0^2+8d^3R^2V_T\log d}}+20\sqrt{\ln N}\sqrt{d^3L^2R^2\log(dT) \E\mbr{S_{T}^\w}}\\
    & \qquad +8d^3\eta_\is V_T\log d+ 16d^3L^2\log(dT)\eta_{\is}\E\mbr{S_{T}^\w}+\frac{4}{\eta_\is}\left(R^2+R P_T\right)\\
    & \qquad +\sbr{\gamma - \frac{1}{4 \eta_\is}} \E\mbr{S_{T,\is}^\w}-\gamma \E\mbr{S_T^{\mix}}-\frac{C_0}{4 \sqrt{\ln N}} \E\mbr{S_T^{\p}}+\O\sbr{\frac{4d^2G^2R^2\sqrt{\ln N}}{C_0}}\\
    \le{}&\O\sbr{\frac{4d^2G^2R^2\sqrt{\ln N}}{C_0}+\sqrt{d^3 V_T\ln N\log d}}+8d^3\eta_\is V_T\log d+\frac{20R^2\ln N}{\eta_\is}\\
    & \qquad +\sbr{16d^3L^2\log(dT)\eta_{\is}+20d^3L^2\log(dT)\eta_\is} \E\mbr{S_{T}^\w}+\frac{4}{\eta_\is}\left(R^2+R P_T\right)\\
    & \qquad +\sbr{\gamma - \frac{1}{4 \eta_\is}} \E\mbr{S_{T,\is}^\w}-\gamma \E\mbr{S_T^{\mix}}-\frac{C_0}{4 \sqrt{\ln N}} \E\mbr{S_T^{\p}}\\
    \le{}&\O\sbr{\frac{4d^2G^2R^2\sqrt{\ln N}}{C_0}+\sqrt{d^3\ln N\sbr{ V_T\log d+\ln N\log(dT)}}}+8d^3\eta_\is V_T\log d\\
    & \qquad +\sbr{32d^3L^2\log(dT)\eta_{\is}+40d^3L^2\log(dT)\eta_\is-\gamma} \E\mbr{S_{T}^{\mix}}\\
    & \qquad +\sbr{32d^3L^2R^2\log(dT)\eta_{\is}+40d^3L^2R^2\log(dT)\eta_\is-\frac{C_0}{4 \sqrt{\ln N}}}\E\mbr{S_T^{\p}}\\
    & \qquad +\sbr{\gamma - \frac{1}{4 \eta_\is}} \E\mbr{S_{T,\is}^\w}+\frac{4}{\eta_\is}\left(R^2+R P_T\right) \\
    \le{}&\O\sbr{\sqrt{d^3\sbr{ V_T\log d+\log(dT)}}+\frac{R^2+RP_T}{\eta_\is}+d^3 \eta_\is V_T\log d},
\end{align*}
where the second step is by \pref{lem:bandit-VT-correction-cvx}, the third step is by AM-GM inequality, the fourth step is by \pref{eq:2-layer decompose}, and the last step is by choosing
\begin{equation*}
   \gamma=5L\sqrt{d^3\log(dT)}, \ \eta_\is\le\frac{1}{20L\sqrt{d^3\log(dT)}},\text{ and } C_0=16R^2L\sqrt{d^3\log(dT)\ln N}.
\end{equation*}
Then, we define $\eta^\dagger=\min\bbr{\frac{1}{20L\sqrt{d^3\log(dT)}},\sqrt{\frac{R^2+RP_T}{d^3V_T\log d}}}$ and choose $\eta_\is\le \eta^\dagger\le 2\eta_\is$.
Thus, we have
\begin{align*}
    \E\mbr{\DReg_T}\le{}& \O\sbr{\sqrt{d^3\sbr{ V_T\log d+\log(dT)}}+\sqrt{d^3 V_T P_T\log d}+\sqrt{d^3\log(dT)}\sbr{R^2+RP_T}}\\
    \le{}& \O\sbr{\sqrt{d^3(\log(dT)+\log(dT) P_T+ V_T\log d)(1+P_T)}},
\end{align*}
which finishes the proof.
\end{proof}
Finally, for the sake of completeness, we provide the proof of \pref{lem:ohedge}.
\begin{proof}[of \pref{lem:ohedge}]
    First, we introduce a useful lemma about the regret guarantee of \ohedge.
    \begin{Lemma}[Theorem 7.47 of \citet{orabona2019modern}]
        \label{lem:ohedge_orabana}
        The regret of \ohedge~\eqref{eq:ohedge} with a decreasing learning rate $\varepsilon_t>0$ to any expert $i \in[N]$ satisfies
        \begin{align*}
        \sumT \inner{\p_t}{\ellb_t}-\sumT  \ell_{t, i} \le \max_{\p \in \Delta_N} \psi_{T+1}(\p)+\sumT \inner{\ellb_t-\m_t}{\p_t-\p_{t+1}}-\sumT  \frac{1}{2 \varepsilon_{t-1}}\left\|\p_t-\p_{t+1}\right\|_1^2,
        \end{align*}
    where $\psi_{t+1}(\p)=\frac{1}{\epsilon_t}\sbr{\sum_{i=1}^{N} p_i\log p_i+\ln N}$.
    \end{Lemma}
By \pref{lem:ohedge_orabana}, we have
\begin{align*}
&\sumT \left\langle \p_t, \ellb_t\right\rangle-\sumT  \ell_{t, i}\le
\frac{\ln N}{\varepsilon_T}
+\sum_{t=1}^T\varepsilon_{t-1} \|\ellb_t-\m_t\|_\infty^2 -\sum_{t=1}^T\frac{1}{4\varepsilon_{t-1}}\|\p_t-\p_{t+1}\|_1^2\\
\le{}&5\sqrt{\ln N\sbr{C_0^2+\sumT  \left\|\ellb_t-\m_t\right\|_{\infty}^2}}-\sumT  \frac{C_0}{4 \sqrt{\ln N}}\left\|\p_t-\p_{t+1}\right\|_1^2+\frac{\sqrt{\ln N}}{C_0}\max_{t\in[T]}\|\ellb_t-\m_t\|^2_\infty.
\end{align*}
where the last step is by \pref{lem:sum_cvx}.
\end{proof}

\subsection{Omitted Details in \pref{subsec:universal}}
\label{app:universal}
In this part, we provide the omitted details for our applications in universal regret.
Following the online ensemble framework in \citet{yan2024simple}, we design new surrogate loss functions to accommodate different types of functions, including strongly convex, convex, and linear functions.
Specifically, we instantiate multiple base learners for strongly convex functions with varying strong convexity parameters from a predefined pool, along with one base learner each for convex and linear functions.
The meta-algorithm aggregates the decisions of all base learners using the \omlprod algorithm~\citep{wei2016tracking}.

We first provide a decomposition of the non-consecutive gradient variation $\Vb_T$~\eqref{eq:VbT} for linear functions, analogous to that in \pref{lem:bandit-VT-bregman-cvx}.
\begin{Lemma}[Lemmas 7,9 of \citet{chiang2013beating}]
\label{lem:bandit-VT-bregman-lin}
    Under Assumptions \ref{ass:boundedness}-\ref{ass:Smoothness}, for linear functions, \pref{alg:chiang} satisfies:
    \begin{equation*}
        \E\mbr{\Vb_T} \le 2d^3V_T + \O(1).
    \end{equation*}
\end{Lemma}

Next, similar to the dynamic regret analysis in \pref{app:dynamic-cvx}, we elaborate on the online scheduling strategy, the base learners, and the meta-algorithm in the following.
\paragraph{Online Scheduling Strategy and Base Learners.}
To tackle different types of functions, we design multiple base learners as instances of \pref{alg:chiang} with distinct step size schedules.
To address the unknown strong convexity $\lambda$, we employ a set of learners with candidate parameters from a predefined pool, in addition to single base learners for linear and convex function classes.
Specifically, for the nondegenerate case of $\lambda \in [1/T, 1]$, we can discretize the unknown $\lambda$ into a candidate pool $\H^{\scvx}$ using an exponential grid, defined as
\begin{equation}
    \label{eq:candidate-pool}
\H^{\scvx} \triangleq \left\{ \frac{1}{T}, \frac{2}{T}, \frac{2^2}{T}, \cdots, \frac{2^{N_\scvx-1}}{T} \right\},
\end{equation}
where $N_\scvx = \ceil{\log_2 T} + 1 = \O(\log T)$ is the number of candidates.
It can be proved that the discretized candidate pool $\H^{\scvx}$ can approximate the continuous value of $\lambda$ with only constant errors.
In total, there are $N \define 2 + \abs{\H^{\scvx}} = 2+N_\scvx = \O(\log T)$ base learners, where the first $N_\scvx$ base learners correspond to strongly convex functions with different guessed strong convexity coefficients from $\H^{\scvx}$, and the last two base learners correspond to convex and linear functions, respectively.
The best base learner is the one with the right guess of the curvature type and the closest guess of the curvature coefficient.
For example, suppose the online functions are $\lambda$-strongly convex (while this is unknown to the online learner), then the right guessed coefficient of the best base learner (indexed by $\is$) satisfies $\lambda_\is \le \lambda \le 2 \lambda_\is$.

To handle limited information, we devise surrogate loss functions for the various base learner categories.
This allows us to treat each base learner as an \OOGD \eqref{eq:OOGD} instance operating on the constructed surrogate feedback at each round $t$.
Specifically, the configurations of base learners (surrogate loss functions and \OOGD step sizes) are summarized in \pref{table:base-learners}.
\begin{table}[!t]
    \centering
    \caption{Summary of Base Learner Configurations.}
    \label{table:base-learners}
    \renewcommand{\arraystretch}{1.3}
    \begin{tabular}{c|c|c|c}
    \hline

    \hline
     & \textbf{Strongly Convex ($\lambda_i$)} & \textbf{Convex} & \textbf{Linear} \\ \hline
    \textbf{Surrogate Loss} & $\hsc_{t,i}(\x) \define \inner{\g_t}{\x} + \frac{\lambda_i}{4}\|\x - \w_t\|^2$ & $h^{\cvx}_t(\x) \define \inner{\g_t}{\x}$ & $h^{\lin}_t(\x) \define \inner{\g_t}{\x}$ \\ \hline
    \rule{0pt}{5mm}\textbf{\OOGD Step Size} & $\eta_t = \frac{4}{\lambda_i t}$ & $\eta_t = \frac{2R}{\sqrt{d^2+\sum_{s=1}^{t-1}\|\mathbf{g}_s-\tilde{\mathbf{g}}_s\|^2}}$ & $\eta_t = \frac{2R}{\sqrt{d^2+\sum_{s=1}^{t-1}\|\mathbf{g}_s-\tilde{\mathbf{g}}_s\|^2}}$ \\[2mm]
    \hline

    \hline
    \end{tabular}
\end{table}

Here, we present the regret guarantees of \OOGD for convex and strongly convex functions, which will be utilized in our analysis of universal regret.
\begin{Lemma}[Lemma 10 of \citet{yan2023universal}]
    \label{lem:cvx-base}
    Under Assumptions \ref{ass:boundedness}-\ref{ass:Smoothness},
    if the loss functions are convex, \OOGD~\eqref{eq:OOGD} with $\eta_t =2R / \sqrt{\kappa^2+A_{t-1}}$, where $\Bottomcoef$ is a parameter to be specified, enjoys the following regret guarantee: for any $\v\in\X$
    \begin{equation*}
        \sumT f_t(\x_t) - \sumT f_t(\v) \le 10R \sqrt{\kappa^2 + A_T} + \Bottomcoef R-\frac{\Bottomcoef}{4} \sumTT \|\x_t - \x_{t-1}\|^2 + \frac{1}{\kappa}\O\sbr{\max_{t \in [T]} \|\nabla f_t(\x_t) - M_t\|^2},
    \end{equation*}
    where $A_t\define  \sum_{s=1}^t \|\nabla f_s(\x_s) - M_s\|^2$ for any $t\in[T]$.
\end{Lemma}
\begin{Lemma}[Lemma 12 of \citet{yan2023universal}]
    \label{lem:scvx-base}
    Under Assumptions \ref{ass:boundedness}-\ref{ass:Smoothness}, if the loss functions are $\lambda$-strongly convex, \OOGD~\eqref{eq:OOGD} with $\eta_t = 2/(\lambda t)$ enjoys the following regret guarantee: for any $\v\in\X$
    \begin{equation*}
        \sumT f_t(\x_t) - \sumT f_t(\v)  \le 2 \sumT \eta_t \|\nabla f_t(\x_t) - M_t\|^2 + \O(1).
    \end{equation*}
\end{Lemma}

\paragraph{Meta-algorithm: \omlprod.}
We introduce the \omlprod algorithm as our meta algorithm to combine the decisions of all base learners.

Specifically, the weight vector $\p_{t+1} \in \Delta_N$ is updated as follows: $\forall i \in [N]$,
\begin{equation}
    \label{eq:omlprod}
    \begin{gathered}
        \forall t \ge 1, W_{t,i} = \sbr{W_{t-1,i} \cdot \exp\sbr{\epsilon_{t-1,i} r_{t,i} - \epsilon_{t-1,i}^2 (r_{t,i} - m_{t,i})^2}}^{\frac{\epsilon_{t,i}}{\epsilon_{t-1,i}}}, \text{ and } W_{0,i} = \tfrac{1}{N},\\
        p_{t+1,i} \propto \epsilon_{t,i} \cdot \exp(\epsilon_{t,i} m_{t+1,i}) \cdot W_{t,i},
    \end{gathered}
\end{equation}
where $W_{t,i}$ and $\epsilon_{t,i}$ denote the potential variable and learning rate for the $i$-th base learner, respectively.
The feedback loss vector $\ellb_t \in \R^N$ is configured as
$\ell_{t,i} \define \frac{1}{2\sqrt{10}dGR}\inner{\g_t}{\w_{t,i}} + \frac12 \in [0,1]$, where by property \rom{6} of \pref{lem:alpha}, $\|\g_t\|\le \sqrt{10}dG$.
$r_{t,i} = \inner{\p_t}{\ellb_t} - \ell_{t,i}$ measures the instantaneous regret.
The optimistic vector $\m_t \in \R^{N}$ is designed as\footnote{While $\w_t$ depends on $\m_{t,i}$, the update only requires the scalar $\langle \tgb_t, \w_t \rangle$, which can be obtained by solving the one-dimensional fixed-point problem $\langle \tgb_t, \w_t(z) \rangle = z$. Here, the function $\w_t(z)$ is induced by the dependency chain $\w_t(p_{t,i}(\m_{t,i}(z)))$. For a comprehensive derivation, see \citet{wei2016tracking}}
\begin{equation}
    \label{eq:Bregman-optimism}
      m_{t,i} = 0 \mbox{ for } i \in [N_{\scvx}], \mbox{ and } m_{t,i} = \tfrac{1}{2\sqrt{10}dGR}\inner{\tgb_t}{\w_t - \w_{t,i}} \mbox{ for } i \in \{N-1,N\},
\end{equation}
where the last two entries correspond to the convex and linear base learners, respectively.

For any $i\in[N]$, the learning rate $\epsilon_{t,i}$ is set as
\begin{equation}
    \label{eq:omlprod-lr}
    \epsilon_{t,i} = \min \bbr{\frac{1}{8},\ \sqrt{\frac{\log N}{\sum_{s=1}^t (r_{s,i} - m_{s,i})^2}}}.
\end{equation}

Here, we present the guarantee of \omlprod, which will be utilized in our analysis of universal regret.
\begin{Lemma}[Theorem 3.4 of \citet{wei2016tracking}]
        \label{lem:optimistic-mlprod}
        Denote by $\p_t \in \Delta_N$ the algorithm's weights, $\ellb_t \in [0,1]^N$ the loss vector, and $m_{t,i}$ the optimism.
        With the learning rate in \eqref{eq:omlprod-lr}, the regret of \omlprod~\eqref{eq:omlprod} with respect to any expert $i \in [N]$ satisfies
        \begin{equation*}
            \sumT \inner{\ellb_t}{\p_t - \eb_i} \le C_1 \sqrt{1 + \sumT (r_{t,i} - m_{t,i})^2} + C_2,
        \end{equation*}
        where $\eb_i$ is the $\ith$ standard basis vector, $C_1 = \sqrt{\log N} + \log (1 + \frac{N}{e} (1 + \log (T+1))) / \sqrt{\log N}$, and $C_2 = \frac{1}{4}(\log N + \log (1 + \frac{N}{e} (1 + \log (T+1)))) + 2 \sqrt{\log N} + 16 \log N$.
    \end{Lemma}
Since the number of base learners is $N = \O(\log T)$, the constants $C_1$ and $C_2$ are of order $\O(\log \log T)$ and can be ignored.

\paragraph{Overall Algorithm and Universal Regret Analysis.}
Integrating the above components, we propose \pref{alg:UniGrad-Bregman-1grad} for universal regret minimization in two-point BCO, which attains the performance guarantee established in \pref{thm:2point-BCO-Bregman}.
\begin{Theorem}
    \label{thm:2point-BCO-Bregman}
    Under Assumptions \ref{ass:boundedness}, \ref{ass:Lipschitzness}, \ref{ass:Smoothness++}, Algorithm~\ref{alg:UniGrad-Bregman-1grad} achieves the following universal regret:
        \begin{equation*}
        \E[\Reg_T] \le
        \begin{cases}
            \O\sbr{\frac{d}{\lambda}\log(dV_T)}, & \textnormal{($\lambda$-strongly convex case)}, \\[2mm]
            \Ot\sbr{\sqrt{d^3V_T}+ d^3}, & \textnormal{(convex case)},\\[2mm]
            \O\sbr{\sqrt{d^3V_T}},& \textnormal{(linear case)}.
        \end{cases}
    \end{equation*}
\end{Theorem}

\begin{algorithm}[!t]
    \caption{Universal Algorithm for Gradient-Variation Regret in Two-Point BCO}
    \label{alg:UniGrad-Bregman-1grad}
    \begin{algorithmic}[1]
    \Require Base learner configurations $\{\Bc_i\}_{i \in [N]}\define \{\Bc_i^\scvx\}_{ i\in [N_\scvx]} \cup \Bc^{\cvx}\cup\Bc^{\lin}$
        \State \textbf{Initialize}: $\M$~---~meta learner \omlprod with $W_{0,i} = \frac{1}{N}$ for all $i \in [N]$ \\
        \makebox[0.5cm]{} $\{\Bc_i\}_{i \in [N]}$~---~base learners as specified in \pref{app:universal}
    \For{$t=1$ {\bfseries to} $T$}
       \State Choose $i_t$ uniformly from $[d]$.
       \State Submit $\x_t = \w_t + \delta \eb_{i_t}$, and $\x_t^\prime = \w_t - \delta \eb_{i_t}$, where $\w_t = \sumN p_{t,i} \w_{t,i}$
       \State Observe partial information $f_t(\x_t)$ and $f_t(\x_t^\prime)$
       \State Construct the gradient estimator $\g_t$ and the optimism $\tgb_t$ as in \eqref{eq:chiangestimator}
       \State $\{\Bc_i^\scvx\}_{i \in [N_{\scvx}]}$, $\Bc^{\cvx}$, and $\Bc^\lin$ update their own decisions to $\{\w_{t+1,i}\}_{i=1}^N$ using surrogate losses of $\{\hsc_{t,i}(\cdot)\}_{\lambda_i \in \H^{\scvx}}$, $h_t^{\cvx}(\cdot)$, and $h_t^{\lin}(\cdot)$ in \pref{table:base-learners} \label{line:bregman-1grad}
       \State Calculate $\m_{t+1}$~\eqref{eq:Bregman-optimism} and $\r_t$ using $\{\w_{t,i}\}_{i=1}^N$, $\w_t$, $\g_t$, $\tgb_t$, and $\{\w_{t+1,i}\}_{i=1}^N$, send them to $\M$, and obtain $\p_{t+1} $
    \EndFor
    \end{algorithmic}
\end{algorithm}

\begin{Remark}
In contrast to the full-information setting, where universal guarantees are commonly studied for convex, exp-concave, and strongly convex functions, we consider linear, convex, and strongly convex functions under two-point bandit feedback. On the one hand, linear functions constitute a fundamental function class that has been extensively studied in the bandit literature. On the other hand, exp-concave functions are substantially more challenging under bandit feedback, as exploiting their curvature structure typically requires gradient information. Indeed, obtaining worst-case regret guarantees for exp-concave functions with two-point feedback remains an open problem.
\end{Remark}
We first present a key lemma that is crucial for analyzing the meta-regret of \omlprod.
\begin{Lemma}
    \label{lem:meta-gradient-inner-product}
    Under Assumptions \ref{ass:boundedness}-\ref{ass:Smoothness}, with the same exploration parameter as in \pref{lem:alpha}, for the best strongly convex base learner indexed by $\is$ in \pref{alg:UniGrad-Bregman-1grad}, we have
    \begin{equation*}
        \E\mbr{\inner{\g_t}{\w_t-\w_{t,\is}}^2}
        \le 10dG^2\E\mbr{\norm{\w_t-\w_{t,\is}}^2}.
    \end{equation*}
\end{Lemma}

\begin{proof}
We first give different regret decompositions, then analyze the meta regret, and finally provide the proofs for different kinds of loss functions.
\paragraph{Regret Decomposition.}
For \emph{$\lambda$-strongly convex} functions, we have the following decomposition:
    \begin{align}
        \E[\Reg_T] \le {} & \underbrace{\E\mbr{\sumT \inner{\g_t}{\w_t - \w_{t,\is}} - \frac{\lambda_\is}{4} \sumT \|\w_t - \w_{t,\is}\|^2}}_{\meta} \tag*{(by $\lambda_\is \le \lambda \le 2 \lambda_\is$)}\\
        & \qquad + \underbrace{\E\mbr{\sumT \hsc_{t,\is}(\w_{t,\is}) - \sumT \hsc_{t,\is}(\ws)}}_{\base} - \frac{1}{2}\E\mbr{\sumT \D_{f_t}(\ws, \w_t)}+\O(1), \label{eq:thm9 sc-reg-de}
    \end{align}
    due to the definition of the surrogate $\hsc_{t,i}(\x) \define \inner{\g_t}{\x} + \frac{\lambda_i}{4} \|\x - \w_t\|^2$, where $\lambda_i \in \H^\scvx$ in~\eqref{eq:candidate-pool}.
    For \emph{convex} functions, we decompose the regret as
    \begin{align}
        &\E[\Reg_T] \le \E\mbr{\sumT \inner{\g_t}{\w_t - \ws} - \sumT \D_{f_t}(\ws, \w_t)}+\O(1)\notag\\
        &= \underbrace{\E\mbr{\sumT \inner{\g_t}{\w_t - \w_{t,\is}}}}_{\meta} + \underbrace{\E\mbr{\sumT \hc_{t,\is}(\w_{t,\is}) - \sumT \hc_{t,\is}(\ws)}}_{\base} - \E\mbr{\sumT \D_{f_t}(\ws, \w_t)}+\O(1), \label{eq:thm9 cvx-reg-de}
    \end{align}
    where $\hc_{t,i}(\x) \define \inner{\g_t}{\x}$.

    For \emph{linear} functions, since $\D_{f_t}(\ws, \w_t) = 0$ for $t\in[T]$, we decompose the regret as
    \begin{align}
    \E[\Reg_T] \le \underbrace{\E\mbr{\sumT \inner{\g_t}{\w_t - \w_{t,\is}}}}_{\meta} + \underbrace{\E\mbr{\sumT \hl_{t,\is}(\w_{t,\is}) - \sumT \hl_{t,\is}(\ws)}}_{\base} +\O(1), \label{eq:thm9 lin-reg-de}
    \end{align}
 where $\hl_{t,i}(\x) \define \inner{\g_t}{\x}$.

    \paragraph{Meta Regret Analysis.}
    For \emph{$\lambda$-strongly convex} functions, we have
    \begin{align}
        \meta = {} & \E\mbr{4\sqrt{10}dGR\sumT \inner{\ellb_t}{\p_t - \eb_\is} - \frac{\lambda_\is}{4} \sumT \|\w_t - \w_{t,\is}\|^2}\notag \\
        \le{}& C_1 \E\mbr{\sqrt{160d^2G^2R^2 + \sumT \inner{\g_t}{\w_t - \w_{t,\is}}^2}} - \frac{\lambda_\is}{4} \E\mbr{\sumT \|\w_t - \w_{t,\is}\|^2} + 2\sqrt{10}dGRC_2  \notag\\
        \le{}&\O(d)+C_1\sqrt{\sumT 10dG^2\E\mbr{\|\w_t-\w_{t,\is}\|^2}}- \frac{\lambda_\is}{4} \E\mbr{\sumT \|\w_t - \w_{t,\is}\|^2}\notag\\
        \le {} & \O(d) + \sbr{\frac{10C_1 G^2}{C_3} - \frac{\lambda_\is}{4}} \E\mbr{\sumT \|\w_t - \w_{t,\is}\|^2}, \label{eq:thm9 scvx-meta}
    \end{align}
    where the first step is by definition, the second step is by \pref{lem:optimistic-mlprod}, the third step is by \pref{lem:meta-gradient-inner-product}, and the last step  is due to AM-GM inequality.
    $C_3$ is a $d$-\emph{independent} constant to be specified.

    For \emph{convex} functions, we bound the meta regret by
    \begin{align}
        \meta \le {} & C_1 \sqrt{160d^2G^2R^2 + \E\mbr{\sumT \inner{\g_t - \tgb_t}{\w_t - \w_{t,\is}}^2}} + 2\sqrt{10}dGR C_2\notag\\
        \le {}&C_1 \sqrt{160d^2G^2R^2 + 48d^3 R^2V_T\log d + 192d^3LR^2\log(dT) \E\mbr{\sumT \D_{f_t}(\ws, \w_t)}} + 2\sqrt{10}dGR C_2\notag\\
        \le {} & \O(\sqrt{d^3 V_T\log d}) + C_1\sqrt{192d^3LR^2\log(dT) \E\mbr{\sumT \D_{f_t}(\ws, \w_t)}} \notag\\
        \le{}& \O(\sqrt{d^3 V_T\log d}) + \O(C_4 d^3 \log(dT) ) + \frac{C_1}{2C_4} \E\mbr{\sumT \D_{f_t}(\ws,\w_t)},
        \label{eq:thm9 cvx-meta VT}
  \end{align}
    where the first step is by \pref{lem:optimistic-mlprod} and Jensen's inequality, the second step is by \pref{lem:bandit-VT-bregman-cvx} and $\|\w_t-\w_{t,\is}\|\le 2R$, the third step uses $\sqrt{a+b} \le \sqrt{a} + \sqrt{b}$ for any $a,b \ge 0$, the last step uses AM-GM inequality: $\sqrt{ab} \le \frac{ax}{2} + \frac{b}{2x}$ for any $a,b,x > 0$.
Note that the $d,T$-independent constant $C_4$ is used to ensure that the Bregman divergence term is canceled and will be specified in the end.

    For \emph{linear} functions, we bound the meta regret by
    \begin{align}
      \meta \le {} & C_1 \sqrt{160d^2G^2R^2 + \E\mbr{\sumT \inner{\g_t - \tgb_t}{\w_t - \w_{t,\is}}^2}} + 2\sqrt{10}dGR C_2\notag\\
      \le {}&C_1 \sqrt{160d^2G^2R^2 + 8d^3 R^2V_T} + 2\sqrt{10}dGR C_2
      \le  \O(\sqrt{d^3 V_T}),
      \label{eq:thm9 lin-meta VT}
  \end{align}
where the second step is by \pref{lem:bandit-VT-bregman-lin}.
\paragraph{Base Regret Analysis.}
    In this part, we first provide different decompositions of the non-consecutive gradient variation defined on surrogates for strongly convex and convex functions, respectively, and then analyze the base regret in the corresponding cases.

    For \emph{$\lambda$-strongly convex} functions, we apply \pref{lem:scvx-base} to surrogate loss function sequence $\{\hsc_{t,\is}\}_{t=1}^T$.
    By choosing the optimism $M_t$ as $\tgb_t + \frac{\lambda_\is}{2} (\w_{t-1,\is} - \w_{t-1})$, we bound the non-consecutive gradient variation on surrogates, i.e., $\Vb_{T,\is}^{\scvx} \define \sumTT \|\nabla \hsc_{t,\is}(\w_{t,\is}) -M_t\|^2$, by
\begin{align*}
        \E[\Vb_{T,\is}^{\scvx}] = {} & \E\mbr{\sumTT \norm{\sbr{\g_t + \frac{\lambda_\is}{2} (\w_{t,\is} - \w_t)} - \sbr{\tgb_t + \frac{\lambda_\is}{2} (\w_{t-1,\is} - \w_{t-1})}}^2} \notag\\
        \le {} & \E\mbr{3\sumTT \|\g_t - \tgb_t\|^2 + 3 \sumTT \norm{\frac{\lambda_\is}{2} (\w_{t,\is} - \w_t)}^2 + 3 \sumTT \norm{\frac{\lambda_\is}{2} (\w_{t-1,\is} - \w_{t-1})}^2}\notag\\
        \le {} & 36 d^3V_T \log d + 144d^3L\log(dT)\E\mbr{\sumT \D_{f_t}(\ws, \w_t)}+ 2 \lambda_\is^2 \E\mbr{\sumTT \norm{\w_{t,\is} - \w_t}^2},
    \end{align*}
    where the first step uses the definition of $\nabla \hsc_{t,i}(\x) = \g_t + \frac{\lambda_i}{2} (\x - \w_t)$ and the last step is due to \pref{lem:bandit-VT-bregman-cvx}.

    For \emph{convex} and \emph{linear} functions, the non-consecutive gradient variation $\Vb^{\cvx}_{T,\is} \define \sumTT \|\g_t- \tgb_t\|^2$ and $\Vb^{\lin}_{T,\is} \define \sumTT \|\g_t- \tgb_t\|^2$ can be bounded via \pref{lem:bandit-VT-bregman-cvx} and \pref{lem:bandit-VT-bregman-lin}, respectively.

    To conclude, for different curvature types, we provide correspondingly different analyses of the non-consecutive gradient variation on surrogates:
    \begin{equation*}
    \E[\Vb_{T,\is}^{\{\scvx,\cvx, \lin\}}] \le
    \begin{cases}
        36 d^3 V_T\log d + 144d^3L\log(dT)\E\mbr{\sum\limits_{t=1}^T \D_{f_t}(\ws, \w_t)}+ 2 \lambda_\is^2 \E\mbr{\sum\limits_{t=2}^T \norm{\w_{t,\is} - \w_t}^2},
        & \text{($\lambda$-strongly convex)}, \\[5mm]
        12d^3 V_T\log d + 48d^3 L \log(dT) \E\mbr{\sum\limits_{t=1}^T \D_{f_t}(\ws, \w_t)}, & \text{(convex)},\\[5mm]
        2d^3 V_T, & \text{(linear)}.
    \end{cases}
    \end{equation*}
    In the following, we analyze the base regret for different curvature types.

    For \emph{$\lambda$-strongly convex} functions, we have for any $t \ge 1$,
    \begin{align*}
       &\E\mbr{\|\nabla \hsc_t(\w_{t,\is})-M_t\|^2}=\E\mbr{\left\|\sbr{\g_t + \frac{\lambda_\is}{2} (\w_{t,\is} - \w_t)} - \sbr{\tgb_t + \frac{\lambda_\is}{2} (\w_{t-1,\is} - \w_{t-1})}\right\|^2}\\
       \le{}&3 \E\mbr{\|\g_t - \tgb_t\|^2} + 3 \E\mbr{\left\|\frac{\lambda_\is}{2} (\w_{t,\is} - \w_t)\right\|^2} + 3 \E\mbr{\left\|\frac{\lambda_\is}{2} (\w_{t-1,\is} - \w_{t-1})\right\|^2}\\
       \le{}& 24dG^2+8R^2 \lambda_\is^2+\O\sbr{\frac{1}{T}},
    \end{align*}
    where the last step is by \pref{lem:alpha} and the boundedness of $\X$.
    We denote $C_5= 24dG^2 + 8R^2 \lambda_\is^2+\O\sbr{\frac{1}{T}}$ for simplicity.

    Thus, by \pref{lem:scvx-base}, the $\is$-th base learner guarantees:
    \begin{align}
      &\base \le 2 \E\mbr{\sumT \eta_t \|\nabla \hsc_t(\w_{t,\is}) - M_t\|^2 }+ \O(1)
      \le \frac{C_5}{\lambda_\is} \log \sbr{1 + \lambda_\is \E\mbr{\Vb_{T,\is}^{\scvx}}} \notag\\
      \le {} & \frac{C_5}{\lambda_\is} \log \sbr{1 + 36 d^3\lambda_{\is} V_T\log d + 144d^3L\lambda_{\is}\log(dT)\E\mbr{\sumT \D_{f_t}(\ws, \w_t)}+ 2 \lambda_\is^3 \E\mbr{\sumTT \norm{\w_{t,\is} - \w_t}^2}}\notag\\
      \le{}& \frac{C_5}{\lambda_\is} \log \sbr{C_6\sbr{ 1+ 36 d^3\lambda_{\is} V_T\log d + 144d^3L\lambda_{\is}\log(dT)\E\mbr{\sumT \D_{f_t}(\ws, \w_t)}}}+ \frac{\lambda_\is}{8}  \E\mbr{\sumTT \norm{\w_{t,\is} - \w_t}^2}\notag \\
      \le {} & \O\sbr{\frac{d}{\lambda_\is}\log \sbr{V_T+\log(dT)}} + \frac{1}{2}\E\mbr{\sumT \D_{f_t}(\ws, \w_t)} + \frac{\lambda_\is}{8} \E\mbr{\sumT \norm{\w_{t,\is} - \w_t}^2},\label{eq:thm9 scvx base VT}
    \end{align}
    where the first inequality is by \pref{lem:sum_scvx}, the fourth step is by $C_6=16C_5$, $\ln(1+x)\le x$ for $x\geq0$, and $\lambda_{\is}\le\lambda\le1$.
    The last step is by the following analysis.
    If $\E\mbr{\sumT \D_{f_t}(\ws, \w_t)}\le1$, then we  obtain
    \begin{align*}
    \log \sbr{C_6\sbr{1 + 36d^3\lambda_{\is} V_T\log d + 144d^3L\lambda_{\is}\log(dT)\E\mbr{\sumT \D_{f_t}(\ws, \w_t)}}}\le \O(\log \sbr{dV_T+d\log(dT)}).
    \end{align*}

     Otherwise, we have
  \begin{align*}
    &\log \sbr{C_6\sbr{ 1 + 36 d^3\lambda_{\is} V_T\log d + 144d^3L\lambda_{\is}\log(dT)\E\mbr{\sumT \D_{f_t}(\ws, \w_t)}}}\\
    \le{}& \log C_6+\log\sbr{ 1+\sbr{36 d^3V_T\log d + 144d^3L\log(dT) \lambda_{\is}}\E\mbr{\sumT \D_{f_t}(\ws, \w_t)}} \\
    \le{}&\O(\log \sbr{dV_T+d\log(dT)})+ \frac{1}{2}\E\mbr{\sumT \D_{f_t}(\ws, \w_t)},
  \end{align*}
where the first inequality is due to $\E\mbr{\sumT \D_{f_t}(\ws, \w_t)}>1$ and the last inequality is by $\ln (1+x)\le x$ for $x\ge0$.

    For \emph{convex} functions, we apply \pref{lem:cvx-base} to surrogate loss function sequence $\{\hc_{t,\is}\}_{t=1}^T$.
    Choosing $M_t=\tgb_t$ and $\kappa=d$,
    \begin{align}
        & \base \le 10R \sqrt{d^2 + \E[\Vb_{T,\is}^\cvx] } + \O(d R)+\frac{1}{d}\O\sbr{\max_{t \in [T]} \|\g_t - \tgb_t\|^2}
        \le 10R \sqrt{d^2 + \E[\Vb_{T,\is}^\cvx] } +\O\sbr{d}\notag\\
        \le{}& \O(\sqrt{d^2+d^3V_T\log d}) + 10R \sqrt{48d^3L\log(dT) \E\mbr{\sumT \D_{f_t}(\ws, \w_t)}} +\O\sbr{d}\tag*{(by \pref{lem:bandit-VT-bregman-cvx})}\\
        \le {} & \O(\sqrt{d^3V_T\log d}) + \O(d^3\log(dT)) + \frac{10RL}{2C_7} \E\mbr{\sumT \D_{f_t}(\ws, \w_t)},\label{eq:thm9 cvx base VT}
    \end{align}
    where the first step is by Jensen's inequality, the second step is by \pref{lem:alpha}, and the last step uses AM-GM inequality.
    $C_7$ is a $d$-independent constant to be specified.

    For \emph{linear} functions, we apply \pref{lem:cvx-base} to surrogate loss function sequence $\{\hl_{t,\is}\}_{t=1}^T$.
    Choosing $M_t=\tgb_t$ and $\kappa=d$,
    \begin{align}
        \base \le {} & 10R \sqrt{d^2 + \E[\Vb_{T,\is}^\lin] } + \O(d R)+\frac{1}{d}\O\sbr{\max_{t \in [T]} \|\g_t - \tgb_t\|^2} \notag\\
        \le{}& 10R \sqrt{d^2 + \E[\Vb_{T,\is}^\lin] } +\O\sbr{d}\le \O\sbr{\sqrt{d^3 V_T}},\label{eq:thm9 lin base VT}
    \end{align}
where the first step is by Jensen's inequality.

\paragraph{Overall Regret Analysis.}
    For \emph{$\lambda$-strongly convex} functions, plugging \pref{eq:thm9 scvx-meta} and \pref{eq:thm9 scvx base VT} into \pref{eq:thm9 sc-reg-de} and choosing $C_3={160C_1G^2}/{\lambda}$, we obtain
    \begin{align*}
        \E[\Reg_T]\le{}& \O\sbr{\frac{d}{\lambda}\log \sbr{V_T+\log(dT)}}+\O(d)+\sbr{\frac{1}{2}-\frac{1}{2}} \E\mbr{\sumT \D_{f_t}(\ws, \w_t)}\\
        &+\sbr{\frac{10C_1 G^2}{C_3} +\frac{\lambda_\is}{8} - \frac{\lambda_\is}{4}}\E\mbr{\sumT \norm{\w_t - \w_{t,\is}}^2}
        \le\O\sbr{\frac{d}{\lambda}\log \sbr{V_T+\log(dT)}}.
    \end{align*}
    For \emph{convex} functions, plugging \pref{eq:thm9 cvx-meta VT} and \pref{eq:thm9 cvx base VT} into \pref{eq:thm9 cvx-reg-de}, we obtain
    \begin{align*}
        \E[\Reg_T]\le{}& \O(\sqrt{d^3 V_T\log d})+\O(d^3\log(dT))\\
        &\qquad+\sbr{\frac{C_1}{2C_4}+\frac{10RL}{2C_7}-1} \E\mbr{\sumT \D_{f_t}(\ws, \w_t)}\le  \O\sbr{\sqrt{d^3V_T\log d}+d^3\log(dT)},
    \end{align*}
    by choosing $C_4=C_1$ and $C_7=10RL$.

    For \emph{linear} functions, plugging \pref{eq:thm9 lin-meta VT} and \pref{eq:thm9 lin base VT} into \pref{eq:thm9 lin-reg-de}, we obtain
    \begin{align*}
        \E[\Reg_T]\le{}& \O(\sqrt{d^3 V_T}).
    \end{align*}
    Note that the constants $C_3, C_4, C_6,C_7$ only exist in analysis and thus can be chosen arbitrarily, which finishes the proof.
\end{proof}

Finally, we present the proofs of Lemmas~\pref{lem:bandit-VT-bregman-cvx} and~\pref{lem:meta-gradient-inner-product} for completeness.

\begin{proof}[of \pref{lem:bandit-VT-bregman-cvx}]
    By \pref{lem:alpha}, we have
 \begin{align}
        &\Vb_T=\sumT \|\g_t - \tgb_t\|^2\le d^2\sumT\sum_{i=1}^d \rho_{t,i} (\nabla_{i} f_t(\w_t) - \nabla_{i} f_{t-1}(\w_{t-1}))^2 + \O(1)\notag\\
        \le {} &3d^2\sumT\sum_{i=1}^d \rho_{t,i} (\nabla_{i} f_t(\w_t) - \nabla_{i} f_{t}(\w^\star))^2 +3d^2\sumT\sum_{i=1}^d \rho_{t,i} (\nabla_{i} f_t(\w^\star) - \nabla_{i} f_{t-1}(\w^\star))^2\notag\\
        &+3d^2\sumT\sum_{i=1}^d \rho_{t,i} (\nabla_{i} f_{t-1}(\w^\star) - \nabla_{i} f_{t-1}(\w_{t-1}))^2+ \O(1),\label{eq:correction eq1}
    \end{align}
    where $\rho_{t, i}$ is defined in \pref{eq:rho} and the second step is due to Cauchy-Schwarz inequality.
    Taking expectation over the randomness of $\seq{\eb_{i_t}}$, we bound the second term in \pref{eq:correction eq1} by
    \begin{equation}
        \E\mbr{\sumT\sum_{i=1}^d \rho_{t,i} (\nabla_{i} f_t(\w^\star) - \nabla_{i} f_{t-1}(\w^\star))^2}\le \E\mbr{\sumT \max_{i\in[d]}[\rho_{t,i}] \norm{\nabla f_t(\w^\star) - \nabla f_{t-1}(\w^\star)}^2}\le 4d\log d V_T,\label{eq:correction eq2}
    \end{equation}
where the second step is by \pref{eq:max_rho}.

The analyses for the first and third terms in \pref{eq:correction eq1} are similar to the counterpart in the proof of \pref{lem:bandit-VT-correction-cvx}.
Let $Q$ be the same bad event as in \pref{eq:high_prob_rho}, with $\bar{\rho}=4d\log(dT)$.
We can upper-bound the expectation of the first term in \pref{eq:correction eq1} by
\begin{align}
    &\E\mbr{\sumT\sum_{i=1}^d \rho_{t,i} (\nabla_{i} f_t(\w_t) - \nabla_{i} f_{t}(\w^\star))^2}\notag\\
    \le{}& \E\mbr{\sumT\sum_{i=1}^d \rho_{t,i} (\nabla_{i} f_t(\w_t) - \nabla_{i} f_{t}(\w^\star))^2\I(Q)}+ \E\mbr{\sumT\sum_{i=1}^d \rho_{t,i} (\nabla_{i} f_t(\w_t) - \nabla_{i} f_{t}(\w^\star))^2\I(\neg Q)}\notag\\
    \le{}& \O(1)+4d\log(dT) \E\mbr{\sumT \|\nabla f_t(\w_t)-\nabla f_t(\w^\star)\|^2}
    \le\O(1)+ 8dL\log(dT) \E\mbr{\sumT \D_{f_t}(\w^\star,\w_t)},\label{eq:correction eq3}
\end{align}
where the third step uses $\rho_{t,i}\le T+1$ on $Q$, $\rho_{t,i}\le \bar{\rho}$ on $\neg Q$, \pref{eq:high_prob_rho}, and the boundedness of $\X$, and the last step is by $\|\nabla f_t(\x) - \nabla f_t(\y)\|^2 \le 2L \D_{f_t}(\y,\x)$ for any $\x,\y\in\X^+$~\citep{yan2024simple}.
Similarly, we can upper-bound the expectation of the third term of \pref{eq:correction eq1} by
\begin{align}
    \E\mbr{\sumT\sum_{i=1}^d \rho_{t,i} (\nabla_{i} f_{t-1}(\w_{t-1}) - \nabla_{i} f_{t-1}(\w^\star))^2}
    \le\O(1)+ 8dL\log(dT) \E\mbr{\sumT \D_{f_t}(\w^\star,\w_t)}.\label{eq:correction eq4}
\end{align}
Plugging \pref{eq:correction eq2}, \pref{eq:correction eq3}, and \pref{eq:correction eq4} back into \pref{eq:correction eq1}, we finish the proof.
\end{proof}
\begin{proof}[of \pref{lem:meta-gradient-inner-product}]
    Let $\E_t[\cdot]$ denote the conditional expectation given the history before sampling $i_t$ at round $t$.
    Since $\w_t$, $\w_{t,\is}$, and $\tgb_t$ are determined before sampling $i_t$, they are fixed under $\E_t[\cdot]$.

    For each $j\in[d]$, define the directional difference
    \begin{equation*}
        q_{t,j}\define \frac{1}{2\delta}\sbr{f_t(\w_t+\delta\eb_j)-f_t(\w_t-\delta\eb_j)}.
    \end{equation*}
    Thus, the observed quantity in \pref{eq:chiangestimator} satisfies $v_t=q_{t,i_t}$.
    By $G$-Lipschitzness, $|q_{t,j}|\le G$ for every $j\in[d]$.
    Moreover, since each coordinate of $\tgb_t$ stores a previously observed directional difference, property~\rom{2} of \pref{lem:alpha} implies $|\tilde{g}_{t,j}|\le G$ for every $j\in[d]$.

    By \pref{eq:chiangestimator}, we have
    \begin{equation*}
        \g_t=\tgb_t+d\sbr{q_{t,i_t}-\tilde{g}_{t,i_t}}\eb_{i_t}.
    \end{equation*}
    Therefore,
    \begin{equation*}
        \inner{\g_t}{\w_t-\w_{t,\is}}
        =
        \inner{\tgb_t}{\w_t-\w_{t,\is}}
        +
        d\sbr{q_{t,i_t}-\tilde{g}_{t,i_t}}\inner{\eb_{i_t}}{\w_t-\w_{t,\is}}.
    \end{equation*}
    Using $(a+b)^2\le 2a^2+2b^2$, we obtain
    \begin{align*}
        \inner{\g_t}{\w_t-\w_{t,\is}}^2
        \le
        2\inner{\tgb_t}{\w_t-\w_{t,\is}}^2+2d^2\sbr{q_{t,i_t}-\tilde{g}_{t,i_t}}^2
        \inner{\eb_{i_t}}{\w_t-\w_{t,\is}}^2.
    \end{align*}
    By property~\rom{5} of \pref{lem:alpha}, the first term satisfies
    \begin{equation*}
        2\inner{\tgb_t}{\w_t-\w_{t,\is}}^2
        \le 2\norm{\tgb_t}^2\norm{\w_t-\w_{t,\is}}^2
        \le 2dG^2\norm{\w_t-\w_{t,\is}}^2.
    \end{equation*}
    For the second term, since $i_t$ is uniformly sampled from $[d]$ conditional on the history before round $t$, we have
    \begin{align*}
        &\E_t\mbr{2d^2\sbr{q_{t,i_t}-\tilde{g}_{t,i_t}}^2
        \inner{\eb_{i_t}}{\w_t-\w_{t,\is}}^2}
        =
        2d\sum_{j=1}^d\sbr{q_{t,j}-\tilde{g}_{t,j}}^2
        \inner{\eb_j}{\w_t-\w_{t,\is}}^2 \\
        \le{}&
        8dG^2\sum_{j=1}^d\inner{\eb_j}{\w_t-\w_{t,\is}}^2
        =
        8dG^2\norm{\w_t-\w_{t,\is}}^2.
    \end{align*}
    Combining the above bounds yields
    \begin{equation*}
        \E_t\mbr{\inner{\g_t}{\w_t-\w_{t,\is}}^2}
        \le 10dG^2\norm{\w_t-\w_{t,\is}}^2.
    \end{equation*}
    Taking expectation over the history completes the proof.
\end{proof}
\paragraph{Worst-Case Universal Regret.}
As a byproduct, we show that the online ensemble framework also achieves worst-case universal regret guarantees using a single gradient estimator shared across all base learners, with a simpler analysis.

Following \citet[Proposition~1]{yan2023universal}, we use \mlprod~\citep{gaillard2014prod} as the meta-algorithm and Bandit Gradient Descent~\citep{agarwal2010optimal} as the base algorithm. As in the gradient-variation universal setting, we instantiate a collection of strongly convex base learners with candidate curvature parameters from the predefined pool~\eqref{eq:candidate-pool}, together with one additional base learner for the linear and convex cases. Therefore, the ensemble contains $N\define 1+|\H^\scvx|=\O(\log T)$ base learners in total.

At each round $t$, we construct the two-point gradient estimator
\begin{align*}
    \g_t=\frac{d}{2\delta}\bigl(f_t(\w_t+\delta\u_t)-f_t(\w_t-\delta\u_t)\bigr)\u_t,
\end{align*}
and feed \mlprod the normalized linearized loss vector $\ellb_t=(\ell_{t,1},\ldots,\ell_{t,N})$, where $\ell_{t,i}\define \frac{1}{2dGR}\inner{\g_t}{\w_{t,i}}+\frac{1}{2}\in[0,1]$ and $\w_{t,i}$ denotes the decision of the $i$-th base learner. Here, the range condition follows from $\|\g_t\|\le dG$ and $\|\w_{t,i}\|\le R$. For each candidate curvature parameter $\lambda_i\in\H^\scvx$, the corresponding strongly convex base learner runs \pref{alg:scvx-T} with the associated configuration. The additional base learner handles the linear and convex cases using the configuration of Theorem~1 of \citet{shamir2017optimal}.

\begin{Theorem}
    \label{thm:2point-BCO-worst}
    Under Assumptions \ref{ass:boundedness}-\ref{ass:Lipschitzness}, the above ensemble algorithm achieves the following universal regret:
        \begin{equation*}
        \E[\Reg_T] \le
        \begin{cases}
            \O\sbr{\frac{d}{\lambda}\log T}, & \textnormal{($\lambda$-strongly convex case)}, \\[2mm]
            \O\sbr{\sqrt{dT}+d}, & \textnormal{(convex/linear case)}
        \end{cases}
    \end{equation*}
\end{Theorem}
\begin{proof}
For \emph{$\lambda$-strongly convex} functions, let $\is$ index a base learner whose candidate parameter satisfies $\lambda_\is\le\lambda\le2\lambda_\is$. We have the decomposition
    \begin{align*}
        \E[\Reg_T] \le {} & \underbrace{\E\mbr{\sumT \inner{\g_t}{\w_t - \w_{t,\is}}}}_{\meta}+ \underbrace{\E\mbr{\sumT \hsc_{t,\is}(\w_{t,\is}) - \sumT \hsc_{t,\is}(\ws)}}_{\base} - \frac{\lambda_\is}{2} \E\mbr{\sumT\|\w_t - \w_{t,\is}\|^2}+\O(1),
    \end{align*}
where $\hsc_{t,i}(\x) \define \inner{\g_t}{\x} + \frac{\lambda_i}{2} \|\x - \w_t\|^2$ is the surrogate loss corresponding to $\lambda_i\in\H^\scvx$ in~\eqref{eq:candidate-pool}.

For the meta regret, applying Corollary~4 of \citet{gaillard2014prod} to the normalized losses above and rescaling back to the original linearized losses, there exist constants $C_8$ and $C_9$ such that
\begin{align*}
    &\meta
    \le \O(d)+C_8\E\mbr{\sqrt{\ln\ln T\cdot\sumT \inner{\g_t}{\w_t-\w_{t,\is}}^2}}\\
    \le{}& \O(d)+C_8\sqrt{\ln\ln T\cdot \E\mbr{\sumT \|\g_t\|^2\|\w_t-\w_{t,\is}\|^2}}
    \le\O(d)+C_8\sqrt{C_9dG^2\ln\ln T\cdot \E\mbr{\sumT\|\w_t-\w_{t,\is}\|^2}}\\
    \le{}& \O(d)+\frac{C_8^2C_9dG^2\ln\ln T}{2\lambda_\is}+\frac{\lambda_\is}{2}\E\mbr{\sumT\|\w_t-\w_{t,\is}\|^2}
    \le \O(d)+\frac{C_8^2C_9dG^2\ln\ln T}{\lambda}+\frac{\lambda_\is}{2}\E\mbr{\sumT\|\w_t-\w_{t,\is}\|^2}.
\end{align*}
Here, the second inequality follows from Jensen's inequality and Cauchy-Schwarz, the third inequality follows from Lemma~5 of \citet{shamir2017optimal}, the fourth inequality follows from the AM-GM inequality, and the last inequality uses $\lambda_\is\le\lambda\le2\lambda_\is$. For the base regret, \pref{thm:T-strcvx} implies $\base\le \O(\frac{d}{\lambda_\is}\log T)=\O(\frac{d}{\lambda}\log T)$. Combining the meta and base regret bounds with the negative quadratic term in the decomposition gives $\E[\Reg_T]\le \O(\frac{d}{\lambda}\log T)$, where the $\log\log T$ factor is omitted according to our convention.

For \emph{convex} and \emph{linear} functions, we have the decomposition
    \begin{align*}
        \E[\Reg_T] \le {} & \underbrace{\E\mbr{\sumT \inner{\g_t}{\w_t - \w_{t,\is}}}}_{\meta}+ \underbrace{\E\mbr{\sumT \inner{\g_t}{\w_{t,\is}-\ws}}}_{\base}+\O(1),
    \end{align*}
For the meta regret, applying Corollary~4 of \citet{gaillard2014prod} to the normalized losses and rescaling, there exists a constant $C_8$ such that
\begin{align*}
    \meta
    \le{}& \O(d)+C_8\E\mbr{\sqrt{\ln\ln T\cdot\sumT \inner{\g_t}{\w_t-\w_{t,\is}}^2}}\\
    \le{}& \O(d)+2RC_8\sqrt{\ln\ln T\cdot\sumT\E\mbr{\|\g_t\|^2}}
    \le \O(\sqrt{dT}+d).
\end{align*}
Here, the second inequality follows from Jensen's inequality and $\|\w_t-\w_{t,\is}\|\le2R$, while the last inequality follows from Lemma~5 of \citet{shamir2017optimal}. By Theorem~1 of \citet{shamir2017optimal}, the base regret satisfies $\base\le \O(\sqrt{dT})$. Hence, combining the meta and base regret bounds yields $\E[\Reg_T]\le \O(\sqrt{dT}+d)$, where the $\log\log T$ factor is omitted according to our convention.

\end{proof}


\section{Technical Lemmas}
\label{app:technical}
\begin{Lemma}[{Lemma 4.8 of \citet{pogodin2019first}}]
    \label{lem:sum_cvx}
    Let $a_1, a_2, \ldots, a_T$ be non-negative real numbers. Then
    $$
    \sum_{t=1}^T \frac{a_t}{\sqrt{1+\sum_{s=1}^{t-1} a_s}} \leq 4 \sqrt{1+\sum_{t=1}^T a_t}+\max _{t \in[T]} a_t .
    $$
\end{Lemma}
\begin{Lemma}[{Lemma 9 of \citet{yan2023universal}}]
    \label{lem:sum_scvx}
    For a sequence of $\{a_t\}_{t=1}^T$ and $b$, where $a_t, b > 0$ for any $t \in [T]$, denoting by $a_{\max} \define \max_t a_t$ and $A \define \ceil{b \sumT a_t}$, we have
    \begin{equation*}
        \sumT \frac{a_t}{bt} \le \frac{a_{\max}}{b} (1 + \log A) + \frac{1}{b^2}.
    \end{equation*}
\end{Lemma}
\begin{Lemma}[{Lemma 9 of \citet{zhao2024adaptivity}}]
    \label{lem:small-loss-sqrt}
    For any $x, y, a, b>0$ satisfying $x-y \le  \sqrt{a x}+b$, it holds that
    \begin{equation*}
        x-y \le  \sqrt{a y+ab}+a+b.
    \end{equation*}
\end{Lemma}
\begin{Lemma}[{Lemma 16 of \citet{orabona2012beyond}}]
    \label{lem:small-loss-log}
    If $a, b, c, x, y>0$ satisfy $x-y \le  a \log (b x+c)$, then it holds that
    \begin{equation*}
        x-y \le  a \log \sbr{2 a b \log \frac{2 a b}{e} + 2by + 2c}.
    \end{equation*}
\end{Lemma}

\end{document}